\documentclass[lettersize,journal]{IEEEtran}
\usepackage{amsmath,amsfonts}
\usepackage{algorithmic}
\usepackage{algorithm}
\usepackage{array}
\usepackage[caption=false,font=normalsize,labelfont=sf,textfont=sf]{subfig}
\usepackage{textcomp}
\usepackage{stfloats}
\usepackage{url}
\usepackage{verbatim}
\usepackage{graphicx}
\usepackage{cite}

\usepackage{graphics} 
\usepackage{epstopdf}
\graphicspath{{figures/}}
\usepackage{epsfig} 
\usepackage{mathptmx} 
\usepackage{times} 
\usepackage{amssymb}  
\usepackage{booktabs}
\usepackage[table]{xcolor}
\usepackage{array}
\usepackage{xcolor}
\usepackage{longtable}
\usepackage{amsfonts}
\usepackage{multirow}
\usepackage{bm}
\usepackage{wasysym}
\usepackage{threeparttable}
\usepackage{textcomp}
\usepackage{epsf}
\usepackage{enumerate}
\usepackage{fontenc}
\usepackage{diagbox}
\usepackage{setspace}
\usepackage{multirow, makecell}
\usepackage{amssymb}

\usepackage{bbding}

\usepackage{amsthm}

\newtheorem{corollary}{Corollary}

 \newtheorem{proposition}{Proposition}
\newtheorem{remark}{Remark}

\newtheorem{assumption}{Assumption}

\newtheorem{theorem}{Theorem}

\newtheorem{lemma}{Lemma}
\def\BibTeX{{\rm B\kern-.05em{\sc i\kern-.025em b}\kern-.08em
    T\kern-.1667em\lower.7ex\hbox{E}\kern-.125emX}}
\usepackage{balance}
\begin{document}
\title{Energy-Based Model for Accurate Estimation of Shapley Values in Feature Attribution}

\author{Cheng Lu, Jiusun Zeng, Yu Xia, Jinhui Cai, Shihua Luo
\thanks{This work was supported in part by the Baima Lake Laboratory Joint Fund of Zhejiang Provincial Natural Science Foundation of China(LBMHZ25F03001), the "Pioneer and Leading Goose + X" S\&T Program of Zhejiang Province (2025C01022), the NSFC grant (12271133, U21A20426), and the key project of Zhejiang Provincial Natural Science Foundation (LZ23A010002). \emph{(Corresponding author: Jiusun Zeng)}

Cheng Lu and Jinhui Cai are with the College of Metrology Measurement and Instrument, China Jiliang University, Hangzhou, 310018, China. (lucheng.cjlu@gmail.com, caijinhui@cjlu.edu.cn)

Jiusun Zeng and Yu Xia are with the School of Mathematics, Hangzhou Normal University, Hangzhou, 311121, China. (jszeng@hznu.edu.cn, yxia@hznu.edu.cn)

Shihua Luo is with the School of Statistics, Jiangxi University of Finance and Economics, Nanchang, 330013, China. (luoshihua@aliyun.com)}}


\markboth{}%
{IEEEtran \LaTeX \ Templates}

\maketitle

\begin{abstract}
Shapley value is a widely used tool in explainable artificial intelligence (XAI), as it provides a principled way to attribute contributions of input features to model outputs. However, estimation of Shapley value requires capturing conditional dependencies among all feature combinations, which poses significant challenges in complex data environments. In this article, EmSHAP (Energy-based model for Shapley value estimation), an accurate Shapley value estimation method, is proposed to estimate the expectation of Shapley contribution function under the arbitrary subset of features given the rest. By utilizing the ability of energy-based model (EBM) to model complex distributions, EmSHAP provides an effective solution for estimating the required conditional probabilities. To further improve estimation accuracy, a GRU (Gated Recurrent Unit)-coupled partition function estimation method is introduced. The GRU network captures long-term dependencies with a lightweight parameterization and maps input features into a latent space to mitigate the influence of feature ordering. Additionally, a dynamic masking mechanism is incorporated to further enhance the robustness and accuracy by progressively increasing the masking rate. Theoretical analysis on the error bound as well as application to four case studies verified the higher accuracy and better scalability of EmSHAP in contrast to competitive methods.
\end{abstract}

\section{Introduction}\label{Sec-1}
Artificial intelligence techniques, particularly deep learning (DL), have been extensively applied to a wide range of complex supervised and unsupervised learning tasks, such as medical diagnosis~\cite{Aggarwal2021medical}, industrial intelligence~\cite{qi2020wrinkled}, and financial modeling~\cite{Ozbayoglul2020financial}. In many scenarios, DL models achieve superior performance compared to conventional methods in terms of accuracy. However, due to their typically black-box nature and complex internal structures, interpreting the results of DL models remains challenging. This opacity hinders the identification of key features influencing model decisions and complicates the understanding of relationships between inputs and outputs. While in tasks where predictive or classification performance is the sole objective, black-box models may be acceptable without requiring interpretability, the situation is markedly different in high-stakes or safety-critical domains such as healthcare, finance, and industrial safety. In these fields, interpretability is not just advantageous but essential for the responsible deployment of artificial intelligence (AI) systems, given the serious consequences that may arise from model errors. The pressing need for transparency in such applications has driven the emergence and growth of Explainable Artificial Intelligence (XAI) as a vital research area~\cite{Arrieta2020XAI}.

In XAI, the concept of model interpretability lacks a clear mathematical formulation. Nevertheless, there is a general consensus on the formulation proposed by Ref.\cite{doshi2017towards}. According to Ref.\cite{doshi2017towards}, an explanation algorithm is considered trustworthy if it accurately reveals the underlying reasoning behind a model's decisions. In this regard, Shapley value-based explanation becomes popular due to its rigid mathematical foundation and theoretical basis. Originated from the field of cooperative games~\cite{shapley1953value} and reintroduced as an explanation framework~\cite{Strumbelj2010explanation},  Shapley value effectively overcomes the limitations of black-box models through the implementation of fair allocation schemes. Although the utilization of Shapley values for explaining DL models may sometimes lead to non-intuitive feature attributions~\cite{pmlr-v119-kumar20e, pmlr-v119-sundararajan20b, 9565902}, the ideal mathematical properties make it possible to fairly measure the contribution of each input in the model. Hence, Shapley value-based interpretable approaches still remain a popular choice in XAI.

A notable problem for Shapley-based explanation is its high computation load. By definition, the computational complexity of Shapley value increases exponentially with the number of variables. In order to accelerate computation, a series of methods have been developed that can be roughly divided into five categories: sampling-based methods, regression-based methods, gradient-based methods, structure model-based methods, and generative model-based methods. Sampling-based methods utilize Monte Carlo sampling to approximate the true Shapley value. These methods involve Monte Carlo approximation using randomly sampled permutations~\cite{vstrumbelj2014explaining} or sampling based on random permutations of local features~\cite{chen2018shapley}. Regression-based methods employ weighted linear regression to accelerate the estimation of Shapley value, with examples including SHAP (SHapley Additive exPlanation)~\cite{lundberg2017unified} and KernelSHAP~\cite{covert2020improving}. Gradient-based methods estimate Shapley values by exploiting the gradient of the model output with respect to input features. Typical approaches include DeepSHAP~\cite{lundberg2017unified} and its variant G-DeepSHAP~\cite{chen2022explaining}. Structure model-based methods leverage specific model structure constraints to approximate the Shapley value. Examples of such methods include tree Shapley~\cite{kar2002axiomatization} and graph Shapley~\cite{skibski2014algorithms}. More recently, generative model-based methods have been developed for Shapley value estimation. A representative example is the variational autoencoder with arbitrary conditioning (VAEAC)~\cite{olsen2022using}, which uses a single model for all feature combinations.

Despite the advancements made in the calculation of Shapley value, several significant issues still remain. For example, sampling-based methods require a large quantity of samplings to achieve sufficient accuracy, and it may suffer from high variance during sampling. Regression-based methods can improve the estimation efficiency of Shapley value, however, the inherent assumption of independence among variables may negatively affect its accuracy. Whilst gradient-based and structure model-based approaches demonstrate improved efficiency, their application scenarios are limited to specific models. Another notable research gap is that there still lacks a unified theoretical framework on the error bounds of different Shapley value-based approaches, making it difficult to quantitatively evaluate and compare their accuracy.

In order to increase estimation accuracy and provide theoretical guidance on the error bound, this article presents the EmSHAP (Energy-based model SHapley value estimation). Similar to Ref.~\cite{olsen2022using}, EmSHAP introduces a generative model during the estimation of Shapley contribution function. It relies on the energy-based model(EBM) to obtain the expectation of Shapley contribution function of the deep learning model under some arbitrary subset of features given the rest. The use of EBM is motivated by its ability to model complex, high-dimensional distributions without explicit normalization, which provides a flexible and efficient framework for estimating the conditional distributions required in Shapley value computation. Unlike original EBM, the EBM in EmSHAP is decomposed into an energy network to approximate the unnormalized conditional density and a GRU network for approximation of the proposal distribution in the partition function. The GRU network has a simple structure and is capable of capturing long-term dependencies, and the adopted dynamic masking scheme improves the flexibility. The contributions of the paper can be summarized as: i) Methodological refinements on the EBM framework by introducing the GRU network and the dynamic masking scheme, resulting in more accurate Shapley value estimation; ii) Theoretical analysis of error bounds for several state-of-the-art methods are performed with the purpose of providing a quantitative measure to evaluate and compare the accuracy of different methods; iii) Extensive application experiments are conducted to verify the accuracy and scalability of different Shapley value estimation methods, so that practitioners can refer to when choosing appropriate Shapley value estimation methods for specific applications.

\section{Related work}\label{Sec-2}
Shapley value is a concept originated from cooperative game theory that provides a fair and reasonable method for quantifying feature attribution contributions. As a result, they have garnered significant attentions in the field of XAI. However, the calculation of exact Shapley values is an NP-hard problem~\cite{matsui2001np}. To address this computational complexity, researchers have developed various approximation methods. This section provides a brief review on the different kinds of approximation methods in the literature, which can be summarized into sampling-based, regression-based, gradient-based, structure model-based and generative model-based methods.

\textbf{Sampling-based methods:}
Sampling-based Shapley value estimation is among the earliest approaches studied and remains the most widely adopted in practice. The core principle is to randomly sample subsets of features and compute the average contribution of the features to the model output. The most commonly used sampling-based methods are Monte Carlo sampling~\cite{Strumbelj2014explanation, pang2025shapley}, and Monte Carlo sampling with random arrangements~\cite{castro2009polynomial, Strumbelj2010explanation}.
In addition to standard Monte Carlo sampling, recent studies have explored alternative strategies to improve estimation accuracy, such as Owen-based sampling~\cite{Okhrati2020Multilinear}, Quasi-Monte Carlo sampling~\cite{Rory2022Sampling}, permutation sampling~\cite{tsai2023faith}, and stratified sampling~\cite{castro2017improving, zhang2023efficient}. Most of these approaches are model-agnostic and do not require explicit feature removal, making them attractive for practical applications. However, sampling-based methods may suffer from high variance during the sampling procedures.

\textbf{Regression-based methods:}
Regression-based methods estimate Shapley values by using regression models to approximate the marginal contribution of each feature to the output. These methods approximate the original black-box model by constructing a simpler surrogate model. Among these methods, KernelSHAP~\cite{lundberg2017unified} is the most widely used. It transforms the Shapley value estimation into a weighted least squares problem through the Shapley kernel, leading to significantly enhanced estimation efficiency. Subsequent works further extended this line of research. For example, Covert \emph{et al.}~\cite{covert2020improving} proposed an improved version of KernelSHAP to reduce the estimation bias, Simon and Vincent~\cite{simon2020projected} introduced SGD-Shapley, which uses a stochastic gradient descent scheme to iteratively solve the weighted least squares problem, while Jethani \emph{et al.}~\cite{jethani2022fastshap} developed FastSHAP, a more computationally efficient variant ST-SHAP~\cite{kelodjou2024shaping} is developed to enhance stability by employing layer-wise neighbor selection. Additionally, the more complex variants like generalized linear models~\cite{bordt2023shapley} and kernel regression~\cite{chau2022rkhs} also belong to regression-based approaches. Compared to sampling-based methods, regression-based methods offer higher computational efficiency and lower estimation variance. However, since most regression methods compute the marginal contribution instead of the real contribution of features, they may introduce bias in the estimates.

\textbf{Gradient-based methods:}
Gradient-based methods estimate Shapley values by exploiting the gradient of the model output with respect to input features. The basic idea is to approximate the contribution of each feature by evaluating how incremental changes in the feature influence the output. Typical approaches include DeepSHAP~\cite{lundberg2017unified} and its variant G-DeepSHAP~\cite{chen2022explaining}, both designed to alleviate the computational cost of exact Shapley value estimation. In addition, Ancona \emph{et al.}~\cite{ancona2019explaining} proposed a polynomial-time approximation for Shapley values in deep neural networks based on uncertainty propagation, and Wang \emph{et al.}~\cite{wang2021shapley} introduced Shapley Explanation Networks (SHAPNETs) that treat Shapley values as latent representations to enable hierarchical explanations and regularization. Despite their efficiency, gradient-based methods require access to model gradients and are therefore limited to differentiable architectures.

\textbf{Structure model-based methods:}
Structure model-based methods estimate Shapley values by leveraging the inherent structure of specific models to simplify computation. These methods exploit the constraints and decomposition properties of structured models, such as decision trees and graphs, to efficiently calculate exact or approximate Shapley values. Tree Shapley~\cite{kar2002axiomatization}, TreeSHAP~\cite{Scott2020treeshap}, and graph Shapley~\cite{skibski2014algorithms} fall into this category. They utilize the model structure to significantly reduce computational complexity. However, structure model-based methods are restricted to certain model classes so that their application is limited.

\textbf{Generative model-based methods:}
Different from other methods, generative model-based methods estimate the conditional probability of features in Shapley value estimation. For example, the on-manifold Shapley method~\cite{frye2021shapley} estimates the Shapley value within the manifold space and employs VAE to compute conditional probabilities across different feature subsets. VAEAC~\cite{olsen2022using}, a variant of this approach, applies various masking techniques to the training data to train a model capable of generating any conditional probability distribution. The problem of VAE-based generative models is that the amortization error~\cite{cremer2018inference} in VAE may bring bias to the Shapley value estimation.

In addition to the above categories, recent studies have extended Shapley values to quantify feature interactions rather than individual feature importance. Techniques like SHAP-IQ~\cite{fumagalli2023shap}, TreeSHAP-IQ~\cite{muschalik2024beyond}, KernelSHAP-IQ~\cite{Fabian2024kernelSHAPIQ}, and SVARM-IQ~\cite{Patrick2024svarmiq}, which provide effective frameworks to identify and interpret higher-order interactions among features. These methods are particularly useful in settings where feature dependencies play a central role.

The aforementioned methods represent the mainstream approaches for Shapley value estimation. For a broader synthesis, readers are referred to Ref.~\cite{chenhugh2023shapoverview}.

\section{Preliminaries}\label{Sec-3}
In this section, preliminary knowledge about Shapley value-based explanation framework is presented.
\subsection{Shapley value}\label{Sec-3-1}
Shapley value is a tool from game theory designed to fairly allocate the contribution generated by a coalition game among its individual players \cite{shapley1953value}. Suppose a total of $d$ players participate in a game, and these players form a set of $D=\{1,\cdots, i,\cdots, |D|\}$, with $|D|$ being the number of all players. The contribution/revenue generated by these players in this game can be quantified as $v(D)$, with $v(\cdot)$ being a predefined contribution function. Similarly, the contribution generated by a player subset $S(S \subseteq D)$ can be quantified as $v(S)$. The Shapley value is defined as the marginal contribution of the player to the payoff of all possible player subsets. The definition of the Shapley value is as follows.
\begin{equation}\label{eq-2-1}
\begin{small}
    \phi_i=\sum_{S \subseteq D \backslash \{i\}}\frac{|S|!(|D|-|S|-1)!}{|D|!}(v(S\cup\{i\})-v(S)).
\end{small}
\end{equation}
Here, $\phi_i$ is the Shapley value of the $i$-th player and $|S|$ is the number of players in subset $S$. The definition of Eq.(\ref{eq-2-1}) renders Shapley value the following favorable properties.
\begin{itemize}
\item \textbf{Efficiency}: The sum of Shapley values for players from a subset is equal to the difference between the contribution function of the collection of these players and the contribution function of the empty set, i.e. $\sum_{i=1}^{|S|}\phi_i = v(S)-v(\varnothing)$;

\item \textbf{Symmetry}: For any coalition $S$ not containing $i_1$ and $i_2$, if player $i_1$ and player $i_2$ satisfy $v(S \cup \{i_1\})=v(S \cup \{i_2\})$, then $\phi_{i_1}=\phi_{i_2}$;

\item \textbf{Dummy}: Regardless of which coalition of players it is added to, a player $i$ does not change the contribution value, that is
 $v(S)=v(S \cup \{i\})$ for any coalition $S$, then $\phi_i=0$;

\item \textbf{Additivity}: For a game with combined contribution function $v_1+v_2$, the Shapley value can be calculated by $\phi_{i_1}+\phi_{i_2}$.
\end{itemize}

\subsection{Shapley value-based model explanation}\label{Sec-3-2}
The Shapley value, grounded in a rigorous mathematical framework and underpinned by fairness principles, serves as a powerful tool for feature attribution. By conceptualizing features as players and the model output as the value function, Shapley value offers a principled tool for quantifying the contribution of each feature to the model’s output, thereby enhancing the interpretability of complex models.

Let $\tilde{f}: \mathbf{x} \in \mathbb{R}^{|D|} \rightarrow y \in \mathbb{R}$ be an unknown data generating function that generates a set of $n$ measurements $\mathbf{X}=\left\{\mathbf{x}^t\right\}^{n}_{t=1}, \mathbf{x}^t\in \mathbb{R}^{|D|}$ of input features $\mathbf{x}$ and $\mathbf{y}=\left\{y^t\right\}_{t=1}^n, y^t \in \mathbb{R}$ of output feature $y$, and $y^t=\tilde{f}(\mathbf{x}^t)$. $D$ is the input feature set and $|D|$ is the number of features. In a classification/predictive task, the purpose is to approximate the unknown function $\tilde{f}$ using a specific model $\hat{y}=f(\mathbf{x})$ to minimize the approximation error between $y$ and $\hat{y}$ as well as the out-of-sample generalization error. The approximation model $f(\cdot)$ can be of any type of machine learning models, in this paper, deep learning models are mainly considered.

In Shapley value-based model explanation framework, the $|D|$ input features are related to the $|D|$ players, the classification/predictive model $f(\mathbf{\cdot})$ relates to the cooperative game and the model output $\hat{y}$ can be regarded as the sum of contribution(or Shapley values) from all features in $\mathbf{x}$. Noting that the framework is model-agnostic, it can be used to interpret a wide variety of deep learning models.

The output for a specific sample $\mathbf{x}^t \in \mathbb{R}^{|D|}$ in $\mathbf{X}$ can be expressed as the sum of the Shapley values associated with each input feature, obtained by decomposing the model output. Specifically, this decomposition provides the following expression for the output corresponding to $\mathbf{x}^t$.
\begin{equation}\label{eq-2-2}
\begin{small}
    f(\mathbf{x}^t)=\phi_0+\sum^{|D|}_{i=1}\phi_i^t.
\end{small}
\end{equation}
Here, $\phi_0$ is the base value of the model outputs and can be regarded as the expectation of global output~\cite{lundberg2017unified}, with $\phi_0=\mathbb{E}[f(\mathbf{x})]$, $\phi_i^t$ is the Shapley value of the $i$-th input feature $x^t_i$. Eq.(\ref{eq-2-2}) highlights the disparity between the model output and the global output by employing the Shapley value associated with each feature in $\mathbf{x}^t$.

To calculate $\phi_i^t$, an appropriate contribution function $v(S)$ should be determined. Since $f(\mathbf{x}^t)$ is the sum of contribution from all input features, we can easily relate $v(\cdot)$ to $f(\cdot)$, that is, $v(D)=f(\mathbf{x}^t)$ for a specific sample and $\bar{v}(D)=\mathbb{E}[f(\mathbf{x})]$ for global output. Following Ref.~\cite{lundberg2017unified}, the contribution function can be reformed as the expected outcome of $f(\mathbf{x})$ conditioned on the features in $S$ taking on specific values $\mathbf{x}_S^t$ as follows.
\begin{equation}\label{eq-2-3}
\begin{small}
    \begin{aligned}
        v(S) &=  \mathbb{E}[f(\mathbf{x})|\mathbf{x}_S=\mathbf{x}_S^t]=\mathbb{E}[f(\mathbf{x}_{\bar{S}},\mathbf{x}_S)|\mathbf{x}_S=\mathbf{x}_S^t] \\ & =\int f(\mathbf{x}_{\bar{S}},\mathbf{x}_S^t)p(\mathbf{x}_{\bar{S}}|\mathbf{x}_S=\mathbf{x}_S^t)d{\mathbf{x}_{\bar{S}}},
    \end{aligned}
\end{small}
\end{equation}
$\bar{S}$ is the complement of subset $S$ in the feature set $D$ (i.e., $\bar{S}=D \backslash S$), $p(\mathbf{x}_{\bar{S}}|\mathbf{x}_S=\mathbf{x}_S^t)$ is the conditional density of $\mathbf{x}_{\bar{S}}$ given $\mathbf{x}_S^t$, $\mathbf{x}_S$ and $\mathbf{x}_S^t$ contains the part of inputs $\mathbf{x}$ and $\mathbf{x}^t$ corresponding to features inside $S$, $\mathbf{x}_{\bar{S}}$ and $\mathbf{x}_{\bar{S}}^t$ relate to the part of inputs in $\mathbf{x}$ and $\mathbf{x}^t$ corresponding to features outside $S$.

The Shapley value provides a fair assessment of each feature's contribution to the model output. However, from Eq. (\ref{eq-2-1}) it can be seen that its estimation involves high computation load as there are a total of $2^{|D|-1}$ combinations of $S$. The computational complexity will increase exponentially as the number of input features grows.

\section{Methodology}\label{Sec-4}
By definition of Shapley value in Eq.(\ref{eq-2-1}), it can be seen that the most important step is to estimate the contribution function $v(S)$, which can be reformulated as the conditional expectation in Eq.(\ref{eq-2-3}). The conditional expectation can be empirically approximated by techniques like Monte Carlo integration provided the conditional probability $p(\mathbf{x}_{\bar{S}}|\mathbf{x}_S)$ is known.
\begin{equation}\label{eq-3-0}
\begin{small}
    \begin{aligned}
            v(S)&=\mathbb{E}[f(\mathbf{x}_{\bar{S}},\mathbf{x}_S)|\mathbf{x}_S=\mathbf{x}_S^t]=\int f(\mathbf{x}_{\bar{S}},\mathbf{x}_S^t)p(\mathbf{x}_{\bar{S}}|\mathbf{x}_S=\mathbf{x}_S^t) d\mathbf{x}_{\bar{S}}
             \\ & \approx \hat{v}(S) = \frac{1}{K}\sum_{k=1}^{K}f(\mathbf{x}_{\bar{S}}^{(k)},\mathbf{x}_S^t),
    \end{aligned}
\end{small}
\end{equation}
where $\mathbf{x}_{\bar{S}}^{(k)} \sim p(\mathbf{x}_{\bar{S}}|\mathbf{x}_S=\mathbf{x}_S^t)$ for $k=1,2,...,K$, and $K$ is the number of Monte Carlo sampling. It is worth noting that estimation of the conditional density $p(\mathbf{x}_{\bar{S}}|\mathbf{x}_S)$ can be achieved by techniques like KDE. However, KDE is computationally challenging for high-dimensional inputs. More importantly, since there are $2^{|D|-1}$ possible subsets of features, KDE would require estimating an exponential number of conditional densities, rendering it infeasible for practical applications.

\subsection{Network structure}\label{Sec-4-1}
To address the computation challenge of conditional density estimation, EmSHAP is designed to capture the complex dependencies among features and produce accurate conditional estimates, which are then used to calculate the Shapley values. The detailed network structure of EmSHAP is shown in Fig.~\ref{fig1}.
\begin{figure*}[htbp]
\centerline{\includegraphics[width=\linewidth]{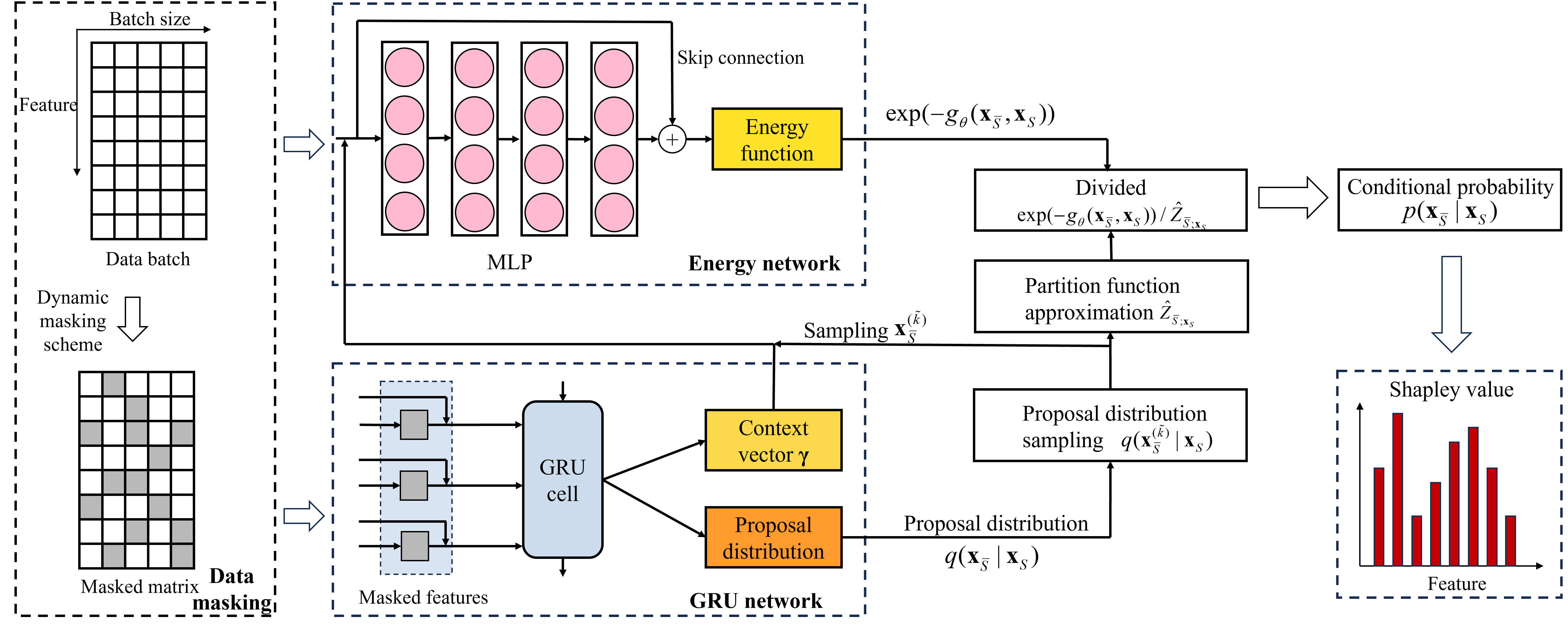}}
\caption{Structure of EmSHAP for estimating Shapley value.}
\label{fig1}
\end{figure*}

The network structure in Fig.~\ref{fig1} mainly consists of three components, namely, the data masking mechanism for generating different combinations of conditional densities, the energy network for estimating of energy function/unnormalized density function and the GRU network for estimation of proposal distribution. The proposal distribution serves to estimate the partition function, which provides the normalization for the energy function. In Fig.~\ref{fig1}, data batches of input features are masked and fed into the energy network to estimate the energy function $g_{\theta}(\mathbf{x}_{\bar{S}},\mathbf{x}_S)$ and the unnormalized probability $\text{exp}(-g_{\theta}(\mathbf{x}_{\bar{S}},\mathbf{x}_S))$, the unnormalized probability approximates the conditional probability $p(\mathbf{x}_{\bar{S}}|\mathbf{x}_S)$ after divided by the partition function. The energy network is designed as an unnormalized Multilayer Perceptron (MLP) with skip connections~\cite{He_2016_CVPR}. The partition function, on the other hand, needs to be approximated by Monte Carlo integration from the proposal distribution $q(\mathbf{x}_{\bar{S}}|\mathbf{x}_S)$. The proposal distribution can be estimated using an autoregressive approach~\cite{strauss2021arbitrary}, however, its estimation accuracy will be negatively affected by feature orderings\footnote{Feature ordering refers to the arrangement of input features in a specific sequence when fed into a neural network.}. Instead, a GRU network with a dynamic masking scheme is used, which outputs parameters related to the estimated proposal distribution and a context vector $\boldsymbol{\gamma}$. The context vector is used to connect the energy network and the GRU network. With the conditional probability $p(\mathbf{x}_{\bar{S}}|\mathbf{x}_S)$ obtained, the contribution function can be estimated from Eq.(\ref{eq-3-0}), and Shapley value can be subsequently obtained from Eq.(\ref{eq-2-1}). This network structure is under the framework of energy-based model, which can universally approximate any probability density according to Ref.~\cite{Arbel2021Gebm}. Hence the proposed model structure inherits the good approximation ability of EBM. It should be noted that for estimating the proposal distribution, other RNN variants, such as Long Short-Term Memory (LSTM), BiGRU, and Attention-based GRU (Att-GRU), can also be applied. However, the more complex gating structure of LSTM often leads to higher computational costs and an increased risk of overfitting. On the other hand, advanced variants such as BiGRU and attention-augmented GRU (Att-GRU) require substantially more time and memory during training and inference, while the observed performance improvement may be marginal. This is because the introduction of bidirectional or attention mechanisms may introduce redundant information that does not substantially enhance feature attribution, and the added complexity can exacerbate the risk of overfitting. Hence, the GRU is selected over other RNN variants, as it represents a more balanced choice that offers a favorable trade-off between performance and computational efficiency. This conclusion is further supported by the comparative studies in Section~\ref{Sec-6-5}. By capturing the intrinsic information associated with the sequential masking of variables, the GRU network ensures that the generated proposal distribution is more accurate, thereby enabling the model to better approximate complex conditional probability densities. The details of the network structure will be explained in Subsections~\ref{Sec-4-0} and~\ref{Sec-4-2}.

\subsection{Energy network for estimation of unnormalized conditional probability density}\label{Sec-4-0}
For estimation of Shapley values, EmSHAP relies on the energy-based model. EBMs are a class of probabilistic models with the probability of a configuration $\mathbf{x}$ defined through an energy function $g_{\theta}(\mathbf{x})$~\cite{lecun2006tutorial}. EBMs have strong universal approximation capability~\cite{zhang2022identifiable, khemakhem2020ice}, enabling them to capture high-dimensional and multi-modal distributions without restrictive assumptions on the data. Formally, the data distribution can be approximated as follows.
\begin{equation}\label{eq-2-5}
\begin{small}
    p_{data}(\mathbf{x}) \approx p_{\theta}(\mathbf{x}) = \frac{\exp(-g_{\theta}(\mathbf{x}))}{Z},
\end{small}
\end{equation}
where $Z=\int \exp(-g_{\theta}(\mathbf{x})) d\mathbf{x}$ is the partition function.

A key advantage of EBMs lies in their flexibility: by parameterizing the energy function $g_{\theta}(\mathbf{x})$ with neural networks, they can serve as a powerful framework for modeling complex probability distributions. In particular, this formulation allows EBMs to represent both joint and conditional distributions. Extending Eq.(\ref{eq-2-5}) to the conditional case yields,
\begin{equation}\label{eq-3-1}
\begin{small}
    p(\mathbf{x}_{\bar{S}}|\mathbf{x}_{S}) \approx \frac{\exp(-g_{\theta}(\mathbf{x}_{\bar{S}},\mathbf{x}_{S}))}{Z_{\bar{S};\mathbf{x}_S}},
\end{small}
\end{equation}
where $\bar{S}$ denotes the unobserved subset and $S$ the observed subset of features. This conditional formulation is essential for Shapley value estimation, as it requires repeatedly modeling the conditional distribution of missing features given observed ones. Here, $Z_{\bar{S};\mathbf{x}_S}=\int \exp(-g_{\theta}(\mathbf{x}_{\bar{S}},\mathbf{x}_S))d\mathbf{x}_{\bar{S}}$ is the partition function specific to the subset $\bar{S}$ and observed features $\mathbf{x}_S$.

In our implementation, $g_{\theta}(\mathbf{x}_{\bar{S}},\mathbf{x}_S)$ is parameterized by a four-layer MLP with two residual connections~\cite{He_2016_CVPR}, and can be regarded as an energy network. In the energy network, the skip connections help stabilize training and improve representation capacity, especially when modeling high-dimensional feature spaces. The input of the energy network consists of both the observed features $\mathbf{x}_S$ and the masked (unobserved) features $\mathbf{x}_{\bar{S}}$, allowing the network to capture interactions between observed and masked features. The energy network for estimating the unnormalized conditional probability density is as follows.
\begin{equation}
\begin{small}
    g_{\theta}(\mathbf{x}_S, \mathbf{x}_{\bar{S}}) = \text{MLP}([\mathbf{x}_S, \mathbf{x}_{\bar{S}}]).
\end{small}
\end{equation}

The energy network outputs a scalar energy value, where a lower energy corresponds to a higher probability. This scalar output is used to calculate the unnormalized probability $\exp (-g_{\theta}(\mathbf{x}_{\bar{S}}, \mathbf{x}_{S}))$. After dividing this unnormalized probability by the partition function $Z_{\bar{S};\mathbf{x}_S}$, the normalized conditional probability of the unobserved features given the observed ones can be obtained.

\subsection{GRU network for estimating the proposal distribution}\label{Sec-4-2}
When using the EBM to approximate the conditional density $p(\mathbf{x}_{\bar{S}}|\mathbf{x}_S)$, it is always intractable to directly calculate $Z_{\bar{S}, \mathbf{x}_S}$. One solution is to introduce a proposal conditional distribution $q(\mathbf{x}_{\bar{S}}|\mathbf{x}_S)$ that is close to $Z_{\bar{S}, \mathbf{x}_S}$, then use Monte Carlo integration to sample from $q(\mathbf{x}_{\bar{S}}|\mathbf{x}_S)$ to approximate $Z_{\bar{S}, \mathbf{x}_S}$.
\begin{equation}\label{eq-3-2}
\begin{small}
\begin{aligned}
      \hat{Z}_{\bar{S};\mathbf{x}_S} &= \int \text{exp}(-g_{\theta}(\mathbf{x}_{\bar{S}},\mathbf{x}_S))d\mathbf{x}_{\bar{S}} \\
      &= \int \frac{\text{exp}(-g_{\theta}(\mathbf{x}_{\bar{S}},\mathbf{x}_S))}{q(\mathbf{x}_{\bar{S}}|\mathbf{x}_S)}q(\mathbf{x}_{\bar{S}}|\mathbf{x}_S)d\mathbf{x}_{\bar{S}}  \\
      & \approx \frac{1}{\tilde{K}}\sum_{\tilde{k}=1}^{\tilde{K}}\frac{\text{exp}(-g_{\theta}(\mathbf{x}_{\bar{S}}^{(\tilde{k})},\mathbf{x}_S))}{q(\mathbf{x}_{\bar{S}}^{(\tilde{k})}|\mathbf{x}_S)}, \mathbf{x}_{\bar{S}}^{(\tilde{k})} \sim q(\mathbf{x}_{\bar{S}}|\mathbf{x}_S).
\end{aligned}
\end{small}
\end{equation}
Here $\hat{Z}_{{\bar{S}};\mathbf{x}_S}$ is the approximate partition function, $\tilde{K}$ is the number of Monte-Carlo sampling. The above procedure is also called importance sampling. When $q(\mathbf{x}_{\bar{S}}|\mathbf{x}_S)$ and $p(\mathbf{x}_{\bar{S}}|\mathbf{x}_S)$ are close, such a sampling estimation can be regarded as unbiased. Thus, an appropriate proposal distribution $q(\mathbf{x}_{\bar{S}}|\mathbf{x}_S)$ should be defined, which is desired to be close to $p(\mathbf{x}_{\bar{S}}|\mathbf{x}_S)$ as much as possible and have a parametric expression that is easy to be inferred.

The proposal distribution $q(\mathbf{x}_{\bar{S}}|\mathbf{x}_S)$ can be obtained through an autoregressive approach~\cite{strauss2021arbitrary},

\begin{equation}\label{eq-3-7}
\begin{small}
q(\mathbf{x}_{\bar{S}}|\mathbf{x}_S) \approx \prod_{j=1}^{|{\bar{S}}|}q(x_{{\bar{S}_j}}|\mathbf{x}_{S \cup \bar{S}_{<j}}),
\end{small}
\end{equation}
where $j=1,\dots, |{\bar{S}}|$ is the index of features in $|S|$. Eq.(\ref{eq-3-7}) uses the chain rule of conditional probability to obtain the proposal distribution, so that estimation of the proposal distribution can be converted into the multiplication of a series of one-dimensional conditional densities $q(x_{{\bar{S}}_l}|\mathbf{x}_{S \cup \bar{S}_{<l}})$, which is easy and tractable.

An important issue arises here, namely, the feature ordering of the conditional probabilities $q(x_{{\bar{S}}_l}|\mathbf{x}_{S \cup \bar{S}_{<l}})$ will have a significant impact on the estimation results. For example, for a predictive task with three inputs $x_1$, $x_2$ and $x_3$, when calculating the conditional density $q(x_2,x_3|x_1)$, it can be obtained as $q(x_2,x_3|x_1)=q(x_3|x_1,x_2)q(x_2|x_1)$ or $q(x_2,x_3|x_1)=q(x_2|x_1,x_3)q(x_3|x_1)$ according to different feature orderings. Theoretically, $q(x_3|x_1,x_2)q(x_2|x_1)$ is equal to $q(x_2|x_1,x_3)q(x_3|x_1)$, however, since we use the energy-based model to approximate these distributions, the two orderings may produce different results. This is problematic if the number of input features is large. Though Ryan \emph{et al.}~\cite{strauss2021arbitrary} used plenty of random training to alleviate this problem, the solution is still unsatisfactory.

In order to better estimate the conditional distribution and subsequently the Shapley value, a GRU network is applied here. Let $|{\bar{S}}|$ denote the number of masked variables in $\mathbf{x}_{\bar{S}}$ and $x_{\bar{S}_j}$ be the $j$-th ($j=1,\dots, |{\bar{S}}|$) masked variable, the conditional density $q(\mathbf{x}_{\bar{S}}|\mathbf{x}_S)$ can be approximated by multiplying the density functions of $|{\bar{S}}|$ Gaussian distributed auxiliary variables $\mathbf{z}_j$, whose mean vector and covariance matrix can be generated through the GRU network in Fig.~\ref{fig2}.

\begin{figure}[h]
\centerline{\includegraphics[width=\linewidth]{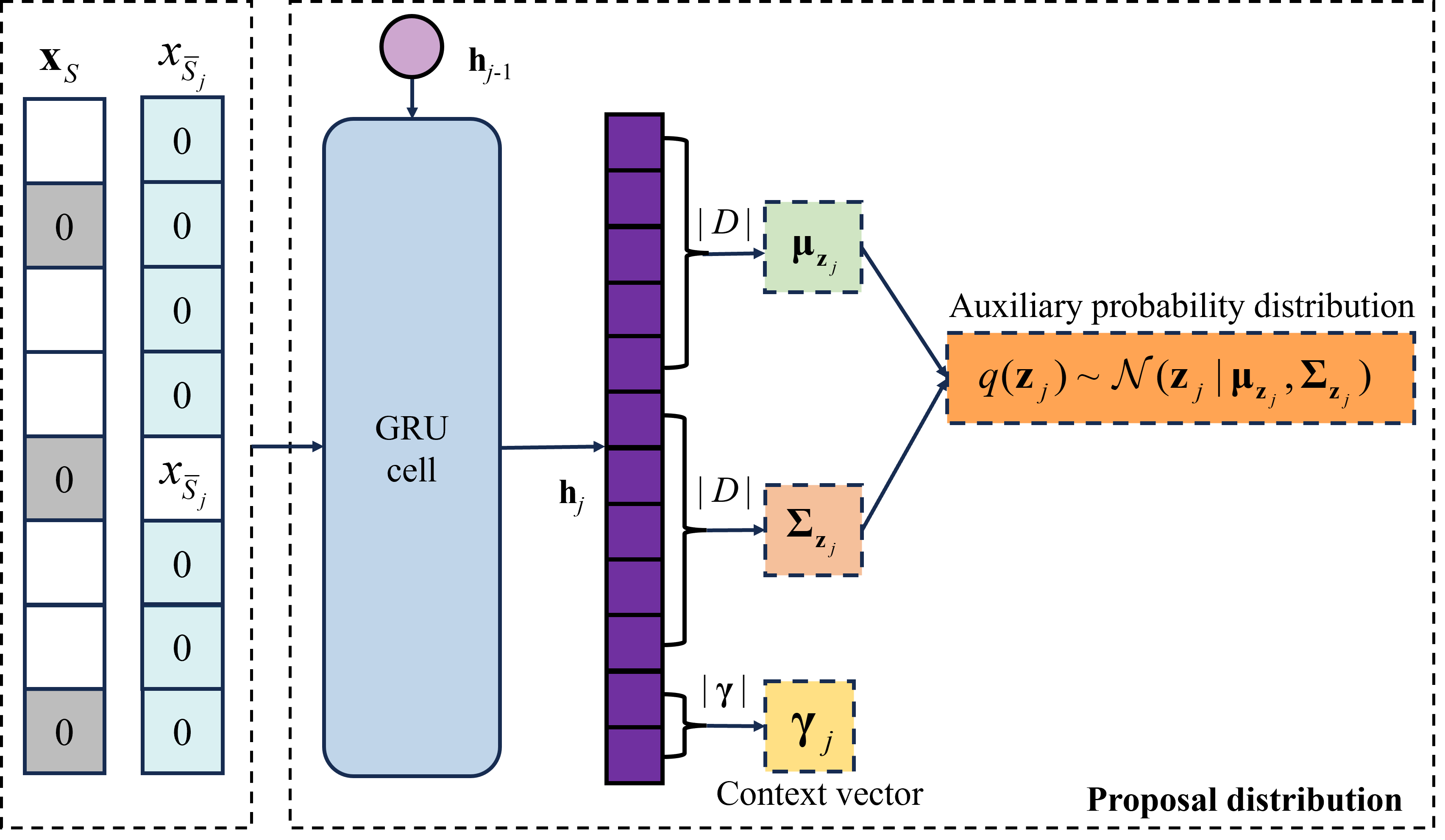}}
\caption{GRU network for proposal distribution estimation.}
\label{fig2}
\end{figure}
In Fig.~\ref{fig2}, the unmasked features $\mathbf{x}_{S}$ together with the $j$-th masked feature $x_{\bar{S}_j}$ are fed into the GRU cell to generate the $j$-th hidden state $\mathbf{h}_j$ based on the information of the $j-1$-th hidden state $\mathbf{h}_{j-1}$, which can be summarized as follows.
\begin{equation}\label{eq-3-8}
    \mathbf{h}_j = \text{GRU}\left( \left[
        \mathbf{x}_{S}, x_{\bar{S}_j}\right], \mathbf{h}_{j-1} \right).
\end{equation}

It should be noted that a total of $|D|-1$ zeros are combined with $x_{\bar{S}_j}$ in Fig.~\ref{fig2} to align with the dimension of $\mathbf{x}_{S}$. The hidden state has a dimension of $2|D|+|\boldsymbol{\gamma}|$, with the first $|D|$ elements forming the mean vector of the auxiliary variables $\mathbf{z}_j$, the second $|D|$ elements forming the diagonal vector of $\mathbf{z}_j$'s covariance matrix, and the last $|\boldsymbol{\gamma}|$ elements forming the context vector $\boldsymbol{\gamma}_j$. The hidden state $\mathbf{h}_j$ can be denoted as follows.
\begin{equation}\label{eq-3-8-0}
    \left[ \boldsymbol{\mu}_{\mathbf{z}_j}^T, \boldsymbol{\sigma}_{\mathbf{z}_j}^T, \boldsymbol{\gamma}_j^T \right]^T = \mathbf{h}_j.
\end{equation}
Here $\boldsymbol{\mu}_{\mathbf{z}_j}$ is the mean vector, $\boldsymbol{\sigma}_{\mathbf{z}_j}$ is the vector containing the diagonal elements of $\mathbf{z}_j$'s covariance matrix. The context vector $\boldsymbol{\gamma}_j$ is used to connect the GRU network with the energy network by transferring the contextual information, thereby facilitating the interaction between the two components and enhancing the overall performance of conditional probability estimation. It does not explicitly involves in feature attribution, but rather provides insight into how feature dependencies are captured and transferred within EmSHAP. Table E2 of Appendix E analyzed how the length of $\boldsymbol{\gamma}$ impacts the estimation accuracy.

The auxiliary variables $\mathbf{z}_{j}$ have a Gaussian distribution as follows.
\begin{equation}\label{eq-3-8-1}
    q (\mathbf{z}_j) \sim \mathcal{N}(\mathbf{z}_j|\boldsymbol{\mu}_{\mathbf{z}_j}, \boldsymbol{\Sigma}_{\mathbf{z}_j}).
\end{equation}
Here $\boldsymbol{\Sigma}_{\mathbf{z}_j}$ is the diagonal covariance matrix with diagonal elements being $\boldsymbol{\sigma}_{\mathbf{z}_j}$.

The masked feature $x_{\bar{S}_j}$ is fed into the GRU cell one by one to yield a series of auxiliary variable $\mathbf{z}_j$, with its distribution generated by $\mathbf{h}_j$. Benefiting from the inherent gating mechanism of the GRU network, $\mathbf{h}_j$ consistently incorporates information from both previously masked features and the current masked feature. Consequently, the product of the densities of all generated auxiliary variables can be used to approximate the desired conditional probability distribution.
\begin{equation}\label{eq-3-8-2}
    q(\mathbf{x}_{\bar{S}}|\mathbf{x}_{S} ) \approx \prod_{j=1}^{|\bar{S}|}q(\mathbf{z}_{j}).
\end{equation}
Throughout the training process, whenever a masked feature is fed into the GRU network, the corresponding information is captured and stored in the hidden state $\mathbf{h}$. The hidden state $\mathbf{h}$, on the other hand, can be globally optimized by minimizing the loss function, making it insensitive to the orderings of feature masking. Hence, compared to Eq.(\ref{eq-3-7}), the introduction of the hidden state $\mathbf{h}$ allows GRU to better capture global dependencies among features, which significantly reduces the impact of feature ordering.

Another notable issue here is that Gaussian distribution is used as the prior distribution of the auxiliary variables. Although it is not a perfect prior for complex data or models, it is widely used in importance sampling. According to Ref.~\cite{Importance2010Tokdar}, importance sampling can effectively approximate the partition function as long as the proposal distribution has a similar shape or overlapping region with the real distribution, and Gaussian distribution well fits this condition.

Eq.(\ref{eq-3-8})-Eq.(\ref{eq-3-8-2}) naturally lead to the GRU network in Figs.~\ref{fig1} and~\ref{fig2}. By substituting the obtained proposal distribution $q(\mathbf{x}_{\bar{S}}|\mathbf{x}_S)$ into Eq.(\ref{eq-3-2}), the approximate partition function $\hat{Z}_{\bar{S};\mathbf{x}_S}$ can be obtained. This partition function is then incorporated into Eq.(\ref{eq-3-1}) to derive the conditional probability density $p(\mathbf{x}_{\bar{S}}|\mathbf{x}_S)$ based on the EBM. Finally, substituting $p(\mathbf{x}_{\bar{S}}|\mathbf{x}_S)$ into Eq.(\ref{eq-3-0}) yields an approximate expression for the contribution function $v(S)$. These derivations provide the foundation for estimating the desired quantities using EmSHAP. With the proposal distribution $q(\mathbf{x}_{\bar{S}}|\mathbf{x}_S)$ determined, the loss function of EmSHAP can be established according to the maximum likelihood criterion.
\begin{equation}\label{eq-3-6}
\begin{small}
\begin{aligned}
    & \quad \mathcal{J}({\Theta};\mathbf{x})= -\mathcal{L}({\Theta};\mathbf{x})
    \\ &
    = -\frac{1}{\tilde{K}} \sum_{\tilde{k}=1}^{\tilde{K}} \log p(\mathbf{x}^{(\tilde{k})}_{\bar{S}}|\mathbf{x}_S, {\Theta}) - \text{log}q(\mathbf{x}_{\bar{S}}|\mathbf{x}_S, {\Theta}).
\end{aligned}
\end{small}
\end{equation}
Here ${\Theta}$ is the parameter set. Using the GRU network structure in Fig.~\ref{fig2}, the conditional densities for any combination of $\mathbf{x}_{\bar{S}}$ given $\mathbf{x}_S$ can be obtained so that a single model can estimate all the conditional densities and repetitive operation of conditional density estimations can be avoided.

\subsection{Dynamic masking scheme}\label{Sec-4-3}
Calculation of Shapley values requires estimating the conditional expectations of the contribution function under different feature combinations. This involves repeatedly evaluating the impact on the model output by including or excluding specific input features. The masking scheme provides an effective way to perform this evaluating procedure. In the masking scheme, a Bernoulli distribution is used to determine whether an input feature is included or excluded, which can be described as follows.
\begin{equation}
    \begin{gathered}
        b_i \sim Bernoulli(\zeta), \\
        \tilde{x}_i = x_i \cdot(1-b_i).
    \end{gathered}
\end{equation}
Here, $\zeta$ is the masking probability, and $b_i$ is a binary value generated from the Bernoulli distribution, with $b_i=0$ indicating $x_i$ being included in the feature set $S$, otherwise, $x_i$ is masked and excluded from the feature set $S$. The masking probability $\zeta$ is an important parameter that needs to be manually tuned, for example, $\zeta$ can be set as 0.5 by trial and error. However, using a fixed masking probability is problematic, as it may not generalize to a wide range of potential feature combinations, especially when the number of possible feature combinations is very large. Here, a dynamic masking scheme is introduced by gradually increasing $\zeta$, for example, from 0.2 to 0.8, to cover a wider range of feature combinations. Another benefit of using the dynamic masking scheme is that there is no need for manually tuning the optimal masking rate.

In the dynamic masking scheme, the $\zeta$ value is linearly increased during the training epochs of the model. This ensures EmSHAP to focus on the simpler scenarios with fewer features being masked in the initial stage, and progressively transiting to more complex scenarios with a high proportion of features being masked. This process can be formulated as follows.

\begin{equation}\label{eq-3-11}
\begin{small}
    \zeta_{e+1}=\zeta_{e}+\Delta.
\end{small}
\end{equation}
Here $e$ is the epoch index in model training and $\Delta$ is the increment of $\zeta$ at each epoch. As the training epoch grows larger, the masking rate increases linearly from a preset initial value of $\zeta_{min}$ until it reaches the preset maximum masking rate $\zeta_{max}$.

\begin{algorithm}
\caption{Training procedures of EmSHAP}
\label{alg-1}
\begin{algorithmic}
\REQUIRE Training data $\mathbf{x}$, training epochs $\mathcal{E}$
\STATE \textbf{Initialize:} Energy network parameters, GRU network parameters, and dynamic masking scheme, $e=1$.
\WHILE{$e \leq \mathcal{E}$}
\STATE $[\mathbf{x}_{\bar{S}}, \mathbf{x}_S]=$ dynamic masking$(\mathbf{x})$
\FOR{$j=1$ to $|\bar{S}|$}
\STATE $\mathbf{h}_j = \text{GRU}( [
        \mathbf{x}_{S}, x_{\bar{S}_j}], \mathbf{h}_{j-1} )$
\STATE $[ \boldsymbol{\mu}_{\mathbf{z}_j}^T, \boldsymbol{\sigma}_{\mathbf{z}_j}^T, \boldsymbol{\gamma}_j^T ]^T = \mathbf{h}_j$
\STATE $q (\mathbf{z}_j) \sim \mathcal{N}(\mathbf{z}_j|\boldsymbol{\mu}_{\mathbf{z}_j}, \boldsymbol{\Sigma}_{\mathbf{z}_j})$
\ENDFOR
\STATE $q(\mathbf{x}_{\bar{S}}|\mathbf{x}_{S} ) = \prod_{j=1}^{|\bar{S}|}q(\mathbf{z}_{j})$
\STATE Sample from $q(\mathbf{x}_{\bar{S}}|\mathbf{x}_S)$ to get $\mathbf{x}_{\bar{S}}^{(\tilde{k})}$
\STATE energy network $g_{\theta}(\mathbf{x}_{\bar{S}},\mathbf{x}_{S}) =\text{MLP}([\mathbf{x}_{\bar{S}}^{(\tilde{k})}, \mathbf{x}_S], \boldsymbol{\gamma}_{|\bar{S}|})$
\STATE Calculate partition function using Eq.(\ref{eq-3-2})
\STATE Calculate loss using Eq.(\ref{eq-3-6})
\STATE $e = e+1$
\ENDWHILE
\end{algorithmic}
\end{algorithm}

By applying the dynamic masking scheme, the GRU network recursively processes the masked input features to capture dependencies among variables. At each step, GRU updates its hidden state by incorporating historical masking information and generates the conditional probability estimate for the current step. This process enables GRU to leverage its internal gating mechanisms to retain and propagate prior masking information, thus effectively capturing feature dependencies without relying on a fixed feature order.
The detailed procedures for training EmSHAP are listed in Algorithm~\ref{alg-1}, and the pipeline of using EmSHAP to calculate Shapley values is illustrated in Fig.~\ref{fig2-0}.

\begin{figure}[h]
\centering
\includegraphics[width=\linewidth]{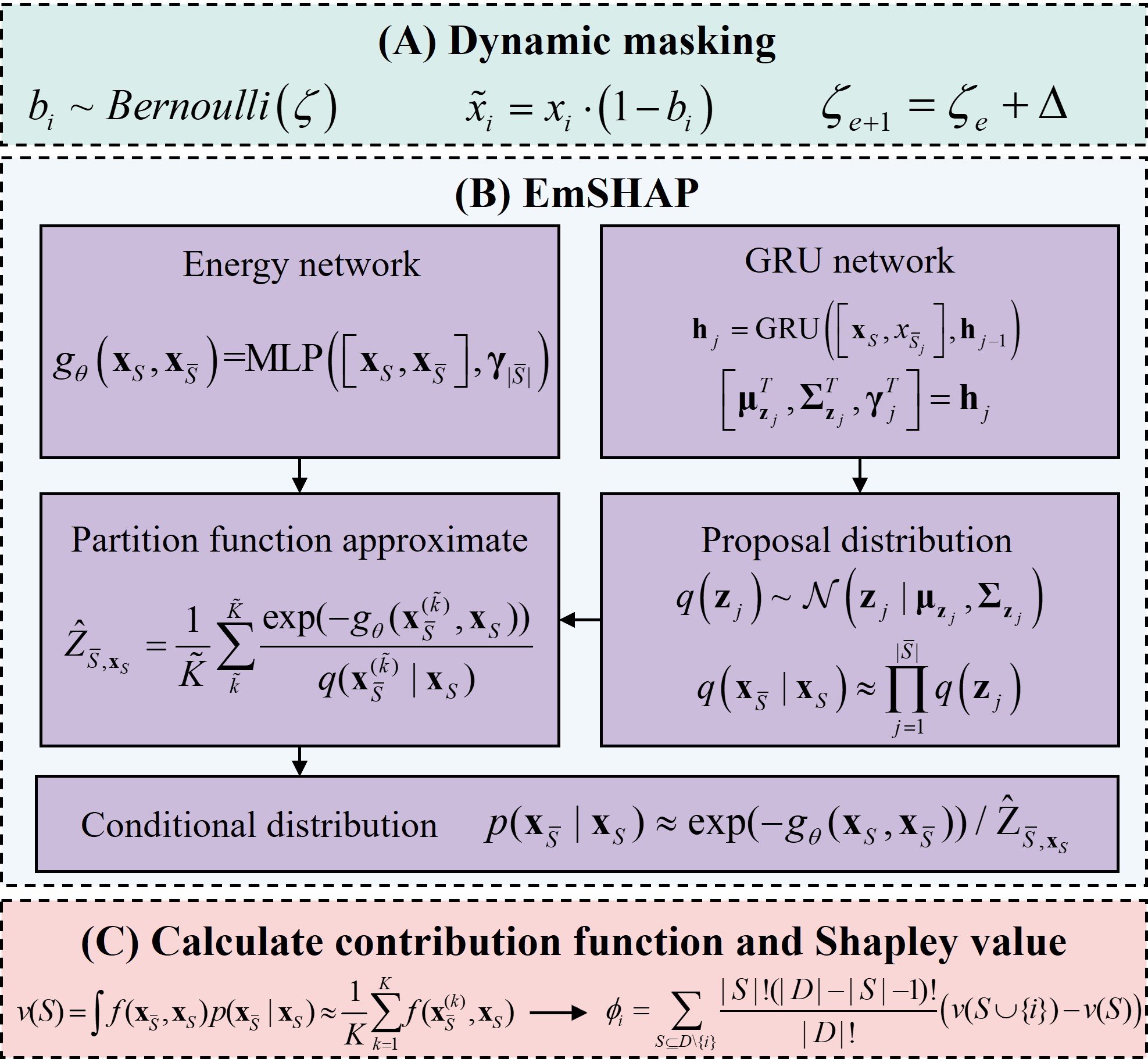}
\caption{Pipeline of EmSHAP for estimation of Shapley values.}
\label{fig2-0}
\end{figure}

\subsection{Computational complexity analysis of EmSHAP}
This subsection analyzes the computational complexity of EmSHAP. The computational burden of EmSHAP primarily arises from two components: the energy network and the GRU network. The energy network is implemented as a fully connected network with two residual connections, with a complexity of $O((2|D|+|\boldsymbol{\gamma}|)\cdot H + H^2)$, where $H$ is the number of latent dimensions. The GRU network consists of a single GRU layer, whose complexity is $O(L \cdot (2|D|+|\boldsymbol{\gamma}|)^2 + L \cdot (2|D|+|\boldsymbol{\gamma}|) \cdot 2|D|)$, where $L$ is the sequence length. Hence, the overall computational complexity of EmSHAP can be approximated as $O(L \cdot (2|D|+|\boldsymbol{\gamma}|)^2 + L \cdot (2|D|+|\boldsymbol{\gamma}|) \cdot (2|D|+H) + H^2)$. Since EmSHAP evaluates samples individually, the sequence length $L$ is typically set to 1, which simplifies the computation complexity to $O((2|D|+|\boldsymbol{\gamma}|)^2 + (2|D|+|\boldsymbol{\gamma}|) \cdot (2|D|+H) + H^2)$. It can be seen that the computational complexity is determined by the input feature dimension $|D|$, the latent dimension $H$, and context vector length $|\boldsymbol{\gamma}|$.

\section{Analysis on Error bound}\label{Sec-5}
In this section, the error bound of the contribution function in EmSHAP is discussed. Theorem 1 indicates that the expectation of the mean absolute deviation (MAD) between the estimated contribution function $\hat{v}(S)$ and the true contribution function $v(S)$ has an upper bound of $\frac{\sqrt{\pi}}{\sqrt{2K}}+\epsilon_2$, where $K$ is the number of Monte Carlo sampling, $\epsilon_2$ is a small positive number related to the absolute probability density ratio between the estimated conditional density and the true data-generating density. For comparison, Theorems 2 and 3 present the error bounds of two competitive methods, i.e., KernelSHAP~\cite{covert2020improving} and VAEAC~\cite{olsen2022using}.

To begin with, the following assumptions are introduced to facilitate the analysis.
\begin{assumption}\label{asp-1}
    The parameters of the estimated conditional density $p_\theta(\mathbf{x}_{\bar{S}}|\mathbf{x}_S=\mathbf{x}_S^t)$ are unbiased estimates of the parameters of the true data-generating density $p_{data}(\mathbf{x}^t_{\bar{S}}|\mathbf{x}_S=\mathbf{x}_S^t)$, that is, for a given positive $\epsilon_1$, the probability density ratio can be described as follows,
    \begin{equation}
    \begin{small}
        \left|\frac{p_\theta(\mathbf{x}_{\bar{S}}|\mathbf{x}_S=\mathbf{x}_S^t)} {p_{data}(\mathbf{x}^t_{\bar{S}}|\mathbf{x}_S=\mathbf{x}_S^t)} -1 \right| \leq \epsilon_1.
    \end{small}
    \end{equation}
\end{assumption}

\begin{assumption}\label{asp-2}
    The model $f(\mathbf{x})$ is continuous and second-order differentiable. Once $\mathbf{x}_{\bar{S}}$ estimated from EmSHAP is close to the real value $\mathbf{x}_{\bar{S}}^t$, the Taylor expansion of $f(\mathbf{x}_{\bar{S}}, \mathbf{x}_{S}^t)$ on $f(\mathbf{x}_{\bar{S}}^t, \mathbf{x}_{S}^t)$ can be expressed as follows.
    \begin{equation}
        f(\mathbf{x}_{\bar{S}}, \mathbf{x}_{S}^t) = f(\mathbf{x}_{\bar{S}}^t, \mathbf{x}_{S}^t) + f^{\prime}(\mathbf{x}_{\bar{S}}^t, \mathbf{x}_{S}^t)(\mathbf{x}_{\bar{S}}-\mathbf{x}_{\bar{S}}^t) + o(\mathbf{x}_{\bar{S}}-\mathbf{x}_{\bar{S}}^t),
    \end{equation}
    where $o(\mathbf{x}_{\bar{S}}-\mathbf{x}_{\bar{S}}^t)\leq \epsilon_1^{\prime}$, $\epsilon_1^{\prime}$ is a small positive number.
\end{assumption}

\begin{assumption}\label{asp-3}
    The contribution value $v(S)$ of any subset $S$ and the output value $\hat{y}=f(\mathbf{x})$ are bounded within $[0,1]$.
\end{assumption}
Assumption~\ref{asp-1} holds as EBM can be flexibly parameterized for arbitrary probability distributions~\cite{zhai16, Grathwohl2020Your}. More specifically, the proposed GRU coupled energy-based model consists of an energy network and a GRU network. The energy network is an MLP with a skip connection and the GRU network is a type of gated recurrent network. According to Theorem 2.5 in Ref.\cite{ma2020rademacher}, with weighted path norm uniformly bounded, by setting appropriate network width and depth, MLP with skip connection can approximate any function in Barron space\footnote{Barron space is a function space that consists of functions which can be approximated by smooth combinations of simple basis functions}. On the other hand, Theorem 3.2 in Ref.\cite{kusupati2018fastgrnn} shows the universal approximation ability of GRU.

\begin{remark}
In practical applications, Assumption~\ref{asp-1} may not hold strictly, particularly in high-dimensional or non-stationary data settings where conditional density estimation is inherently challenging. For high-dimensional case, increasing the amount of training data can alleviate this problem to some extent. For non-stationary data, Assumption~\ref{asp-1} will be violated and the estimation accuracy will deteriorate if proper measures are not considered. A potential way to mitigate this scenario is to introduce standard preprocessing techniques such as differencing or detrending, another solution is to develop an adaptive version of EmSHAP. Furthermore, Assumption~\ref{asp-3} enforces boundedness of $v(S)$ and $\hat{y}$ within $[0, 1]$, which simplifies the theoretical error bound. In practice, this assumption can be easily achieved through normalization or scaling to ensure that the values remain within the specified range.
\end{remark}

Before analyzing the upper error bound of EmSHAP, we first demonstrate the asymptotic convergence of $\mathbb{E}_{p_{\theta}(\mathbf{x}_{\bar{S}}|\mathbf{x}_S=\mathbf{x}_S^t)}[ f(\mathbf{x}_{\bar{S}},\mathbf{x}_S^t) ]$ and $\mathbb{E}_{p_{data}(\mathbf{x}^t_{\bar{S}}|\mathbf{x}_S=\mathbf{x}_S^t)}[ f(\mathbf{x}^t_{\bar{S}},\mathbf{x}_S^t) ]$, followed by the derivation of their error bounds.

\subsection{Asymptotic convergence of $\mathbb{E}_{p_{\theta}(\mathbf{x}_{\bar{S}}|\mathbf{x}_S=\mathbf{x}_S^t)}[ f(\mathbf{x}_{\bar{S}},\mathbf{x}_S^t) ]$ and $\mathbb{E}_{p_{data}(\mathbf{x}^t_{\bar{S}}|\mathbf{x}_S=\mathbf{x}_S^t)}[ f(\mathbf{x}^t_{\bar{S}},\mathbf{x}_S^t) ]$}\label{Sec-5-1}
\begin{proposition}\label{pro-1}
(Consistency) From the weak law of large numbers, when the sampling time $K$ is large enough, $\frac{1}{K}\sum_{k=1}^K f(\mathbf{x}_{\bar{S}}^{(k)},\mathbf{x}_S^t)$ is convergent by probability to the conditional expectation $\mathbb{E}_{p_{\theta}(\mathbf{x}_{\bar{S}}|\mathbf{x}_S=\mathbf{x}_S^t)}[f(\mathbf{x}_{\bar{S}},\mathbf{x}_S^t)]$, that is, for any positive $\epsilon_2$, we have,
\end{proposition}
\begin{equation}\label{eq-4-1}
\begin{small}
    \begin{aligned}
    \mathop{\lim}_{K \rightarrow +\infty} \mathbb{P} \left( \left| \frac{1}{K}\sum_{k=1}^K f(\mathbf{x}_{\bar{S}}^{(k)},\mathbf{x}_S^t)-   \mathbb{E}_{p_{\theta}(\mathbf{x}_{\bar{S}}|\mathbf{x}_S=\mathbf{x}_S^t)} \left[f(\mathbf{x}_{\bar{S}},\mathbf{x}_S^t) \right] \right| \leq \epsilon_2 \right)
     = 1.
    \end{aligned}
\end{small}
\end{equation}
\begin{proof}
    Let $\bar{f}(\mathbf{x}^{(k)}_{\bar{S}},\mathbf{x}_S^t)=\frac{1}{K}\sum_{k=1}^K f(\mathbf{x}_{\bar{S}}^{(k)},\mathbf{x}_S^t)$ and $\text{Var}(f(\mathbf{x}_{\bar{S}}^{(k)},\mathbf{x}_S^t))$ $=\sigma$ for all $k$, the variance of $\bar{f}(\mathbf{x}^{(k)}_{\bar{S}},\mathbf{x}_S^t)$ is as follows.
    \begin{equation}
    \begin{small}
    \begin{aligned}
         &\quad \text{Var}\left( \bar{f}(\mathbf{x}^{(k)}_{\bar{S}},\mathbf{x}_S^t) \right)
         = \text{Var}\left( \frac{1}{K}\sum_{k=1}^K f(\mathbf{x}_{\bar{S}}^{(k)},\mathbf{x}_S^t) \right)
         \\& = \frac{1}{K^2}\text{Var}\left( \sum_{k=1}^K f(\mathbf{x}_{\bar{S}}^{(k)},\mathbf{x}_S^t) \right)
         = \frac{K\sigma^2}{K^2} = \frac{\sigma^2}{K}.
    \end{aligned}
    \end{small}
    \end{equation}
    By using Chebyshev's inequality~\cite{Alsmeyer2011} on $\bar{f}(\mathbf{x}^{(k)}_{\bar{S}},\mathbf{x}_S^t)$, we have,

    \begin{equation}
    \begin{small}
            \mathbb{P}\left(\left| \bar{f}(\mathbf{x}^{(k)}_{\bar{S}},\mathbf{x}_S^t) - \mathbb{E}_{p_{\theta}(\mathbf{x}_{\bar{S}}|\mathbf{x}_S=\mathbf{x}_S^t)} f(\mathbf{x}_{\bar{S}},\mathbf{x}_S^t) \right| \leq \epsilon_2 \right) \geq 1-\frac{\sigma^2}{K\epsilon_2^2}.
    \end{small}
    \end{equation}
    As $K$ approaches infinity, Eq.(\ref{eq-4-1}) can be obtained.
\end{proof}

\begin{proposition}\label{pro-2}
The expectation of estimated conditional density $p_\theta(\mathbf{x}_{\bar{S}}|\mathbf{x}_S=\mathbf{x}_S^t)$ is convergent by probability to the expectation of true data-generating density $p_{data}(\mathbf{x}^t_{\bar{S}}|\mathbf{x}_S=\mathbf{x}_S^t)$, that is, for a positive $\epsilon_2$ that is related to $\epsilon_1$ and $\epsilon_1^{\prime}$, we have,
\end{proposition}
\begin{equation}\label{eq-4-2}
\begin{small}
\begin{aligned}
    \mathbb{P} & \left(  \left|  \mathbb{E}_{p_\theta(\mathbf{x}_{\bar{S}}|\mathbf{x}_S=\mathbf{x}_S^t)}\left[ f(\mathbf{x}_{\bar{S}},\mathbf{x}_S^t)\right] - \mathbb{E}_{p_{data}(\mathbf{x}^t_{\bar{S}}|\mathbf{x}_S=\mathbf{x}_S^t)} \left[ f(\mathbf{x}^t_{\bar{S}},\mathbf{x}_S^t) \right] \right| \leq \epsilon_2 \right)=1.
\end{aligned}
\end{small}
\end{equation}
\begin{proof}
    To prove Proposition~\ref{pro-2}, Taylor expansion for the expectation function is used. The absolute value of the difference between the estimated conditional expectation and the true conditional expectation is as follows.
    \begin{equation}\label{eq-4-3}
    \begin{small}
        \begin{aligned}
            &\quad\left| \mathbb{E}_{p_\theta(\mathbf{x}_{\bar{S}}|\mathbf{x}_S=\mathbf{x}_S^t)}\left[ f(\mathbf{x}_{\bar{S}},\mathbf{x}_S^t) \right] - \mathbb{E}_{ p_{data}(\mathbf{x}^t_{\bar{S}}|\mathbf{x}_S=\mathbf{x}_S^t)}\left[ f(\mathbf{x}^t_{\bar{S}},\mathbf{x}_S^t) \right] \right| \\
            & = \left|\mathbb{E}_{p_\theta(\mathbf{x}_{\bar{S}}|\mathbf{x}_S=\mathbf{x}_S^t)}\left[ \left( f(\mathbf{x}^t_{\bar{S}},\mathbf{x}^t_S) + f^{\prime}(\mathbf{x}^t_{\bar{S}},\mathbf{x}^t_S)(\mathbf{x}_{\bar{S}}-\mathbf{x}^t_{\bar{S}})+ o(\mathbf{x}_{\bar{S}}-\mathbf{x}^t_{\bar{S}} )\right) \right] \right.
            \\
            &\quad \left.- \mathbb{E}_{p_{data}(\mathbf{x}^t_{\bar{S}}|\mathbf{x}_S=\mathbf{x}_S^t)}\left[ f(\mathbf{x}^t_{\bar{S}},\mathbf{x}_S^t)\right] \right| \\
            & = \left| \mathbb{E}_{p_\theta(\mathbf{x}_{\bar{S}}|\mathbf{x}_S=\mathbf{x}_S^t)}\left[ f(\mathbf{x}^t_{\bar{S}},\mathbf{x}^t_S)\right]- \mathbb{E}_{p_{data}(\mathbf{x}^t_{\bar{S}}|\mathbf{x}_S=\mathbf{x}_S^t)}\left[ f(\mathbf{x}^t_{\bar{S}},\mathbf{x}_S^t) \right]  \right.\\
            & \quad\left.+\mathbb{E}_{p_\theta(\mathbf{x}_{\bar{S}}|\mathbf{x}_S=\mathbf{x}_S^t)}\left[ f^{\prime}(\mathbf{x}^t_{\bar{S}},\mathbf{x}^t_S)(\mathbf{x}_{\bar{S}}-\mathbf{x}^t_{\bar{S}})+ o(\mathbf{x}_{\bar{S}}-\mathbf{x}^t_{\bar{S}} )\right] \right|.
        \end{aligned}
    \end{small}
    \end{equation}

According to Assumption~\ref{asp-1}, the first term of the last equation in Eq.(\ref{eq-4-3}) is bounded by $\epsilon_1$, that is, $\left|\mathbb{E}_{p_\theta(\mathbf{x}_{\bar{S}}|\mathbf{x}_S=\mathbf{x}_S^t)}\left[ f(\mathbf{x}^t_{\bar{S}},\mathbf{x}^t_S)\right] - \mathbb{E}_{p_{data}(\mathbf{x}^t_{\bar{S}}|\mathbf{x}_S=\mathbf{x}_S^t) }\left[ f(\mathbf{x}^t_{\bar{S}},\mathbf{x}_S^t)\right]\right| \leq \epsilon_1$. In addition, since the parameters of $p_{\theta}(\mathbf{x}_{\bar{S}}|\mathbf{x}_S=\mathbf{x}_{S}^t)$ are unbiased estimates of $p_{data}(\mathbf{x}^t_{\bar{S}}|\mathbf{x}_S=\mathbf{x}_{S}^t)$, we have $ \mathbb{E}\left[ \mathbf{x}_{\bar{S}}-\mathbf{x}^t_{\bar{S}} \right] = 0$ and $\mathbb{E}\left[ o(\mathbf{x}_{\bar{S}}-\mathbf{x}^t_{\bar{S}}) \right] \leq \epsilon_1^{\prime}$.
Thus, for the given smalls positive $\epsilon_1$ and $\epsilon_1^{\prime}$, there always exists a small positive $\epsilon_2$ and Eq.(\ref{eq-4-2}) can be obtained.
\end{proof}

Proposition~\ref{pro-1} and Proposition~\ref{pro-2} imply the convergence of Monte Carlo sampling and the proposed EmSHAP. Furthermore, the following corollary can be obtained.
\begin{corollary}\label{cly-1}
     When the sampling number $K$ goes to infinity, the mean of the samplings converges to the true conditional expectation, for a given positive number $\epsilon_2$, we have,
\end{corollary}
\begin{equation}\label{eq-4-4}
\begin{small}
\begin{aligned}
    \mathop{\lim}_{k \rightarrow \infty} \mathbb{P} \left( \left| \frac{1}{K}\sum_{k=1}^K f(\mathbf{x}_{\bar{S}}^{(k)},\mathbf{x}_S^t) - \mathbb{E}_{p_{data}(\mathbf{x}^t_{\bar{S}}|\mathbf{x}_S=\mathbf{x}_S^t)}\left[ f(\mathbf{x}^t_{\bar{S}},\mathbf{x}_S^t) \right] \right| \leq \epsilon_2 \right) \\
    = 1.
\end{aligned}
\end{small}
\end{equation}
The larger $K$ is, the closer the mean of the samplings is to the true conditional expectation, which indicates that the lower bound of the difference between the estimated conditional expectation and the true conditional function is zero.

\subsection{Error bound of approximating $\mathbb{E}_{p_{\theta}(\mathbf{x}_{\bar{S}}|\mathbf{x}_S=\mathbf{x}_S^t)}[ f(\mathbf{x}_{\bar{S}},\mathbf{x}_S^t) ]$ and $\mathbb{E}_{p_{data}(\mathbf{x}^t_{\bar{S}}|\mathbf{x}_S=\mathbf{x}_S^t)}[ f(\mathbf{x}^t_{\bar{S}},\mathbf{x}_S^t) ]$}\label{Sec-5-2}

Here the error bound of approximating $\mathbb{E}_{p_{\theta}(\mathbf{x}_{\bar{S}}|\mathbf{x}_S=\mathbf{x}_S^t)}\left[ f(\mathbf{x}_{\bar{S}},\mathbf{x}_S^t) \right]$ is first discussed.

\begin{lemma}\label{lemma-1}
(Hoeffding's inequality)~\cite{Hoeffding} Let $u_1$,...$u_K$ be i.i.d. random variables such that $a\leq u_k \leq b$. Then for any $\epsilon_3>0$,
\end{lemma}
\begin{equation}\label{eq-4-10}
\begin{small}
    \mathbb{P}\left( \left| \frac{1}{K} \sum_{k=1}^{K} \left( u_k-\mathbb{E}[u_k] \right) \right| \geq \epsilon_3 \right)\leq 2\exp\left(-\frac{2K\epsilon_3^2}{(b-a)^2}\right).
\end{small}
\end{equation}

Hoeffding's inequality is an important technique for establishing upper bounds on the probability of sums involving bounded random variables deviating significantly from their expected values. By using Lemma~\ref{lemma-1}, the following corollary is obtained.
\begin{corollary}\label{cly-3}
    (corollary of Hoeffding's inequality) For any $\epsilon_3 > 0$, we have,
\end{corollary}
\begin{equation}\label{eq-4-11}
\begin{small}
\begin{aligned}
    \mathbb{P}\left( \left| \frac{1}{K}\sum_{k=1}^K f(\mathbf{x}_{\bar{S}}^{(k)},\mathbf{x}_S^t) - \mathbb{E}_{ p_\theta(\mathbf{x}_{\bar{S}}|\mathbf{x}_S=\mathbf{x}_S^t)}\left[ f(\mathbf{x}_{\bar{S}},\mathbf{x}_S^t)\right] \right| \geq \epsilon_3 \right)\\
    \leq 2\exp\left(-2K\epsilon_3^2\right).
\end{aligned}
\end{small}
\end{equation}
\begin{proof}
    Since the output value is bounded within $[0,1]$ according to Assumption~\ref{asp-3}, by applying Hoeffding$^{\prime}$s inequality, we can get,

    \begin{equation}\label{eq-4-12}
    \begin{small}
    \begin{aligned}
        & \quad \mathbb{P}\left( \left| \frac{1}{K}\sum_{k=1}^{K}\left(
        f(\mathbf{x}_{\bar{S}}^{(k)},\mathbf{x}_S^t)-\mathbb{E}_{p_\theta(\mathbf{x}_{\bar{S}}|\mathbf{x}_S=\mathbf{x}_S^t)}
        [f(\mathbf{x}_{\bar{S}}^{(k)},\mathbf{x}_S^t)] \right) \right|\geq \epsilon_3 \right)  \\
        & = \mathbb{P}\left( \left| \frac{1}{K}\sum_{k=1}^{K}
        f(\mathbf{x}_{\bar{S}}^{(k)},\mathbf{x}_S^t)-\frac{1}{K}\sum_{k=1}^{K}\mathbb{E}_{p_\theta(\mathbf{x}_{\bar{S}}|\mathbf{x}_S=\mathbf{x}_S^t)}
        [f(\mathbf{x}_{\bar{S}}^{(k)},\mathbf{x}_S^t) ] \right|\geq \epsilon_3 \right)  \\
         & \leq 2\exp\left(-2K\epsilon_3^2\right).
    \end{aligned}
    \end{small}
    \end{equation}
    $\mathbf{x}^{(k)}_{\bar{S}}$ is sampled from $p_\theta(\mathbf{x}_{\bar{S}}|\mathbf{x}_S=\mathbf{x}_S^t)$ and the expectation of $f(\mathbf{x}_{\bar{S}}^{(k)},\mathbf{x}_S^t)$ is equal to $\mathbb{E}_{p_\theta(\mathbf{x}_{\bar{S}}|\mathbf{x}_S=\mathbf{x}_S^t)}[ f(\mathbf{x}_{\bar{S}},\mathbf{x}_S^t)]$ for any sample. Thus, Eq.(\ref{eq-4-11}) is obtained.
\end{proof}

The error bound of approximating $\mathbb{E}_{p_{data}(\mathbf{x}^t_{\bar{S}}|\mathbf{x}_S=\mathbf{x}_S^t)}[ f(\mathbf{x}^t_{\bar{S}},\mathbf{x}_S^t)]$ is discussed as follows.
\begin{proposition}\label{prop-1-1}
    For a given small positive $\epsilon_2$, the following inequality can be constructed,
    \begin{equation}\label{eq-4-4-0}
    \begin{small}
    \begin{aligned}
            \mathbb{E}_{p_{data}(\mathbf{x}_{\bar{S}}^t|\mathbf{x}_{S}=\mathbf{x}_{S}^t)}\left[\mathbb{E}\left| f(\mathbf{x}_{\bar{S}},\mathbf{x}_S^t)\frac{p_{\theta}(\mathbf{x}_{\bar{S}}|\mathbf{x}_S=\mathbf{x}_S^t)} {p_{data}(\mathbf{x}_{\bar{S}}^t|\mathbf{x}_S=\mathbf{x}_S^t)} - f(\mathbf{x}^t_{\bar{S}},\mathbf{x}_S^t)\right| \right] \\
            \leq \epsilon_2.
    \end{aligned}
    \end{small}
    \end{equation}
\end{proposition}
\begin{proof}
    The proof of Proposition~\ref{prop-1-1} relies on the Taylor expansion as well, which is described as follows.
    \begin{equation}\label{eq-4-4-1}
    \begin{small}
            \begin{aligned}
                &  \quad \mathbb{E}_{p_{data}(\mathbf{x}_{\bar{S}}^t|\mathbf{x}_{S}=\mathbf{x}_{S}^t)}\left[ \left| f(\mathbf{x}_{\bar{S}},\mathbf{x}_S^t)\frac{p_{\theta}(\mathbf{x}_{\bar{S}}|\mathbf{x}_S=\mathbf{x}_S^t)} {p_{data}(\mathbf{x}_S^t|\mathbf{x}_S=\mathbf{x}_S^t)} - f(\mathbf{x}^t_{\bar{S}},\mathbf{x}_S^t)\right| \right]\\
                & = \mathbb{E}_{p_{data}(\mathbf{x}_{\bar{S}}^t|\mathbf{x}_{S}=\mathbf{x}_{S}^t)}\left[\left| \frac{p_{\theta}(\mathbf{x}_{\bar{S}}|\mathbf{x}_S=\mathbf{x}_S^t)} {p_{data}(\mathbf{x}_{\bar{S}}^t|\mathbf{x}_S=\mathbf{x}_S^t)} \left[ f(\mathbf{x}^t_{\bar{S}},\mathbf{x}_S^t) + f^{\prime}(\mathbf{x}^t_{\bar{S}},\mathbf{x}_S^t)(\mathbf{x}_{\bar{S}} - \mathbf{x}^t_{\bar{S}}) \right.\right.\right.\\
                &\quad\left.\left. \left.+ o(\mathbf{x}_{\bar{S}} - \mathbf{x}^t_{\bar{S}})\right]- f(\mathbf{x}^t_{\bar{S}},\mathbf{x}_S^t) \right| \right]\\
                & =\mathbb{E}_{p_{data}(\mathbf{x}_{\bar{S}}^t|\mathbf{x}_{S}=\mathbf{x}_{S}^t)}\left[ \left| \left(\frac{p_{\theta}(\mathbf{x}_{\bar{S}}|\mathbf{x}_S=\mathbf{x}_S^t)} {p_{data}(\mathbf{x}_{\bar{S}}^t|\mathbf{x}_S=\mathbf{x}_S^t)}-1\right)f(\mathbf{x}^t_{\bar{S}},\mathbf{x}_S^t) \right.\right.\\
                & \quad\left.\left.+ \frac{p_{\theta}(\mathbf{x}_{\bar{S}}|\mathbf{x}_S=\mathbf{x}_S^t)} {p_{data}(\mathbf{x}_{\bar{S}}^t|\mathbf{x}_S=\mathbf{x}_S^t)}\left(f^{\prime}(\mathbf{x}^t_{\bar{S}}, \mathbf{x}_S^t)(\mathbf{x}_{\bar{S}} - \mathbf{x}^t_{\bar{S}}) + o(\mathbf{x}_{\bar{S}} - \mathbf{x}^t_{\bar{S}})\right) \right|\right] \\
                & = \mathbb{E}_{p_{data}(\mathbf{x}_{\bar{S}}^t|\mathbf{x}_{S}=\mathbf{x}_{S}^t)}\left[\left| \left(\frac{p_{\theta}(\mathbf{x}_{\bar{S}}|\mathbf{x}_S=\mathbf{x}_S^t)} {p_{data}(\mathbf{x}_{\bar{S}}^t|\mathbf{x}_S=\mathbf{x}_S^t)}-1\right)f(\mathbf{x}^t_{\bar{S}},\mathbf{x}_S^t) \right.\right.\\ \notag
\end{aligned}
\end{small}
\end{equation}
\begin{equation}
\begin{small}
\begin{aligned}
                & \quad \left.\left.+ o(\mathbf{x}_{\bar{S}} - \mathbf{x}^t_{\bar{S}})\frac{p_{\theta}(\mathbf{x}_{\bar{S}}|\mathbf{x}_S=\mathbf{x}_S^t)} {p_{data}(\mathbf{x}_{\bar{S}}^t|\mathbf{x}_S=\mathbf{x}_S^t)}\right| \right] \\
                &\leq \mathbb{E}_{p_{data}(\mathbf{x}_{\bar{S}}^t|\mathbf{x}_{S}=\mathbf{x}_{S}^t)}\left[\left| \frac{p_{\theta}(\mathbf{x}_{\bar{S}}|\mathbf{x}_S=\mathbf{x}_S^t)} {p_{data}(\mathbf{x}_{\bar{S}}^t|\mathbf{x}_S=\mathbf{x}_S^t)}-1 \right| \cdot \left| f(\mathbf{x}^t_{\bar{S}},\mathbf{x}_S^t) \right|\right.\\
                & \quad \left.+\left| o(\mathbf{x}_{\bar{S}} - \mathbf{x}^t_{\bar{S}})\frac{p_{\theta}(\mathbf{x}_{\bar{S}}|\mathbf{x}_S=\mathbf{x}_S^t)} {p_{data}(\mathbf{x}_{\bar{S}}^t|\mathbf{x}_S=\mathbf{x}_S^t)}\right| \right] \\
                & \leq \epsilon_1 f(\mathbf{x}^t_{\bar{S}},\mathbf{x}_S^t) + \epsilon_1^{\prime} \leq \epsilon_2.
            \end{aligned}
    \end{small}
    \end{equation}

    According to Assumptions~\ref{asp-1} and \ref{asp-2}, $ \mathbb{E}[ \mathbf{x}_{\bar{S}}-\mathbf{x}^t_{\bar{S}}] = 0$ and $\mathbb{E}[ o(\mathbf{x}_{\bar{S}}-\mathbf{x}^t_{\bar{S}})] \leq \epsilon_1^{\prime}$, we have,
    \begin{equation}
    \begin{small}
             \frac{p_{\theta}(\mathbf{x}_S|\mathbf{x}_S=\mathbf{x}_S^t)} {p_{data}(\mathbf{x}_S^t|\mathbf{x}_S=\mathbf{x}_S^t)}\left(f^{\prime}(\mathbf{x}^t_{S}, \mathbf{x}_S^t)(\mathbf{x}_{\bar{S}} - \mathbf{x}^t_{\bar{S}}) + o(\mathbf{x}_{\bar{S}} - \mathbf{x}^t_{\bar{S}})\right) \leq \epsilon_1^{\prime}.
    \end{small}
    \end{equation}
    Then, according to Assumption~\ref{asp-3}, $f(\mathbf{x}^t_{S},\mathbf{x}_S^t)$ is bounded within $[0, 1]$, and Eq.(\ref{eq-4-4-0}) can be obtained.
\end{proof}

\subsection{Theoretical error bounds of EmSHAP and competitive methods}\label{Sec-5-3}
Based on the propositions and corollaries discussed above, the upper bound between the expectations of the mean absolute deviation for the estimation contribution function $\hat{v}(S)$ and the true contribution function $v(S)$ can be calculated as follows.
\begin{theorem}\label{thm-1}
    (EmSHAP estimation error upper bound) The expectation of the mean absolute deviation (MAD) between the estimation contribution function $\hat{v}(S)$ and the true contribution function $v(S)$ is upper bounded by
    \begin{equation}
        \begin{small}
            \mathbb{E}[|\hat{v}(S)-v(S)|] \leq \frac{\sqrt{\pi}}{\sqrt{2K}} + \epsilon_2,
        \end{small}
    \end{equation}
    where $K$ is the number of sampling and $\epsilon_2$ is a small positive number.
\end{theorem}
\begin{proof}
$\mathbb{E}[| \hat{v}(S) - v(S) |]$ can be decomposed into two parts, namely the statistical error and the approximation error. The detailed derivation is given as follows.
\begin{equation}\label{eq-4-13}
\begin{small}
\begin{aligned}
    & \quad \mathbb{E}[|\hat{v}(S)-v(S)|] \\
    & = \mathbb{E} \left[\left| \frac{1}{K}\sum_{k=1}^K f(\mathbf{x}_{\bar{S}}^{(k)},\mathbf{x}_S^t) -  \mathbb{E}_{p_{data}(\mathbf{x}_{\bar{S}}|\mathbf{x}_S=\mathbf{x}_S^t)}\left[ f(\mathbf{x}_{\bar{S}}^t,\mathbf{x}_S^t)\right]\right| \right] \\
    & = \mathbb{E} \left[ \left| \frac{1}{K}\sum_{k=1}^K f(\mathbf{x}_{\bar{S}}^{(k)},\mathbf{x}_S^t) - \mathbb{E}_{p_\theta(\mathbf{x}_{\bar{S}}|\mathbf{x}_S=\mathbf{x}_S^t)}\left[ f(\mathbf{x}_{\bar{S}},\mathbf{x}_S^t)\right]  \right.\right.\\
    &\quad\left. \left.+ \mathbb{E}_{p_\theta(\mathbf{x}_{\bar{S}}|\mathbf{x}_S=\mathbf{x}_S^t)}\left[ f(\mathbf{x}_{\bar{S}},\mathbf{x}_S^t)\right]
    - \mathbb{E}_{p_{data}(\mathbf{x}^t_{\bar{S}}|\mathbf{x}_S=\mathbf{x}_S^t) }\left[ f(\mathbf{x}_{\bar{S}}^t,\mathbf{x}_S^t)\right]\right|\right] \\
    & \leq \mathbb{E} \left[\left| \frac{1}{K}\sum_{k=1}^K f(\mathbf{x}_{\bar{S}}^{(k)},\mathbf{x}_S^t) - \mathbb{E}_{p_\theta(\mathbf{x}_{\bar{S}}|\mathbf{x}_S=\mathbf{x}_S^t)}\left[ f(\mathbf{x}_{\bar{S}},\mathbf{x}_S^t)\right] \right| \right] \\
    &\quad + \mathbb{E} \left[\left| \mathbb{E}_{p_\theta(\mathbf{x}_{\bar{S}}|\mathbf{x}_S=\mathbf{x}_S^t)}\left[ f(\mathbf{x}_{\bar{S}},\mathbf{x}_S^t)\right]
    -  \mathbb{E}_{p_{data}(\mathbf{x}^t_{\bar{S}}|\mathbf{x}_S=\mathbf{x}_S^t)}\left[ f(\mathbf{x}_{\bar{S}}^t,\mathbf{x}_S^t)\right] \right|\right],
\end{aligned}
\end{small}
\end{equation}
where $\mathbf{x}_{\bar{S}} \sim p_{\theta}(\mathbf{x}_{\bar{S}}|\mathbf{x}_S = \mathbf{x}_S^t)$. The first term in the last inequality of Eq.(\ref{eq-4-13}) is the statistical error and the second is the approximation error. The upper bound of the statistical error is first considered.

\begin{equation}\label{eq-4-14}
\begin{small}
\begin{aligned}
    &\quad \mathbb{E} \left[\left| \frac{1}{K}\sum_{k=1}^K f(\mathbf{x}_{\bar{S}}^{(k)},\mathbf{x}_S^t) - \mathbb{E}_{p_\theta(\mathbf{x}_{\bar{S}}|\mathbf{x}_S=\mathbf{x}_S^t)}\left[ f(\mathbf{x}_{\bar{S}},\mathbf{x}_S^t)\right] \right| \right] \\
\notag
\end{aligned}
\end{small}
\end{equation}
\begin{equation}
\begin{small}
\begin{aligned}
    & = 2 \mathbb{E} \left[ \left( \frac{1}{K}\sum_{k=1}^K f(\mathbf{x}_{\bar{S}}^{(k)},\mathbf{x}_S^t) - \mathbb{E}_{p_\theta(\mathbf{x}_{\bar{S}}|\mathbf{x}_S=\mathbf{x}_S^t)}\left[ f(\mathbf{x}_{\bar{S}},\mathbf{x}_S^t)\right] \right)_{+} \right] \\
    & = 2 \int_0^{+\infty} \mathbb{P}\left( \frac{1}{K}\sum_{k=1}^K f(\mathbf{x}_{\bar{S}}^{(k)},\mathbf{x}_S^t) - \mathbb{E}_{p_\theta(\mathbf{x}_{\bar{S}}|\mathbf{x}_S=\mathbf{x}_S^t)}\left[ f(\mathbf{x}_{\bar{S}},\mathbf{x}_S^t)\right] \geq \epsilon_3 \right)d\epsilon_3 \\
    & \leq 2 \int_0^{+\infty}\exp(-2K \epsilon_3^2)d\epsilon_3 = \frac{\sqrt{\pi}}{\sqrt{2K}},
\end{aligned}
\end{small}
\end{equation}
where $(\cdot)_+$ is $max(\cdot, 0)$, the integral in Eq.(\ref{eq-4-14}) uses the equation $\mathbb{E}[u]=\int_0^{+\infty}\mathbb{P}(u>\epsilon)d\epsilon$ for a non-negative random variable $u$, the inequality in Eq.(\ref{eq-4-14}) holds according to Corollary~\ref{cly-3}. Hence, the statistical error is bounded by $\frac{\sqrt{\pi}}{\sqrt{2K}}$. The derivation of the approximation error is given as follows.
\begin{equation}\label{eq-4-15}
\begin{small}
\begin{aligned}
    & \quad \mathbb{E} \left[\left| \mathbb{E}_{p_\theta(\mathbf{x}_{\bar{S}}|\mathbf{x}_S=\mathbf{x}_S^t)}\left[ f(\mathbf{x}_{\bar{S}},\mathbf{x}_S^t)\right]
    -  \mathbb{E}_{p_{data}(\mathbf{x}_{\bar{S}}|\mathbf{x}_S=\mathbf{x}_S^t)}\left[ f(\mathbf{x}_{\bar{S}}^t,\mathbf{x}_S^t)\right] \right|\right] \\
    & = \mathbb{E} \left[ \left| \int f(\mathbf{x}_{\bar{S}},\mathbf{x}_S^t)p_\theta(\mathbf{x}_{\bar{S}}|\mathbf{x}_S=\mathbf{x}_S^t)d\mathbf{x}_{\bar{S}} \right.\right.\\
    &\quad \left.\left.- \int f(\mathbf{x}_{\bar{S}}^t,\mathbf{x}_S^t)p_{data}(\mathbf{x}^t_{\bar{S}}|\mathbf{x}_S=\mathbf{x}_S^t)d\mathbf{x}_{\bar{S}}   \right| \right] \\
    & = \mathbb{E} \left[ \left| \int f(\mathbf{x}_{\bar{S}},\mathbf{x}_S^t)\frac{p_\theta(\mathbf{x}_{\bar{S}}|\mathbf{x}_S=\mathbf{x}_S^t)}{p_{data}(\mathbf{x}^t_{\bar{S}}|\mathbf{x}_S=\mathbf{x}_S^t)}p_{data}(\mathbf{x}^t_{\bar{S}}|\mathbf{x}_S=\mathbf{x}_S^t)d\mathbf{x}_{\bar{S}}  \right.\right.\\
    &\left.\left. \quad- \int f(\mathbf{x}_{\bar{S}}^t,\mathbf{x}_S^t) p_{data}(\mathbf{x}^t_{\bar{S}}|\mathbf{x}_S=\mathbf{x}_S^t)d\mathbf{x}_{\bar{S}}\right| \right] \\
    & = \mathbb{E} \left[ \left| \mathbb{E}_{p_{data}(\mathbf{x}^t_{\bar{S}}|\mathbf{x}_S=\mathbf{x}_S^t)}\left[ \left(f(\mathbf{x}_{\bar{S}},\mathbf{x}_S^t)\frac{p_\theta(\mathbf{x}_{\bar{S}}|\mathbf{x}_S=\mathbf{x}_S^t)}{p_{data}(\mathbf{x}^t_{\bar{S}}|\mathbf{x}_S=\mathbf{x}_S^t)} -f(\mathbf{x}_{\bar{S}}^t,\mathbf{x}_S^t) \right)\right] \right| \right] \\
    & \leq \mathbb{E} \left[ \mathbb{E}_{p_{data}(\mathbf{x}^t_{\bar{S}}|\mathbf{x}_S=\mathbf{x}_S^t)} \left[ \left|f(\mathbf{x}_{\bar{S}},\mathbf{x}_S^t)\frac{p_\theta(\mathbf{x}_{\bar{S}}|\mathbf{x}_S=\mathbf{x}_S^t)}{p_{data}(\mathbf{x}^t_{\bar{S}}|\mathbf{x}_S=\mathbf{x}_S^t)}-f(\mathbf{x}_{\bar{S}}^t,\mathbf{x}_S^t) \right| \right] \right] \\
    & = \epsilon_2.
\end{aligned}
\end{small}
\end{equation}

By combining Eq.(\ref{eq-4-14}) and Eq.(\ref{eq-4-15}), the upper bound of $\mathbb{E}[|\hat{v}(S)-v(S)|]$ can be obtained, which is $\frac{\sqrt{\pi}}{\sqrt{2K}} + \epsilon_2$.
\end{proof}
We now discuss the tightness of the upper bound. The first inequality used is the triangle inequality for absolute values in Eq.(\ref{eq-4-13}). If the two functions on the right side of the inequality share the same sign (both positive or negative), that is,

\begin{equation}\label{eq-4-16}
\begin{small}
\begin{gathered}
        \left( \frac{1}{K}\sum_{k=1}^K f(\mathbf{x}_{\bar{S}}^{(k)},\mathbf{x}_S^t) - \mathbb{E}_{p_\theta(\mathbf{x}_{\bar{S}}|\mathbf{x}_S=\mathbf{x}_S^t)}\left[ f(\mathbf{x}_{\bar{S}},\mathbf{x}_S^t)\right] \right) \cdot \\
        \left( \mathbb{E}_{p_\theta(\mathbf{x}_{\bar{S}}|\mathbf{x}_S=\mathbf{x}_S^t)}\left[ f(\mathbf{x}_{\bar{S}},\mathbf{x}_S^t)\right]  - \mathbb{E}_{p_{data}(\mathbf{x}^t_{\bar{S}}|\mathbf{x}_S=\mathbf{x}_S^t)}\left[ f(\mathbf{x}_{\bar{S}}^t,\mathbf{x}_S^t) \right] \right) \geq 0,
\end{gathered}
\end{small}
\end{equation}
the bound derived from applying the absolute value triangle inequality is tight. Another inequality is the expectation inequality of absolute values in Eq.(\ref{eq-4-15}). This inequality shows that when the variable always has the same sign, the following holds.
\begin{equation}\label{eq-4-17}
\begin{small}
\begin{gathered}
\sum_{\mathbf{x}_{\bar{S}}}f(\mathbf{x}_{\bar{S}},\mathbf{x}_S^t)\frac{p_\theta(\mathbf{x}_{\bar{S}}|\mathbf{x}_S=\mathbf{x}_S^t)}{p_{data}(\mathbf{x}^t_{\bar{S}}|\mathbf{x}_S=\mathbf{x}_S^t)}-f(\mathbf{x}_{\bar{S}}^t,\mathbf{x}_S^t)
\\ = \sum_{\mathbf{x}_{\bar{S}}}\left| f(\mathbf{x}_{\bar{S}},\mathbf{x}_S^t)\frac{p_\theta(\mathbf{x}_{\bar{S}}|\mathbf{x}_S=\mathbf{x}_S^t)}{p_{data}(\mathbf{x}^t_{\bar{S}}|\mathbf{x}_S=\mathbf{x}_S^t)}-f(\mathbf{x}_{\bar{S}}^t,\mathbf{x}_S^t) \right|.
\end{gathered}
\end{small}
\end{equation}

Hence, if $\mathbf{x}_{\bar{S}}$ satisfies both Eqs.(\ref{eq-4-16}) and (\ref{eq-4-17}), the upper bound is tight.

\begin{theorem}\label{thm-2}
    (KernelSHAP estimation error upper bound) For any fixed $|D|$ and when the sampling number $K$ goes to positive infinity, under the probability greater than $1- 2((|D|+1)^2+(|D|+1)) \exp(-2K \epsilon_3^2)$, the expectation of the mean absolute deviation (MAD) between the estimation contribution function $\hat{v}(S)$ estimated by KernelSHAP~\cite{covert2020improving} and the true contribution function $v(S)$ is upper bounded by,
    \begin{equation}\label{eq-4-00}
    \begin{small}
    \begin{aligned}
        \mathbb{E}[|\hat{v}(S)-v(S)|] &\leq 2(|D|+1)(|S|+1)\frac{\sqrt{\pi}}{\sqrt{2K}}||\boldsymbol{\Sigma}^{*^{-1}}||_F \\
         & \quad + 2(|D|+1)^{\frac{5}{2}}(|S|+1)\frac{\sqrt{\pi}}{\sqrt{2K}}||\boldsymbol{\Sigma}^{*^{-1}}||^2_F+\epsilon_3^2,
    \end{aligned}
    \end{small}
    \end{equation}
    where $\epsilon_3$ is a small positive number related to $\epsilon_1$ and $\epsilon_1^{\prime}$. $\boldsymbol{\Sigma}^{*} \in \mathbb{R}^{K \times K}$ is the expectation of Shapley kernel symmetric matrix, more detailed information about this matrix can be found in Ref.~\cite{covert2020improving}.
\end{theorem}
\begin{proof}
The detailed proof can be found in Appendix A.
\end{proof}

\begin{theorem}\label{thm-3}
    (VAEAC estimation error upper bound) The expectation of the mean absolute deviation (MAD) between the contribution function $\hat{v}(S)$ estimated by VAEAC~\cite{olsen2022using} and the true contribution function $v(S)$ is upper bounded by,
    \begin{equation}
        \begin{small}
            \mathbb{E}[|\hat{v}(S)-v(S)|] \leq \frac{\sqrt{\pi}}{\sqrt{2K}} + \epsilon_2 + \delta,
        \end{small}
    \end{equation}
    where $K$ is the sampling number, $\epsilon_2$ is the probability density ratio between the estimation conditional probability $p_{\theta}(\mathbf{x}_{\bar{S}}|\mathbf{x}_S=\mathbf{x}_S^t)$ and the optimal conditional probability $p_{\theta^*}(\mathbf{x}_{\bar{S}}|\mathbf{x}_S=\mathbf{x}_S^t)$ estimated using VAE, $\delta$ is a positive number related to the amortization error of VAE.
\end{theorem}
\begin{proof}
The detailed proof can be found in Appendix B.
\end{proof}

Theorems~\ref{thm-2} and~\ref{thm-3} analyzed the error bounds of two widely applied methods of KernelSHAP~\cite{covert2020improving} and VAEAC~\cite{olsen2022using}. Theorem~\ref{thm-2} shows that as the sampling time $K$ increases, the first term of the right hand side of Eq.(\ref{eq-4-00}) vanishes, hence its error bound is governed by the second term. Since $||\boldsymbol{\Sigma}^{*^{-1}}||^2_F$ is a constant, it can be seen that the error bound of KernelSHAP is highly related to the dimension of input features $D$. That is to say, the estimation accuracy will decrease when the input feature dimension increases. On the other hand, comparing Theorem~\ref{thm-1} and Theorem~\ref{thm-3}, it can be seen that the error bound of VAEAC is greater than the energy-based model as $\delta$ is always greater than zero. Hence it can be concluded that the proposed EmSHAP has a tighter error bound than KernelSHAP and VAEAC.

\begin{remark}
The above analysis differs from existing work~\cite{maleki2013bounding,covert2020improving,kolpaczki2024approximating,kelodjou2024shaping}. Previous methods primarily address marginal Shapley values, where errors stem from sampling or regression approximations. Moreover, they are often tailored to specific scenarios like correlated features~\cite{verdinelli2024feature}, distribution shifts~\cite{sebastian2024feature}, or non-parametric settings~\cite{williamson2020efficient}. While these are valuable, they remain specialized. In contrast, Theorem~\ref{thm-1} establishes an error bound for conditional Shapley values within an explicit generative modelling framework. This approach is more general in two key aspects: i) it applies to any generative model-based conditional estimator (e.g., VAEAC), and ii) when the conditional distribution is accurately modelled, the bound resembles the concentration inequalities-based bounds of marginal methods (e.g., KernelSHAP). Thus, our analysis provides a unifying framework where sampling- and regression-based approaches emerge as special cases under simplified assumptions.
\end{remark}

\subsection{Validation of Error Bound through a toy example}\label{Sec-5-4}
In order to give a more intuitive exhibition of the results in Theorem~\ref{thm-1}, a toy example is used to visualize the error bound~\ref{thm-1}.

Consider a simple example involving three inputs $\mathbf{x}=[x_1, x_2, x_3]$, with each feature $x_i (i=1, 2, 3)$ generated by a periodic function as follows.
\begin{equation}
    x_i(t) = \sin( i t \pi )+\eta_1(t).
\end{equation}

Here, $t$ is the time instance ranging from 0 to 10, and is uniformly sampled into 5,000 points, $\eta_1$ is Gaussian distributed noise with mean 0 and variance 0.001. The output $y$ is obtained by a linear combination of the features of the input vector $\mathbf{x}$ with a weight vector $\mathbf{w}=[-0.4, -1.2, 0.8]$, as follows.
\begin{equation}
    y(t)=\mathbf{w}^T \mathbf{x}(t) = -0.4x_1(t)-1.2x_2(t)+0.8x_3(t) + \eta_2(t),
\end{equation}
where $\eta_2$ is Gaussian noise with mean 0 and variance 0.001.

The first 80\% of data samples are used to train a random forest model and EmSHAP, and the remaining 20\% of the samples are used to verify Theorem~\ref{thm-1}. In this toy example, the statistical error is the difference between the predictions obtained using the mean of $K$ samples drawn from the EmSHAP output distribution and its expectation. The approximation error $\epsilon_2$ is the difference between the predictions derived from the expectation of the EmSHAP output distribution and the real data.

Fig~\ref{fig3} shows the curve of the contribution function obtained through EmSHAP under different feature combinations. From Fig~\ref{fig3}, it can be seen that the approximate error is fixed for each feature combination, and the statistical error continues to decrease with the increase in the number of samplings. And the sum of the approximate and statistical errors is always below the error boundary. This validates the correctness of Theorem~\ref{thm-1}.

\begin{figure}[htbp]
    \centering
    \includegraphics[width=\linewidth]{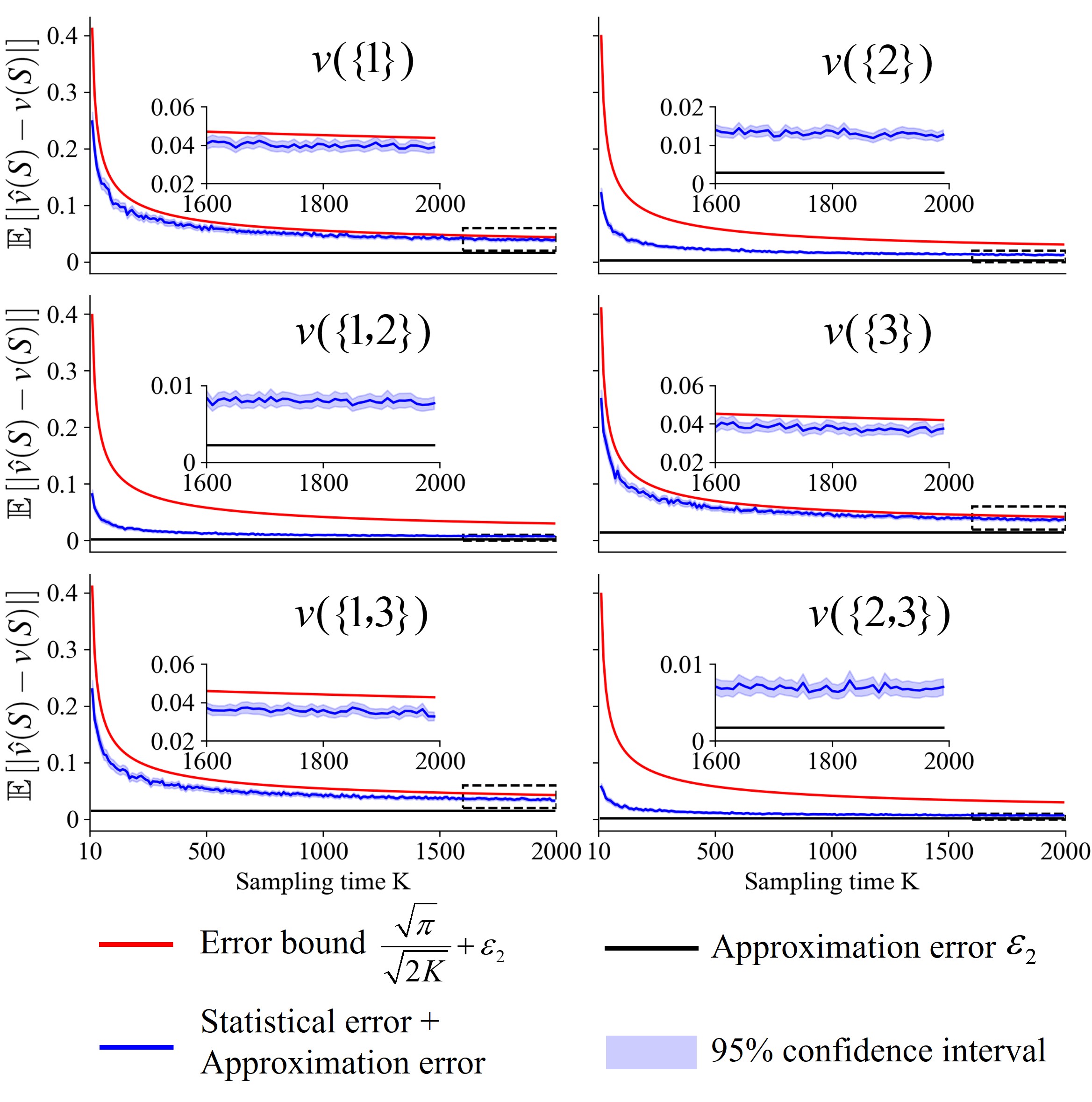}
    \caption{Toy example for validating error bound of EmSHAP.}
    \label{fig3}
\end{figure}

\section{Case studies}\label{Sec-6}
In this section, feature attribution experiments are performed on four real-world tasks: image classification on MNIST\footnote{MNIST dataset was acquired at http://yann.lecun.com/exdb/mnist/.}, tabular classification on adults' income(ADTI)\footnote{ADTI was acquired at https://archive.ics.uci.edu/dataset/2/adult.}, time-series prediction of electric transformer's oil temperature(ETT)\footnote{ETT dataset was acquired at https://github.com/zhouhaoyi/ETDataset.}, and automated medical coding on MIMIC-III\footnote{MIMIC-III dataset was acquired at https://mimic.physionet.org/.}. For the four tasks, EmSHAP is used to estimate Shapley values for feature attribution and is compared with seven competitive methods:
\begin{itemize}
  \item SamplingSHAP~\cite{Strumbelj2010explanation}: A Shapley value estimation method based on sampling from a random permutation of features.
  \item SVARM~\cite{kolpaczki2024approximating}: A sampling-based Shapley value calculation method without marginal contribution calculation.
  \item SHAP-IQ~\cite{fumagalli2023shap}: A sampling-based approximation method for calculating Shapley interaction quantification.
  \item KernelSHAP~\cite{covert2020improving}: An effective method for calculating Shapley values, which uses weighted least squares to calculate marginal Shapley values.
  \item ST-SHAP\cite{kelodjou2024shaping}: A regression-based method utilizing neighbor selection to achieve a more stable version of KernelSHAP.
  \item FastSHAP~\cite{jethani2022fastshap}: A fast Shapley value approximation method that trains an interpreter to get marginal Shapley values using a single forward pass.
  \item VAEAC~\cite{olsen2022using}: A generative model which obtains the conditional Shapley value by calculating the conditional expectation of the contribution function.
\end{itemize}

Since the true Shapley value of the real-world datasets is difficult to calculate, a quantitative metric is used to evaluate the performance of feature attributions.

\textbf{Softmax information curve (SIC AUC)}~\cite{kapishnikov2019xrai}. This metric evaluates the impact of attributed features on model output. Effective feature attribution should align closely with the model's primary concerns~\cite{zhuo2024Integrated}. The metric consists of two distinct indicators. The first indicator involves incrementally replacing the explained features with background features. Features with the highest attribution are introduced first, while those with the lowest attribution are added last, based on a predefined attribution threshold. Features with higher attribution values generally lead to more substantial improvements in model output performance. The model's output curve can be plotted using the Softmax function, and the area under this curve can be quantified using the SIC AUC-ADD index. Conversely, the second indicator progressively removes features that contribute most to the output until all features are replaced by the background~\cite{erion2021improving}. In this case, features with higher attribution values typically cause a more rapid decline in model performance, and the area under the resulting output curve is quantified using the SIC AUC-DEL index.
A more accurate explanation yields a larger SIC AUC-ADD and a smaller SIC AUC-DEL index.

\textbf{Hyperparameter selection}
Before model training, hyperparameters should be determined. For SamplingSHAP, SVARM, and SHAP-IQ, the maximum sampling number (also called evaluation budget) is set as 1000, while for KernelSHAP and ST-SHAP, the number of sampling is fixed at 1000. FastSHAP, VAEAC, and EmSHAP are neural network-based methods, and their hyperparameters are summarized in Table~\ref{table2}. In Table~\ref{table2}, all MLPs include two residual connections, and the models are trained with the Adam optimizer under a learning rate of 0.001 for 100 epochs, with batch sizes of 256 for MNIST, 64 for ADTI, 100 for ETT, and 16 for MIMIC-III. For VAEAC and EmSHAP, the number of sampling from the conditional distribution is set as 20, and in EmSHAP, the number of importance sampling from the proposal distribution is also set as 20 following Ref.~\cite{nash2019autoregressive}. The thresholds for dynamic masking vary from $\zeta_{\min}=0.2$ to $\zeta_{\max}=0.8$, and the context vector dimension $\gamma$ is fixed at 32. To balance between accuracy and computational cost, the hidden layer size is set to be 32 for VAEAC and EmSHAP, and 128 for FastSHAP. The architectures of the predictive models to be explained in different case studies are provided in Appendix~C.

\begin{table}[htbp]
\caption{Network structure of feature attribution methods on real-world datasets}
\label{table2}
\centering
\renewcommand\arraystretch{0.5}
\setlength{\tabcolsep}{2mm}{
\begin{tabular}{cccc}
\toprule
\multirowcell{2}{Method} & \multirowcell{2}{Network} & \multirowcell{2}{Prototype} & Latent \\
& & &dimensions\\
\midrule
\multirowcell{2}{FastSHAP} & Surrogate & MLP      & 128       \\
                        & Explainer   & MLP       & 128       \\
\midrule
\multirowcell{3}{VAEAC}  & Full encoder & MLP       & 32       \\
                        & Masked encoder   & MLP       & 32       \\
                        & Decoder     & MLP        & 32 \\
\midrule
\multirowcell{2}{EmSHAP} & Proposal network & GRU  & $2|D|+|\boldsymbol{\gamma}|$ \\
                        & Energy network   & MLP      & 32     \\
\bottomrule
\end{tabular}}
\end{table}

Table~\ref{table3} summarizes the feature attribution performance and computation efficiency of different methods on four real-world datasets. All experiments are performed on a workstation equipped with an Intel Xeon Gold 6326 processor, 377 GB of memory, and an RTX4080 graphics processing unit. The calculations are performed using Python 3.9 and PyTorch 2.4.1. The calculation time and memory usage for all methods are also compared and listed in Table~\ref{table3}. Table~\ref{table3} shows that EmSHAP consistently achieves the highest SIC AUC-ADD in all datasets, and the lowest SIC AUC-DEL across most datasets, indicating superior feature attribution performance. Regarding computational efficiency, methods without a training stage generally exhibit higher efficiency\footnote{Our version can be accessed
on https://github.com/icanreachyou/EmSHAP.}.

\begin{table*}[htbp]
\caption{Performance comparison of feature attribution for competitive methods on real-world datasets}
\label{table3}
\centering
\renewcommand\arraystretch{0.5}
\setlength{\tabcolsep}{2mm}{
\begin{tabular}{cccc|cc|cc}
\toprule
\multirowcell{2}{Dataset} & \multirowcell{2}{Explainer} & \multicolumn{2}{c|}{SIC AUC}&\multicolumn{2}{c|}{Run time}&\multicolumn{2}{c}{Memory usage}\cr\cmidrule{3-4}\cmidrule{5-6}\cmidrule{7-8}
& & ADD$\uparrow$ & DEL$\downarrow$ & Training$\downarrow$ & Inference$\downarrow$ & Training$\downarrow$ & Inference$\downarrow$\\
\midrule
\multirowcell{12}{MNIST} & SamplingSHAP & 0.893        & 0.269    & $\slash$        &   2.55s   & $\slash$        &   \textbf{1464.81}MB   \\
                        & SVARM & 0.257        & 0.125    & $\slash$        &   2.03s   & $\slash$        &   1647.98MB   \\
                        & SHAP-IQ &   0.678      &  0.533   & $\slash$        &   2.35s   & $\slash$        &   1578.24MB   \\
                        & KernelSHAP   & 0.861        & 0.134    &  $\slash$       &  1.84s  & $\slash$        &   1714.42MB \\
                        & ST-SHAP &     0.953    &  0.088   & $\slash$        &   1.33s   & $\slash$        &   1674.67MB   \\
                        & FastSHAP     & 0.861        & 0.205    &  3302.66s       &   \textbf{0.19}s   &  2900.80MB       &   1958.56MB \\
                        & VAEAC        & 0.714        & 0.447    &  5079.43s       &   2.71s    &  3373.30MB       &  1533.63MB    \\
                        & EmSHAP       & \textbf{0.965}        & \textbf{0.037}  &  \textbf{2977.10}s     & 1.63s   &  \textbf{1823.91}MB     & 1879.68MB  \\
\midrule
\multirowcell{12}{ADTI} & SamplingSHAP & 0.925        & 0.494    & $\slash$       &  0.20s & $\slash$       &  718.47MB\\
                        & SVARM & 0.927        & 0.518    & $\slash$        &   0.18s   & $\slash$        &   696.97MB   \\
                        & SHAP-IQ &    0.918     &  0.498   & $\slash$        &   \textbf{0.04}s   & $\slash$        &   765.94MB   \\
                        & KernelSHAP   & 0.913        & 0.465    & $\slash$       &  1.43s  & $\slash$       &  703.79MB \\
                        & ST-SHAP &    0.878     &  0.461   & $\slash$        &   4E-3s   & $\slash$        &   700.48MB   \\
                        & FastSHAP     & 0.914        & 0.535    & 295.87s        & 0.26s & \textbf{717.17}MB        & \textbf{696.64}MB \\
                        & VAEAC        & 0.855        & 0.596    & 804.22s        & 2.30s  & 1084.62MB        & 956.01MB  \\
                        & EmSHAP       & \textbf{0.933}   & \textbf{0.411}  & \textbf{270.90}s   & 0.87s & 966.58MB   & 946.61MB\\
\midrule
\multirowcell{12}{ETT} & SamplingSHAP &   0.893 &  0.526  & $\slash$       &  0.08s & $\slash$       & 761.11MB \\
                        & SVARM &    0.875    & 0.342    & $\slash$        &  0.05s    & $\slash$    &  758.97MB   \\
                        & SHAP-IQ &    0.905     &  0.301   & $\slash$        &   2E-3s   & $\slash$        &   757.36MB   \\
                        & KernelSHAP   &    0.907    &  0.289   & $\slash$       &  7E-3s  & $\slash$       &  \textbf{534.54}MB \\
                        & ST-SHAP &     0.895    &   0.318  & $\slash$        &   6E-3s   & $\slash$        &   536.19MB   \\
                        & FastSHAP     &   0.863    &  0.402   &  19.39s  & \textbf{1E-3}s & 707.49MB        & 586.20MB \\
                        & VAEAC        &  0.902      &  0.359  & 46.91s        & 0.33s  & \textbf{530.85}MB        &568.14MB  \\
                        & EmSHAP       & \textbf{0.915}   & \textbf{0.243}  & \textbf{16.55}s   & 4E-3s & 554.11MB   & 540.46MB\\
\midrule
\multirowcell{12}{MIMIC} & SamplingSHAP &   0.604 &  0.431  & $\slash$       &  5.35s & $\slash$       & 1079.22MB \\
                        & SVARM &   0.682   &  0.426   & $\slash$        &  2.87s    & $\slash$    &  1091.45MB   \\
                        & SHAP-IQ &    0.633     &  \textbf{0.351}   & $\slash$        &   12.04s   & $\slash$        &   1193.91MB   \\
                        & KernelSHAP   &   0.663    &   0.361  & $\slash$       &  111.96s  & $\slash$       &  3685.93MB \\
                        & ST-SHAP &  0.509   &    0.493 & $\slash$        &   3.51s   & $\slash$        &   1067.32MB   \\
                        & FastSHAP     & 0.617  &  0.424   &  \textbf{780.72}s  & \textbf{1.17}s & 1254.49MB        & \textbf{987.54}MB \\
                        & VAEAC   &   0.572    &  0.466 & 13262.06s  & 128.65s  & 2762.11MB        & 2769.82MB  \\
                        & EmSHAP       &  \textbf{0.699}  & 0.384  & 1168.72s   & 59.93s & \textbf{1165.51}MB   & 1168.10MB\\
\bottomrule
\end{tabular}}
\end{table*}

\subsection{Feature attribution for MNIST image classification}
In this subsection, the MNIST dataset, which is widely used for evaluating XAI methods~\cite{chen2018shapley,frye2021shapley, lundberg2017unified}, is considered to test the performance of EmSHAP. The MNIST dataset provides a standardized benchmark for assessing algorithm performance in tasks such as classification and pattern recognition. It consists of 70,000 grayscale images of digits (0-9), each having a size of 28x28 pixels, with 60,000 images allocated for training and 10,000 for testing.

To classify the MNIST images, a simple CNN model is established, comprising two convolutional, two max-pooling layers, and a fully connected output layer. The trained model achieved a classification accuracy of 97\%, indicating a good performance.
Given that CNN is a black-box model, it is difficult to understand the importance of different pixels. To perform this task, EmSHAP is employed to quantify the influence of input pixels on the classification results. It is trained and utilized to estimate the conditional probability distributions of the test data, which are further used to estimate the Shapley value.

In order to enhance computational efficiency, we use average pooling to reduce the 28×28 image to a 14×14 image, and the number of attributing features is reduced to 196. Fig~\ref{fig4} shows the feature attribution results using EmSHAP and seven competitive methods.
\begin{figure}[h]
\centerline{\includegraphics[width=\linewidth]{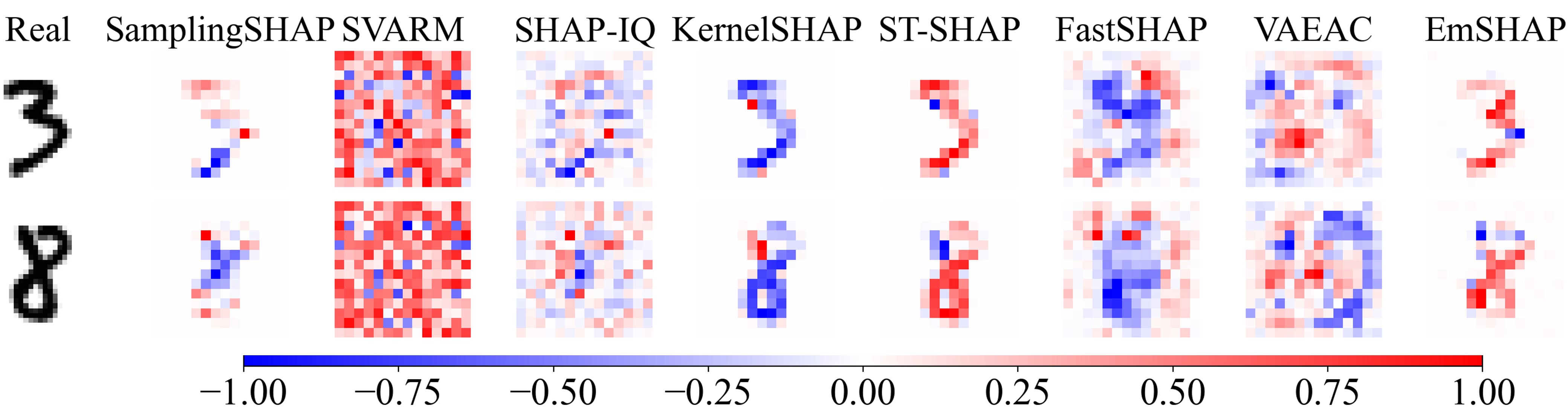}}
\caption{Feature attributions on MNIST dataset.}
\label{fig4}
\end{figure}
As is shown in Fig~\ref{fig4}, SVARM and SHAP-IQ may face challenges in terms of interpretability when applied to image classification tasks, as they yield ambiguous results. Similarly, FastSHAP and VAEAC also produce ambiguous reconstructions. On the other hand, SamplingSHAP, KernelSHAP, ST-SHAP, and EmSHAP all successfully attribute the digits 3 and 8. A more delicate inspection, however, shows that SamplingSHAP, KernelSHAP, and ST-SHAP yielded attribution that align with the classified digits, but tends to omit finer details. In contrast, EmSHAP exhibits finer details in feature attribution,  especially for digit 3. The better performance is also verified by the higher SIC AUC-ADD value and lower SIC AUC-DEL value in Table~\ref{table3}. The robustness of different methods is also tested by adding Gaussian noise with standard deviations varying    from 0.1 to 0.5, with an interval of 0.1. The results are shown in Appendix D, which verifies the better robustness of EmSHAP in the case of noises.

\subsection{Feature attribution for ADTI tabular classification}
Classifying adult income using census data is a typical tabular classification task and has been widely used to test XAI methods. In this subsection, we use this dataset to test the performance of EmSHAP on tabular data. The census data collects information on 48,842 independent individuals, involving 12 input variables related to individual income and a discrete output variable $y$, where $y=0$ means that the individual income is less than or equal to 50k, and $y=1$ means that the individual income is greater than 50k. After the individual data with missing values are removed, 32,562 samples remain. 80\% of these samples are randomly selected as the training set, and the remaining samples are used as the test set. To classify adult income, a three-layer MLP model is established, and the hidden layer size is set to 64. The model is trained using the training set and is optimized using the Adam optimizer with a learning rate of 0.001. The classification accuracy of the model can reach 84.75\%, indicating that the model has good classification performance.

To explain this predictive model, all eight methods are used to quantify the contribution of input features. Fig~\ref{fig8} shows the mean absolute Shapley values of each feature.
\begin{figure}[htbp]
\centerline{\includegraphics[width=\linewidth]{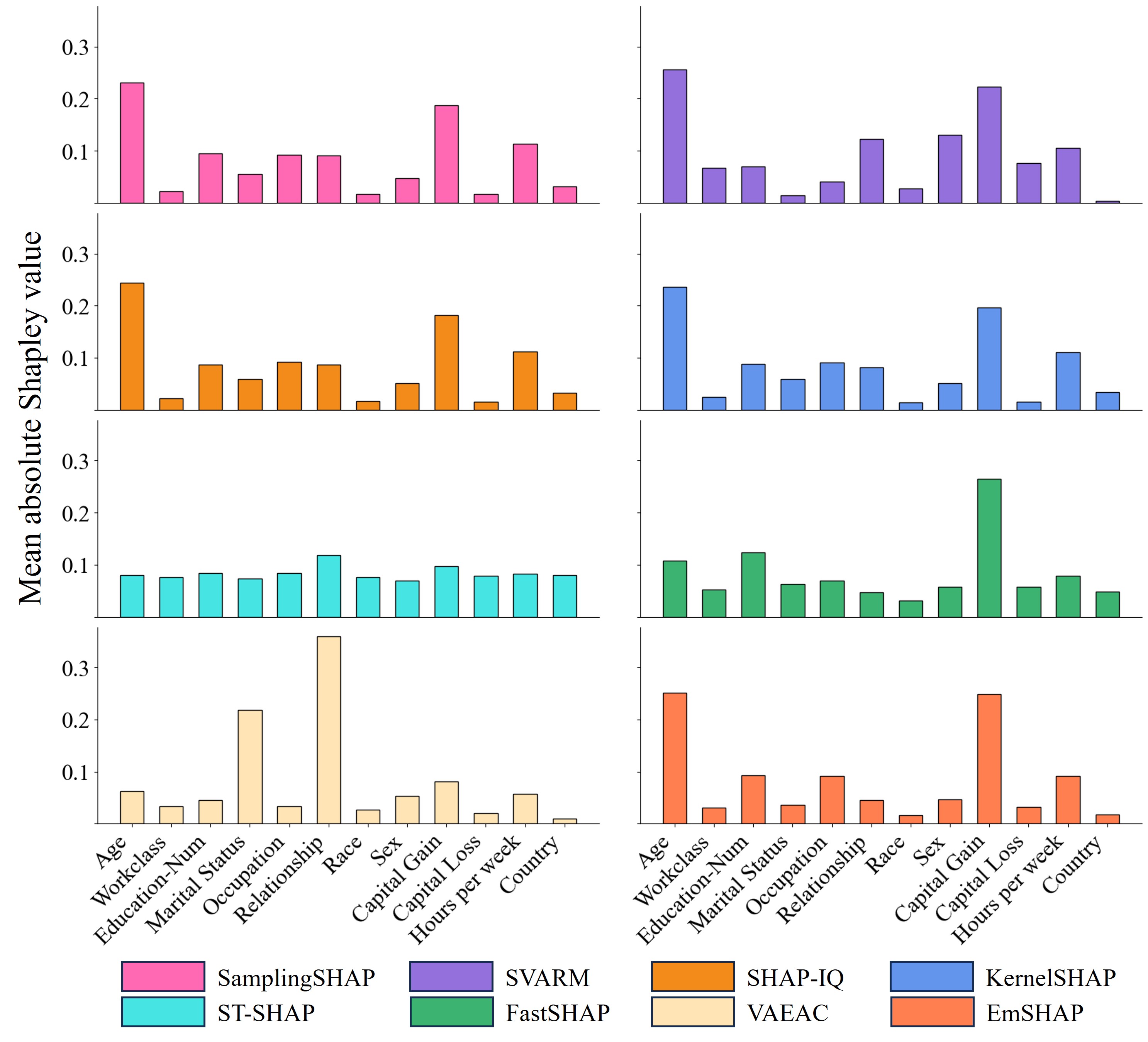}}
\caption{Feature attributions for ADTI dataset.}
\label{fig8}
\end{figure}
Among the eight Shapley value estimation methods, SamplingSHAP, SHAP-IQ, KernelSHAP, and EmSHAP consistently identify age and capital gains as the most influential features in income prediction, which aligns with real-world intuition. These two variables exhibit strong correlations with income: older individuals typically accumulate more work experience and thus higher earnings, while capital gains are directly associated with investment activities. In contrast, demographic attributes such as sex, marital status, and relationship generally show weaker correlations with income. SVARM diverges from this trend by assigning higher importance to relationship and sex, whereas ST-SHAP fails to effectively distinguish the most critical input features. FastSHAP correctly emphasizes the dominant role of capital gains, but treats the remaining variables as having nearly uniform contributions. Similarly, VAEAC attributes substantial importance to marital status and relationship, suggesting a different explanatory perspective. Overall, for the ADTI dataset, EmSHAP produces feature attributions that are not only consistent with domain knowledge but also rank among the most reliable explanations across the evaluated methods.

\begin{figure}[htbp]
\centerline{\includegraphics[width=\linewidth]{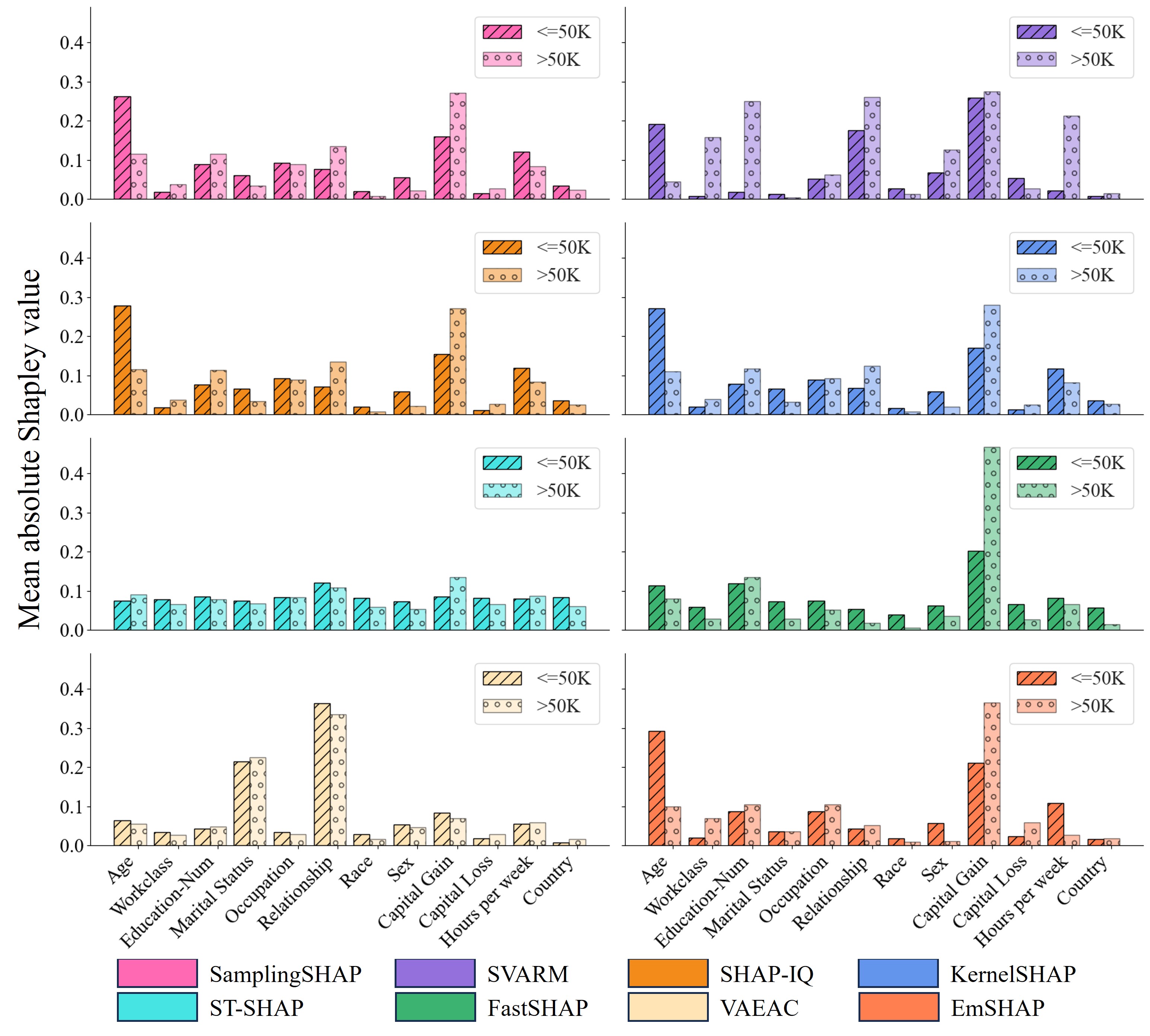}}
\caption{Feature attributions on ADTI dataset for different income classes.}
\label{fig9}
\end{figure}

To clearly illustrate the contribution of each feature to model classifications, we have categorized the data into two groups: income $\leq50$K and income $>50$K. Fig.~\ref{fig9} presents the mean absolute Shapley values for each feature across the two income categories. In Fig.~\ref{fig9}, bar charts with slashes and circles represent the absolute average Shapley values for the classification of individuals with income $\leq50$K and $>50$K, respectively. It can be observed that, for the income $\leq50$K group, age has the greatest contribution to the model’s classification, while for the income $>50$K group, capital gains hold the greatest influence. This difference can be explained. As individuals with income $\leq50$K are typically younger or have less work experience, by accumulating age and work experience, their income will increase. Therefore, in the low-income group, age plays a more significant role in feature attribution. On the other hand, individuals with higher incomes tend to have stronger investment capabilities and more available funds for investment, which results in higher capital gains. Consequently, capital gains are a stronger attribution for individuals with income $>50$K.

\subsection{Feature attributions for ETT dataset}
In this subsection, the electricity transformer dataset (ETT), which has been widely used in long-term and short-term time series prediction tasks~\cite{haoyietal2023informerEx, haoyietal2021informer}, is considered. The time series data in the ETT dataset is typically non-stationary and contains a reasonable portion of outliers. The dataset consists of 6 power load features and an output feature of ”oil temperature”, sampled at different temporal granularities. Here, only the hourly data ETTh1 is used, which contains a total of 17420 samples.

In the ETT dataset, the task is to make predictions of oil temperature based on the 6 power load features. This will help the operators to make corrective operations before the oil temperature gets out of the normal interval. For the purpose of constructing a predictive model, the Informer model~\cite{haoyietal2021informer} is applied, with the first 80\% of the dataset used as training data and the last 20\% as test data. The Informer model performs well in the test data, achieving a mean square error of less than 0.02. Despite the high accuracy, it is still desirable to know which input feature is more response for the predictions, which can be determined through estimation of Shapley values using EmSHAP and the other seven competitive methods.

The feature attribution results can be referred to in Table~\ref{table3}, which shows that EmSHAP outperforms competitive methods in terms of the SIC AUC metric. For a clear display, Fig.~\ref{fig10} presents the heatmap of estimated Shapley values using EmSHAP for the test set. It can be seen from Fig.~\ref{fig10} that the first four variables, HUFL, HULL, MULL, and MULL, are more responsible for the model predictions, whilst the contributions of LUFL and LULL are negligible. This is in accordance with the feature attribution results of other competitive methods.
\begin{figure}[htbp]
\centerline{\includegraphics[width=0.8\linewidth]{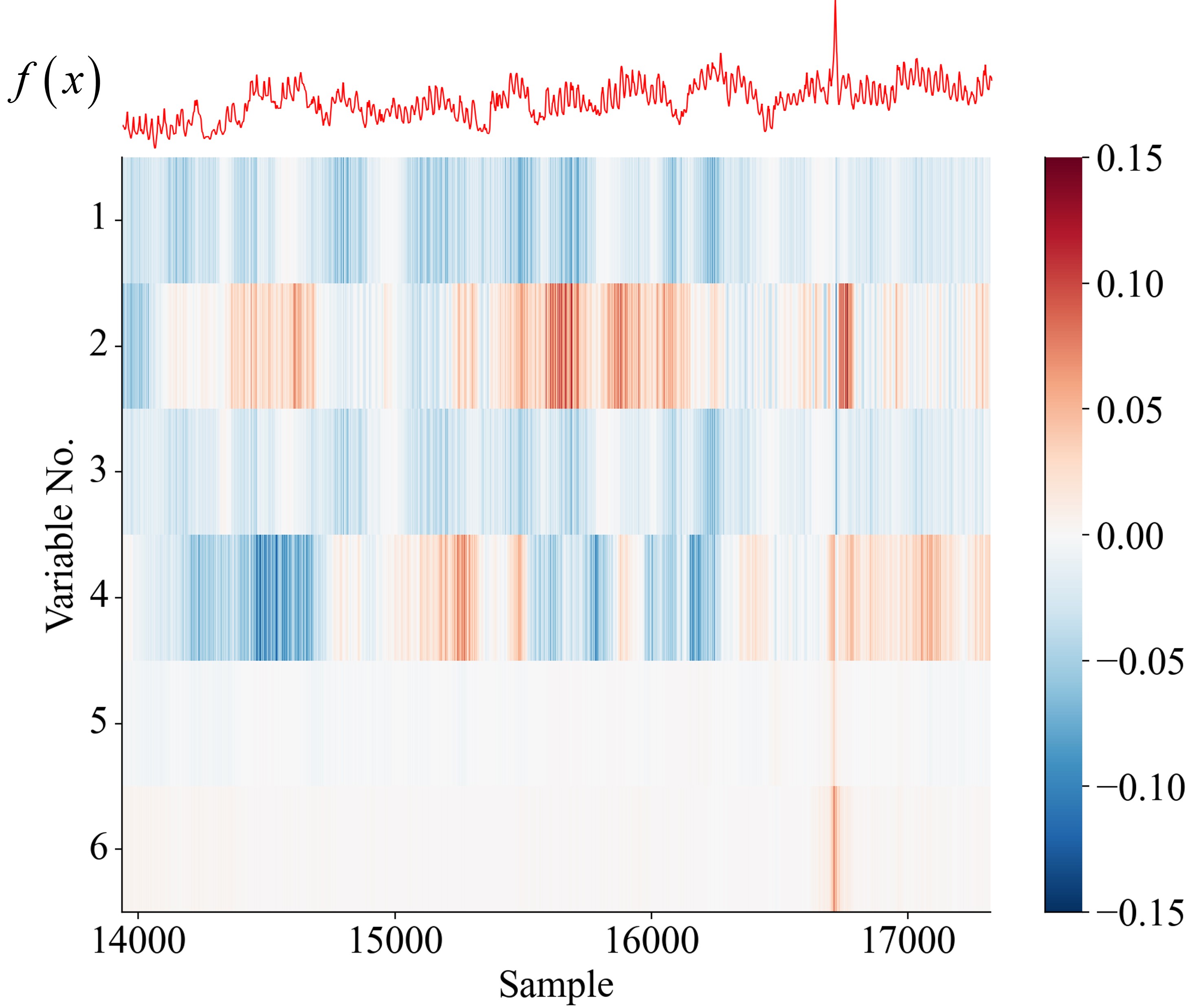}}
\caption{Heatmap of estimating Shapley values using EmSHAP in the ETTh1 prediction task.}
\label{fig10}
\end{figure}
In order to further test the robustness of EmSHAP under adversarial attacks, a series of adversarial samples with varying attack amplitudes is generated using the Fast Gradient Sign Method (FGSM)~\cite{Goodfellow2015Explaining} and added to the original test data. Based on the new test data, feature attribution is again performed using EmSHAP, FastSHAP, and VAEAC, and the results are shown in Table~\ref{table4}. The reason for not comparing with other methods is that they do not involve a training stage. It can be seen from Table~\ref{table4} that the SIC AUC-ADD values of all three methods decrease after adding the adversarial samples. A similar result can be observed for the SIC AUC-DEL values, which increased for all three methods, indicating the deterioration of feature attribution performance. A more delicate observation on the SIC AUC-ADD and SIC AUC-DEL values, however, shows that EmSHAP consistently outperforms the other two methods across all tested amplitudes.
\begin{table}[htbp]
\caption{SIC AUC results for different attack amplitudes on the ETT dataset}
\label{table4}
\centering
\renewcommand\arraystretch{0.5}
\setlength{\tabcolsep}{1.5mm}{
\begin{tabular}{ccccccc}
\toprule
\multirowcell{2}{Explainer} & \multirowcell{2}{SIC \\ AUC} &\multicolumn{5}{c}{Attack amplitude}\cr\cmidrule{3-7}
& & 0.02 & 0.04 & 0.06 & 0.08 & 0.10\\
\midrule
\multirowcell{2}{FastSHAP} & ADD$\uparrow$ & 0.811 & 0.788 & 0.760 & 0.722 & 0.700\\
& DEL$\downarrow$ & 0.320 & 0.365 & 0.388 & 0.427 & 0.454\\
\midrule
\multirowcell{2}{VAEAC} & ADD$\uparrow$ & 0.702 & 0.689 & 0.652 & 0.585 & 0.550 \\
& DEL$\downarrow$ & 0.473 & 0.492 & 0.507 & 0.542 & 0.561 \\
\midrule
\multirowcell{2}{EmSHAP} & ADD$\uparrow$ & 0.882 & 0.873 & 0.832 & 0.792 & 0.733\\
& DEL$\downarrow$ & 0.263 & 0.281 & 0.319 & 0.356 & 0.401 \\
\bottomrule
\end{tabular}}
\end{table}

\subsection{Feature attributions for MIMIC-III}
The MIMIC-III dataset~\cite{johnson2016mimic} is a large-scale, real-world benchmark in medical machine learning, containing clinical documents with International Classification of Diseases (ICD-9) annotations from 2001 to 2012. In this subsection, we focus on the automated medical coding task~\cite{Edin2023Automated} by using 11,371 discharge summaries(each truncated to 2,500 words) to predict the 50 most frequent ICD-9 codes in a multi-label setting. The data preprocessing procedure is followed by~\cite{mullenbach2018explainable}, and the data are split into 8,067/1,574/1,730 samples for training, validation, and testing, respectively. We adopt the Convolutional Attention for Multi-Label (CAML) classification model~\cite{mullenbach2018explainable}, which achieves a top-5 precision of 0.70 on this large-scale task. To interpret its black-box predictions, EmSHAP is applied for word-level attribution and compare it with seven attribution baselines, as reported in Table~\ref{table3}.

\begin{figure}[h]
\centerline{\includegraphics[width=\linewidth]{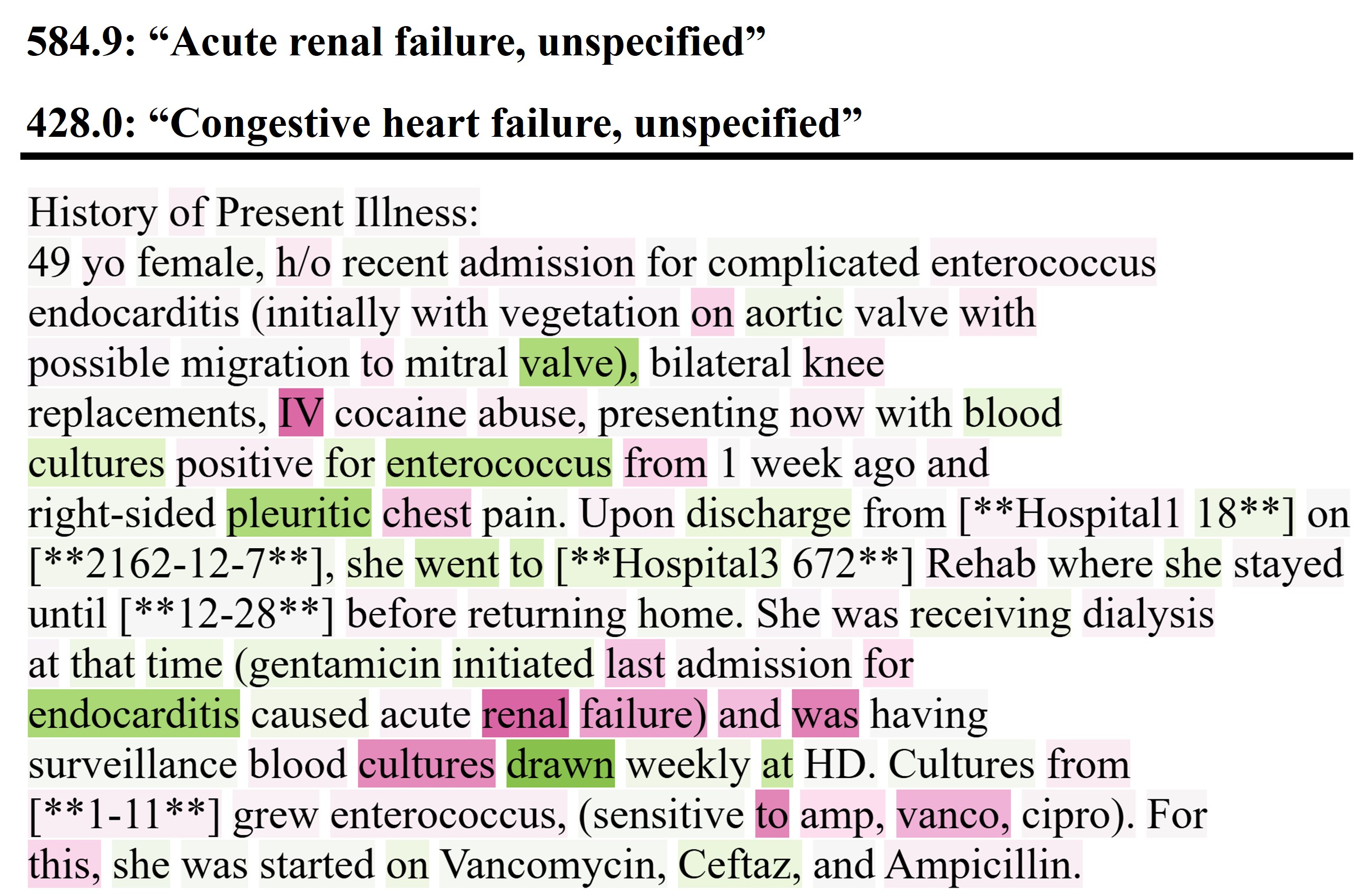}}
\caption{Feature attributions on MIMIC-III dataset.}
\label{fig11}
\end{figure}

Due to space limitations, Fig.~\ref{fig11} presents a partial excerpt of the medical record with ID 120589 together with its feature attribution results, more results are reported in Appendix D. In Fig.~\ref{fig11}, the top two medical codes predicted by CAML are '584.9' (acute renal failure, unspecified) and '428.0' (congestive heart failure, unspecified), which are consistent with the real codes. In the visualization, red and green highlight positive and negative Shapley attributions, respectively, with deeper colors indicating stronger contributions. The results show that EmSHAP successfully identifies clinically meaningful terms such as “\textit{renal}”, “\textit{failure}”, “\textit{endocarditis}”, “\textit{pleuritic}”, and “\textit{chest}”, which align with medical knowledge and confirm that the model predictions rely on relevant clinical cues. At the same time, EmSHAP is not without limitations: due to the complexity and high dimensionality of medical text, irrelevant connective words such as “\textit{and}”, “\textit{from}”, and “\textit{to}” are occasionally assigned spuriously high attributions. This phenomenon may be caused by the GRU's sensitivity to frequent tokens and the amplification of contextual dependencies during conditional probability modeling. Nevertheless, as indicated by the explanation metric SIC AUC in Table~\ref{table3}, EmSHAP consistently highlights key clinical features, thereby providing more interpretable and trustworthy results than existing methods in the automated medical coding task.

\subsection{Performance analysis of EmSHAP under different model structures and parameter settings}\label{Sec-6-5}
In this subsection, the performance of EmSHAP under different model structures is compared and analyzed.

\subsubsection{Dynamic masking vs fixed-rate masking}
As an important component of EmSHAP, the dynamic masking mechanism enabling the model's performance. To better show the functionality of dynamic masking, the performance of EmSHAP under dynamic masking and fix-rate masking is compared, and the results are shown in Appendix~E (Table~E1). As shown in Table~E1, the dynamic masking achieves better performance than fixed-rate masking for all cases. In addition, using the fixed-rate masking requires manually tuning the optimal masking rate, which is undesired in practice. In contrast, the dynamic masking mechanism allows the model to learn masking rate automatically, which leads to better performance~\cite{tian2023heterogeneous}.

\subsubsection{Analysis of context vector}
In addition to investigating the role of the context vector $\boldsymbol{\gamma}$ in EmSHAP, we analyze how the dimension of $\boldsymbol{\gamma}$ affects the model performance. The results of these experiments are presented in Appendix~E (Table~E2). As shown in Table~E2, increasing the context vector dimension leads to a steady improvement in the SIC AUC metrics across all datasets. However, since a larger dimension also increases model complexity, choosing an appropriate dimension is essential to balance efficiency and accuracy.

\subsubsection{Trade-off between accuracy and efficiency}
The trade-off between accuracy and efficiency is a key consideration in model design. In EmSHAP, increasing the latent dimension enhances predictive performance but also raises computational cost in terms of training and inference time, memory usage, and energy consumption. The performance of EmSHAP under different latent dimensions on four datasets is presented in Appendix~E (Table~E3). As shown in Appendix~E (Table~E3), enlarging the latent dimension from 16 to 128 yields consistent gains in SIC AUC, accompanied by higher resource demands. These findings underscore the need to balance accuracy with computational efficiency when selecting latent dimensions for practical deployment.

\subsubsection{Comparison of different RNN-based proposal networks}
To further investigate the impact of different RNN-based architectures, we compared GRU with LSTM, BiGRU, and attention-based GRU (Att-GRU). The results, summarized in Appendix~E (Table~E4), indicate that although LSTM, BiGRU, and Att-GRU occasionally achieve slightly higher SIC AUC scores (e.g., on MNIST and ADTI), these gains come at the expense of substantially higher runtime and memory consumption. Given that the primary goal of EmSHAP is to provide accurate yet efficient explanations, GRU offers the most favorable trade-off between performance and computational cost. Moreover, EmSHAP remains model-agnostic: the GRU module can be readily replaced with LSTM, BiGRU, Att-GRU, or other advanced sequence encoders if additional computational resources are available.

\subsubsection{Ablation study}
An ablation study is performed to reflect the performance of EmSHAP on different standardized models, which involves the combinations of vanilla EBM, the GRU network, and the dynamic masking mechanism.
The detailed results are reported in Appendix~E (Table~E5), which presents both feature attribution performance and runtime across different model configurations. As shown in Table~E5, although EmSHAP requires more time for training, it demonstrates superior performance in terms of the SIC AUC metric. More specifically, the vanilla EBM yields the lowest performance, but by gradually adding the dynamic masking mechanism and GRU network, the performance increases. This clearly shows that the combined effect of these mechanisms contributes to the better performance of EmSHAP.

\section{Conclusion}\label{Sec-5}
This paper proposes an energy-based model for Shapley value estimation (EmSHAP) for feature attribution in deep learning-based model. In EmSHAP, an energy-based model (EBM) is used to estimate the Shapley contribution function under different feature combinations. The EBM consists of a ResNet structure for approximating the energy function and a GRU network for estimating the partition distribution. By mapping the input features into a latent space using the GRU network, the influence of feature ordering on the calculation of the contribution function can be eliminated. In order to explore a wider range of feature combinations, a dynamic masking mechanism is introduced during training, which empowers EmSHAP with better estimation accuracy and generalization capability. In order to evaluate the performance of the proposed method, detailed theoretical analysis on error bound and extensive application studies are conducted, which show that EmSHAP has a tighter error bound and exhibits higher accuracy compared with state-of-the-art approaches.

Although EmSHAP provides a theoretically grounded and empirically competitive approach to Shapley-based feature attribution, several limitations still remain: i) the GRU architecture may be suboptimal for certain data modalities such as graph-structured or highly sparse features, highlighting the need for more flexible conditional modelling; ii) the quality of the proposal distribution directly affects conditional density estimation, the accuracy and stability may not be that good for hard-to-model feature subsets; iii) EmSHAP has a computational complexity that is proportional to the quadratic of input dimensions, hence it may need additional computational resources to handle high dimensional data sets; iv) the performance of EmSHAP may deteriorate under non-stationary data. Future research direction can be focused on: i) incorporating causal attribution methods to disentangle spurious correlations and generate counterfactual explanations; ii) designing adaptive sampling strategies that dynamically prioritize these hard-to-model subsets to improve the quality of the proposal distribution; iii) dimensionality-aware extensions to enhance applicability in high-dimensional domains; and iv) development of time-varying modelling approaches to handle non-stationary data.

\setcounter{equation}{0}
\renewcommand\theequation{A\arabic{equation}}
\setcounter{proposition}{0}
\renewcommand\theproposition{A\arabic{proposition}}
\setcounter{corollary}{0}
\renewcommand\thecorollary{A\arabic{corollary}}
\setcounter{lemma}{0}
\renewcommand\thetheorem{A\arabic{theorem}}
\setcounter{theorem}{0}
\renewcommand\thelemma{A\arabic{lemma}}

\section*{Appendices}
\subsection{Error Bound of KernelSHAP}
KernelSHAP believes that the calculation of the Shapley value can be regarded as an additive model, that is, the contribution function $v(S)$ of the variable set $S$ can be approximated by the sum of the variable weights within $S$,
\begin{equation}\label{eq-a-1}
\begin{small}
    v(S) \approx \beta_0 + \sum_{i\in S} \beta{i},
\end{small}
\end{equation}
where $\beta_0$ is the mean of $v(S)$, $\beta_i$ is the weight scalar of the $i$-th variable. By designing a linear regression model and solving it using the weighted least square method, the weight of each variable can be obtained. The optimization formula of KernelSHAP is as follows.
\begin{equation}\label{eq-a-2}
\begin{small}
    \mathop{\min}_{\beta_0,...,\beta_{|D|}} \sum_{S \in D} \psi(S) \left( \beta_0 + \sum_{i \in S} \beta_i - v(S) \right)^2,
\end{small}
\end{equation}
where $\boldsymbol{\beta}=[\beta_0,...,\beta_d]^T \in \mathbb{R}^{|D|+1}$ is the weight of each variable, $\psi(S) \in [0, 1]$ is called the Shapley kernel. To solve the above optimization problem, Covert \emph{et al.}~\cite{covert2020improving} introduced an binary vector $\mathbf{b} \in \mathbb{R}^{|D|+1}$, $b_i=1$ indicates the $i$-th variable is in the subset $S$, and otherwise $b_i=0$. Thus, Eq.(\ref{eq-a-2}) can be rewritten as follows.
\begin{equation}\label{eq-a-3}
\begin{small}
    \mathop{\min}_{\boldsymbol{\beta}} \sum_{S \in D} \psi(S) \left(\mathbf{b}^T\boldsymbol{\beta} - v(S) \right)^2.
\end{small}
\end{equation}

The matrix form of Eq.(\ref{eq-a-3}) is as follows.
\begin{equation}\label{eq-a-4}
\begin{small}
    \mathop{\min}_{\boldsymbol{\beta}}\frac{1}{K}(\mathbf{B}^T\boldsymbol{\beta}-\mathbf{v})^T\boldsymbol{\psi}(\mathbf{B}^T\boldsymbol{\beta}-\mathbf{v}),
\end{small}
\end{equation}
where $K$ is the number of sampling times, $\mathbf{B} \in \mathbb{R}^{(|D|+1) \times K}$ is a binary matrix indicating which variables are in $S$, $\boldsymbol{\psi} \in \mathbb{R}^{K \times K}$ is a diagonal matrix with Shapley kernel in its diagonal, $\mathbf{v} \in \mathbb{R}^{K \times 1}$ is the contribution function vector of different samples.

Eq.(\ref{eq-a-4}) can be solved by weighted least square method to obtain the estimation $\hat{\boldsymbol{\beta}}$, by taking the first-order derivative of $\boldsymbol{\beta}$ in Eq.(\ref{eq-a-4}) and setting it to zero, the derivation of weighted least square is given by,
\begin{equation}\label{eq-a-5}
\begin{small}
    \frac{1}{K}\left(\mathbf{B}\boldsymbol{\psi}\mathbf{B}^T {\boldsymbol{\beta}} - \mathbf{B}^T\boldsymbol{\psi}\mathbf{v}\right) = 0.
\end{small}
\end{equation}
For the details of KernelSHAP, interested readers can refer to Refs.~\cite{lundberg2017unified,covert2020improving}.

Let $\hat{\boldsymbol{\Sigma}} = \frac{1}{K}\mathbf{B}\boldsymbol{\psi}\mathbf{B}^T$ and $\hat{\boldsymbol{\Gamma}} = \frac{1}{K}\mathbf{B}^T\boldsymbol{\psi}\mathbf{v}$, $\mathbb{E}[\hat{\boldsymbol{\Sigma}}] = \boldsymbol{\Sigma}^*$ and $\mathbb{E}[\hat{\boldsymbol{\Gamma}}] = \boldsymbol{\Gamma}^*$. Assuming that $\hat{\boldsymbol{\Sigma}}$ and $\boldsymbol{\Sigma}^*$ is invertible, we have $\hat{\boldsymbol{\beta}}=\hat{\boldsymbol{\Sigma}}^{-1}\hat{\boldsymbol{\Gamma}}$, $\boldsymbol{\beta}^*=\mathbb{E}[\hat{\boldsymbol{\Sigma}}^{-1}\hat{\boldsymbol{\Gamma}}]$. The error bound of KernelSHAP can be described as the mean absolute deviation between the estimated contribution function $\hat{v}(S)$ and the true contribution function $v(S)$, this error bound is also related to the $\hat{\boldsymbol{\beta}}$ and $\boldsymbol{\beta}^*$ as follows.
\begin{equation}\label{eq-a-5-0}
\begin{small}
        \mathbb{E}\left[|\hat{v}(S)-v(S)|\right] = \mathbb{E}\left[||\hat{\boldsymbol{\beta}}-\boldsymbol{\beta}^*||\right].
\end{small}
\end{equation}
%
The decomposition of Eq.(\ref{eq-a-5-0}) is given by,
\begin{equation}\label{eq-a-6}
\begin{small}
    \begin{aligned}
        & \quad \mathbb{E}\left[|\hat{v}(S)-v(S)|\right] \\
        & =\mathbb{E} \left[ ||\hat{\boldsymbol{\beta}}-\boldsymbol{\beta}^*||\right] = \mathbb{E}[|| \hat{\boldsymbol{\Sigma}}^{-1} \hat{\boldsymbol{\Gamma}} - \mathbb{E}[\hat{\boldsymbol{\Sigma}}^{-1} \hat{\boldsymbol{\Gamma}}] ||] \\
        & = \mathbb{E}[|| \hat{\boldsymbol{\Sigma}}^{-1} \hat{\boldsymbol{\Gamma}} - \mathbb{E}[\hat{\boldsymbol{\Sigma}}^{-1}]  \mathbb{E}[\hat{\boldsymbol{\Gamma}}] - \text{Cov}(\hat{\boldsymbol{\Sigma}}^{-1}, \hat{\boldsymbol{\Gamma}}) ||] \\
        & = \mathbb{E}[|| \hat{\boldsymbol{\Sigma}}^{-1} \hat{\boldsymbol{\Gamma}} - \boldsymbol{\Sigma}^{*^{-1}} \boldsymbol{\Gamma}^* - \text{Cov}(\hat{\boldsymbol{\Sigma}}^{-1}, \hat{\boldsymbol{\Gamma}}) ||],
    \end{aligned}
\end{small}
\end{equation}
where $|| \cdot ||$ is the $L_2$ norm and $\text{Cov}(\cdot, \cdot)$ is the covariance. Next we introduce three corollaries and show that $||\text{Cov}(\hat{\boldsymbol{\Sigma}}^{-1}, \hat{\boldsymbol{\Gamma}})|| \leq \epsilon_3^2$ in some specific cases, where $\epsilon_3$ is a small positive number.

\noindent
\begin{corollary}\label{cly-a-1}
    (Concentration of $\hat{\boldsymbol{\Sigma}}$) For all $\epsilon_3>0$, we have,
    \begin{equation}\label{eq-a-20}
    \begin{small}
        \mathbb{P}(||\hat{\boldsymbol{\Sigma}}-\boldsymbol{\Sigma}^*|| \geq \epsilon_3) \leq 2(|D|+1)^2\exp(-2K\epsilon_3^2).
    \end{small}
    \end{equation}
\end{corollary}

\noindent
{\bf Proof.}{
    Noting that the elements in $\mathbf{B}$ are 0 or 1 and the elements in $\boldsymbol{\psi}$ is bounded with $[0, 1]$. Thus the element in $\hat{\boldsymbol{\Sigma}}$ and $\boldsymbol{\Sigma}^*$ also bounded with $[0, 1]$. According to Hoeffding$^{\prime}$s inequality, each element in $\hat{\boldsymbol{\Sigma}}$ is bounded with their expectation values, that is, for all $\epsilon_3>0$
    \begin{equation}\label{eq-a-21}
    \begin{small}
        \begin{cases}
            \mathbb{P}(|\frac{1}{K} \sum_{k=1}^{K}(\hat{\boldsymbol{\Sigma}}_{ii}-\mathbb{E}[\hat{\boldsymbol{\Sigma}}_{ii}]) |\geq \epsilon_3) \leq 2\exp (-2K\epsilon_3^2) \\
            \mathbb{P}(|\frac{1}{K} \sum_{k=1}^{K}(\hat{\boldsymbol{\Sigma}}_{ij}-\mathbb{E}[\hat{\boldsymbol{\Sigma}}_{ij}]) |\geq \epsilon_3) \leq 2\exp (-2K\epsilon_3^2),
        \end{cases}
    \end{small}
    \end{equation}
    where, $\mathbb{E}[\hat{\boldsymbol{\Sigma}}_{ii}] = \frac{1}{2}$ and $\mathbb{E}[\hat{\boldsymbol{\Sigma}}_{ij}] = \alpha$, so
    \begin{equation}\label{eq-a-22}
    \begin{small}
        \begin{cases}
            \mathbb{P}(|\frac{1}{K} \sum_{k=1}^{K}(\hat{\boldsymbol{\Sigma}}_{ii}-\frac{1}{2}) |\geq \epsilon_3) \leq 2\exp (-2K\epsilon_3^2) \\
            \mathbb{P}(|\frac{1}{K} \sum_{k=1}^{K}(\hat{\boldsymbol{\Sigma}}_{ij}-\alpha) |\geq \epsilon_3) \leq 2\exp (-2K\epsilon_3^2).
        \end{cases}
    \end{small}
    \end{equation}
    By using Boole$^{\prime}$s inequality~\cite{boole1847mathematical}, Eq.(\ref{eq-a-20}) can be obtained.
}

The concentration of $\hat{\boldsymbol{\Gamma}}$ and $\hat{\boldsymbol{\Sigma}}^{-1}$ can be obtained in the same.
\noindent
\begin{corollary}\label{cly-a-2}
    (Concentration of $\hat{\boldsymbol{\Gamma}}$)For all $\epsilon_3>0$, we have,
    \begin{equation}\label{eq-a-23}
    \begin{small}
        \mathbb{P}(||\hat{\boldsymbol{\Gamma}}-\boldsymbol{\Gamma}^*|| \geq \epsilon_3) \leq 2(|D|+1)\exp(-2K\epsilon_3^2).
    \end{small}
    \end{equation}
\end{corollary}

\noindent
{\bf Proof.}{
    According to Assumption 4, the contribution value of any subset $S$ is bounded within $[0, 1]$, that is $v(S) \in [0, 1]$. Then $\mathbf{v}$ is also bounded within $[0, 1]$, which indicates $\boldsymbol{\Gamma}^* \in [0, 1]$, by using Hoeffding$^{\prime}$s inequality, each element in $\hat{\boldsymbol{\Gamma}}$ is bounded with their expectation values, that is, for all $\epsilon>0$
    \begin{equation}\label{eq-a-24}
    \begin{small}
        \mathbb{P}(|\frac{1}{K} \sum_{k=1}^{K}(\hat{\boldsymbol{\Gamma}}_{i}-\mathbb{E}[\hat{\boldsymbol{\Gamma}}_{i}]) |\geq \epsilon_3) \leq 2\exp (-2K\epsilon_3^2).
    \end{small}
\end{equation}
By using Boole$^{\prime}$s inequality~\cite{boole1847mathematical}, the concentration of $\hat{\boldsymbol{\Gamma}}$ can be obtained.
}

\noindent
\begin{corollary}
\label{cly-a-2-0}
    (Concentration of $\hat{\boldsymbol{\Sigma}}^{-1}$) For all $\epsilon_3>0$, we have,
    \begin{equation}\label{eq-a-23-0}
    \begin{small}
        \mathbb{P}(||\hat{\boldsymbol{\Sigma}}^{-1}-{\boldsymbol{\Sigma}^*}^{-1}|| \geq \epsilon_3) \leq 2(|D|+1)^2\exp(-2K\epsilon_3^2).
    \end{small}
    \end{equation}
\end{corollary}

\noindent
\begin{corollary}\label{cly-a-3}
    For the fixed input dimension $|D|$, for any small positive number $\epsilon_3$, there exists a large $K$ that
    \begin{equation}\label{eq-a-24-0}
    \begin{small}
        \mathbb{P}\left(|| \mathrm{Cov}(\hat{\boldsymbol{\Sigma}}^{-1}, \hat{\boldsymbol{\Gamma}}) ||\geq \epsilon^2_3\right) \leq 2((|D|+1)^2+(|D|+1)) \exp(-2K \epsilon_3^2),
    \end{small}
    \end{equation}
    where $\epsilon_3 \in \left[ \right. \sqrt{\frac{1}{2K}\log(2((|D|+1)^2+(|D|+1)))}, +\infty \left. \right)$.
\end{corollary}

\noindent
{\bf Proof.}{
    $\left|\left| \text{Cov}(\hat{\boldsymbol{\Sigma}}^{-1}, \hat{\boldsymbol{\Gamma}}) \right| \right|$ can be decomposed as follows.
    \begin{equation}
    \begin{small}
        \quad \left|\left| \text{Cov}(\hat{\boldsymbol{\Sigma}}^{-1}, \hat{\boldsymbol{\Gamma}}) \right| \right|
        =\left|\left|\mathbb{E}\left[ \left( \hat{\boldsymbol{\Sigma}}^{-1} - \mathbb{E}[\hat{\boldsymbol{\Sigma}}^{-1}] \right) \left( \hat{\boldsymbol{\Gamma}} - \mathbb{E}[\hat{\boldsymbol{\Gamma}}] \right) \right]\right| \right|.
    \end{small}
    \end{equation}
    According to Corollaries~\ref{cly-a-2}, \ref{cly-a-2-0}, and Boole$^{\prime}$s inequality, Eq.(\ref{eq-a-24-0}) can be obtained. Since the probability $\mathbb{P}\left(|| \text{Cov}(\hat{\boldsymbol{\Sigma}}^{-1}, \hat{\boldsymbol{\Gamma}}) ||\geq \epsilon_3\right) \in [0, 1]$, the range of $\epsilon_3$ is $\left[ \right. \sqrt{\frac{1}{2K}\log(2((|D|+1)^2+(|D|+1)))}, +\infty \left. \right)$.

    It should be notice that for a small positive $\epsilon_3$, when $K$ is large enough, the probability $1- 2((|D|+1)^2+(|D|+1)) \exp(-2K \epsilon_3^2)$ close to 1, which means $\left| \left|\text{Cov}(\hat{\boldsymbol{\Sigma}}^{-1}, \hat{\boldsymbol{\Gamma}}) \right|\right| \leq \epsilon_3^2$ holds true under almost any conditions.
}

Thus, Eq.(\ref{eq-a-6}) can be rewritten as follows,
\begin{equation}\label{eq-a-24-3}
\begin{small}
        \begin{aligned}
        & \quad \mathbb{E}\left[|\hat{v}(S)-v(S)|\right] \\
        & \leq \mathbb{E} [|| \hat{\boldsymbol{\Sigma}}^{-1}\hat{\boldsymbol{\Gamma}} - \boldsymbol{\Sigma}^{*^{-1}}\boldsymbol{\Gamma}^* ||] + \mathbb{E}[|| -\text{Cov}(\hat{\boldsymbol{\Sigma}}^{-1}, \hat{\boldsymbol{\Gamma}}) ||] \\
        & \leq \mathbb{E}[|| \hat{\boldsymbol{\Sigma}}^{-1}(\hat{\boldsymbol{\Gamma}} - \boldsymbol{\Gamma}^*) ||] + \mathbb{E}[||(\hat{\boldsymbol{\Sigma}}^{-1} - \boldsymbol{\Sigma}^{*^{-1}})\boldsymbol{\Gamma}^* ||] + \epsilon_3^2\\
        & \leq \mathbb{E}[||\hat{\boldsymbol{\Sigma}}^{-1} || \cdot || \hat{\boldsymbol{\Gamma}} - \boldsymbol{\Gamma}^* ||] + \mathbb{E}[|| \hat{\boldsymbol{\Sigma}}^{-1} - \boldsymbol{\Sigma}^{*^{-1}} || \cdot || \boldsymbol{\Gamma}^*||] + \epsilon_3^2\\
        & = \mathbb{E}[|| \hat{\boldsymbol{\Sigma}}^{-1}(\hat{\boldsymbol{\Gamma}} - \boldsymbol{\Gamma}^*) ||] + \mathbb{E}[||\boldsymbol{\Sigma}^{*^{-1}}(\hat{\boldsymbol{\Sigma}}-\boldsymbol{\Sigma}^*)\hat{\boldsymbol{\Sigma}}^{-1}|| \cdot ||\boldsymbol{\Gamma}^* ||] + \epsilon_3^2 \\
        & \leq \mathbb{E}[||\hat{\boldsymbol{\Sigma}}^{-1} || \cdot || \hat{\boldsymbol{\Gamma}} - \boldsymbol{\Gamma}^* ||] + \mathbb{E}[||\boldsymbol{\Sigma}^{*^{-1}}||\cdot ||\hat{\boldsymbol{\Sigma}}-\boldsymbol{\Sigma}^*||\cdot ||\hat{\boldsymbol{\Sigma}}^{-1}||\cdot
        ||\boldsymbol{\Gamma}^*||] + \epsilon_3^2.
        \end{aligned}
\end{small}
\end{equation}

By calculating the upper bounds of the decomposed elements in Eq.(\ref{eq-a-24-3}) separately and merging them, the upper bound of $\mathbb{E}\left[|\hat{v}(S)-v(S)|\right]$ can be obtained. Notice that Ref.\cite{garreau20a} gives the theoretical analysis of LIME, we follow their contribution and derive the error bound of KernelSHAP.
Here, the explicit expressions of $\boldsymbol{\Sigma}^*$ and its inverse form $\boldsymbol{\Sigma}^*$ are first into consideration. According to Ref.~\cite{covert2020improving}, $\boldsymbol{\Sigma}^*$ has the following form,
\begin{equation}\label{eq-a-7}
\begin{small}
    \boldsymbol{\Sigma}^* =
    \begin{bmatrix}
        \frac{1}{2} & \alpha       & \cdots & \alpha  \\
        \alpha      & \frac{1}{2}  & \cdots & \alpha \\
        \vdots      & \vdots       & \ddots  & \vdots \\
        \alpha      & \alpha       & \cdots  & \frac{1}{2},
    \end{bmatrix}
\end{small}
\end{equation}
that is, the diagonal of $\boldsymbol{\Sigma}^*$ is $\frac{1}{2}$, and the  off-diagonal entries are $\alpha$, where
\begin{equation}\label{eq-a-8}
\begin{small}
    \alpha = \frac{1}{|D|(|D|+1)} \frac{\sum_{i=2}^{|D|}\frac{i-1}{|D|+1-i}}{\sum_{i=1}^{|D|}\frac{1}{i(|D|+1-i)}}.
\end{small}
\end{equation}

\noindent
\begin{lemma}\label{lma-a-1}
    Assuming $\boldsymbol{\Sigma}^*$ is invertible, the inverse of $\boldsymbol{\Sigma}^*$, $\boldsymbol{\Sigma}^{*^{-1}}$, has the following expression,
    \begin{equation}\label{eq-a-9}
    \begin{small}
        \boldsymbol{\Sigma}^{*^{-1}} =
        \begin{bmatrix}
        \frac{1}{\frac{1}{2}-\alpha} & \frac{\alpha}{(\alpha-\frac{1}{2})(K\alpha-\alpha+\frac{1}{2})}      & \cdots & \frac{\alpha}{(\alpha-\frac{1}{2})(K\alpha-\alpha+\frac{1}{2})}  \\
        \frac{\alpha}{(\alpha-\frac{1}{2})(K\alpha-\alpha+\frac{1}{2})}      & \frac{1}{\frac{1}{2}-\alpha}  & \cdots & \frac{\alpha}{(\alpha-\frac{1}{2})(K\alpha-\alpha+\frac{1}{2})} \\
        \vdots      & \vdots       & \ddots  & \vdots \\
        \frac{\alpha}{(\alpha-\frac{1}{2})(K\alpha-\alpha+\frac{1}{2})}      & \frac{\alpha}{(\alpha-\frac{1}{2})(K\alpha-\alpha+\frac{1}{2})}       & \cdots  & \frac{1}{\frac{1}{2}-\alpha}
    \end{bmatrix}.
    \end{small}
    \end{equation}
\end{lemma}

\noindent
{\bf Proof.}{
    According to Eq.(\ref{eq-a-7}), $\boldsymbol{\Sigma}^*$ can be decomposed as follows,
    \begin{equation}\label{eq-a-10}
    \begin{small}
        \boldsymbol{\Sigma}^* = (\frac{1}{2}-\alpha)\mathbf{I} + \alpha\mathbf{J},
    \end{small}
    \end{equation}
    where $\mathbf{I} \in \mathbb{R}^{(|D|+1) \times (|D|+1)}$ is a identity matrix and $\mathbf{J} \in \mathbb{R}^{(|D|+1) \times (|D|+1)}$ is a matrix with all ones.
    Let $\boldsymbol{\Sigma}^{*{-1}} = \kappa\mathbf{I}+\omega\mathbf{J}$, $\kappa$ and $\omega$ are need to be determined, we have,
    \begin{equation}\label{eq-a-11}
    \begin{small}
        \left(\left( \frac{1}{2}-\alpha \right)\mathbf{I}+\alpha\mathbf{J}\right)\left(\kappa\mathbf{I}+\omega\mathbf{J}\right) = \mathbf{I}.
    \end{small}
    \end{equation}
    By sorting out the above formula,
    \begin{equation}\label{eq-a-12}
    \begin{small}
        \left(\frac{1}{2}-\alpha\right)\kappa\mathbf{I}+\left[ \left(\frac{1}{2}-\alpha+K\alpha\right)\omega + \kappa\alpha \right]\mathbf{J} = \mathbf{I}.
    \end{small}
    \end{equation}
    Thus, we can get,
    \begin{equation}\label{eq-a-13}
    \begin{small}
        \begin{cases}
            \begin{aligned}
                & (\frac{1}{2}-\alpha)\kappa = 1 \\
                & (\frac{1}{2}-\alpha+K\alpha)\omega + \kappa\alpha = 0.
            \end{aligned}
        \end{cases}
    \end{small}
    \end{equation}
    By solving Eq.(\ref{eq-a-13}), we have,
    \begin{equation}\label{eq-a-14}
    \begin{small}
        \kappa = \frac{1}{\frac{1}{2}-\alpha}, \quad \omega = \frac{\alpha}{(\alpha-\frac{1}{2})(K\alpha-\alpha+\frac{1}{2})},
    \end{small}
    \end{equation}
    and
    \begin{equation}\label{eq-a-15}
    \begin{small}
        \boldsymbol{\Sigma}^{*^{-1}} = \frac{1}{\frac{1}{2}-\alpha}\mathbf{I} + \frac{\alpha}{(\alpha-\frac{1}{2})(K\alpha-\alpha+\frac{1}{2})}\mathbf{J}.
    \end{small}
    \end{equation}
}

\noindent
\begin{proposition}\label{pro-a-1}
    (The upper bound of $|| \boldsymbol{\Sigma}^{*{-1}} ||$) The $L_2$ norm of $\boldsymbol{\Sigma}^{*^{-1}}$ is upper bounded with,
\begin{equation}\label{eq-a-16}
\begin{small}
    ||\boldsymbol{\Sigma}^{*^{-1}}|| \leq \sqrt{\frac{|D|+1}{(\frac{1}{2} - \alpha)^2} + \frac{(|D|+1)|D|\alpha^2}{((\alpha-\frac{1}{2})(K\alpha-\alpha+\frac{1}{2}))^2}}.
\end{small}
\end{equation}
\end{proposition}

\noindent
{\bf Proof.}{
    The square $L_2$ norm of $\boldsymbol{\Sigma}^{*^{-1}}$ is upper bounded by its Frobenius norm, that is,
    \begin{equation}\label{eq-a-17}
    \begin{small}
        ||\boldsymbol{\Sigma}^{*^{-1}}||^2 \leq ||\boldsymbol{\Sigma}^{*^{-1}}||^2_F = \frac{|D|+1}{(\frac{1}{2} - \alpha)^2} + \frac{(|D|+1)|D|\alpha^2}{((\alpha-\frac{1}{2})(K\alpha-\alpha+\frac{1}{2}))^2}.
    \end{small}
    \end{equation}
    By taking the root square on both sides of Eq.(\ref{eq-a-16}), we can get Eq.(\ref{eq-a-17}).
}

Similarly, the upper bound of $||\boldsymbol{\Gamma}^{*}||$ can be obtained in the same way.

\noindent
\begin{proposition}\label{pro-a-2}
(Upper bound of $||\boldsymbol{\Gamma}^{*}||$) The $L_2$ norm of $\boldsymbol{\Gamma}^{*}$ is upper bounded with,
    \begin{equation}\label{eq-a-18}
    \begin{small}
        ||\boldsymbol{\Gamma}^* || \leq \sqrt{|D|+1}.
    \end{small}
    \end{equation}
\end{proposition}

\noindent
{\bf Proof.}{
    The square $L_2$ norm of $\boldsymbol{\Gamma}^* $ is as follows.
    \begin{equation}\label{eq-a-19}
    \begin{small}
        ||\boldsymbol{\Gamma}^* ||^2 = || \mathbf{B} \boldsymbol{\psi} \mathbf{v} ||^2.
    \end{small}
    \end{equation}
    Since each element in $\boldsymbol{\Gamma}^*$ is bounded within $[0, 1]$, the square $L_2$ norm of $\boldsymbol{\Gamma}^* $ is smaller then its dimension value, that is, $||\boldsymbol{\Gamma}^* ||^2\leq|D|+1$.
}

\noindent
\begin{proposition}\label{pro-a-3}
(Upper bound of $\hat{\boldsymbol{\Sigma}}^{-1}$) For all $\epsilon_3 \in (0, \frac{1}{||\boldsymbol{\Sigma}||_F^{*^{-1}}})$, with the probability greater than $1-2(|D|+1)^2\exp(-2K\epsilon_3^2)$, there exist a small positive number $\epsilon_3^{\prime}$ that
    \begin{equation}\label{eq-a-25}
    \begin{small}
        ||\hat{\boldsymbol{\Sigma}}^{-1} || \leq \frac{||\boldsymbol{\Sigma}^{*^{-1}}||_F}{1-\epsilon_3||\boldsymbol{\Sigma}^{*^{-1}}||_F} = ||\boldsymbol{\Sigma}^{*^{-1}}||_F+\epsilon_3^{\prime},
    \end{small}
    \end{equation}
    where $\epsilon_3^{\prime}= \frac{\epsilon_3||\boldsymbol{\Sigma}^{*^{-1}}||_F^2}{1-\epsilon_3||\boldsymbol{\Sigma}^{*^{-1}}||_F}$, since $\epsilon_3$ is a small positive value, $\epsilon_3^{\prime}$ could be small as well.
\end{proposition}

\noindent
{\bf Proof.}{
    According to Lemma~\ref{lma-a-1}, we have $\lambda_{\max}(\boldsymbol{\Sigma}^{*^{-1}}) \leq ||\boldsymbol{\Sigma}^{*^{-1}}||_F$
    Thus, we have,
    \begin{equation}\label{eq-a-26}
    \begin{small}
        \lambda_{\min}(\boldsymbol{\Sigma}^{*}) \geq \frac{1}{||\boldsymbol{\Sigma}^{*^{-1}}||_F}.
    \end{small}
    \end{equation}
    Particularly, $\hat{\boldsymbol{\Sigma}}$ is concentrated on $\boldsymbol{\Sigma}^*$, by using Weyl$^{\prime}$ inequality~\cite{weyl1912asymptotische},
    \begin{equation}\label{eq-a-27}
    \begin{small}
        |\lambda_{\min}(\hat{\boldsymbol{\Sigma}})- \lambda_{\min}(\boldsymbol{\Sigma}^*)| \leq || \hat{\boldsymbol{\Sigma}} - \boldsymbol{\Sigma}^* || \leq \epsilon_3,
    \end{small}
    \end{equation}
    and,
    \begin{equation}\label{eq-a-28}
    \begin{small}
        \lambda_{\min}(\hat{\boldsymbol{\Sigma}}) \geq \lambda_{\min}(\boldsymbol{\Sigma}^*)-\epsilon_3 \geq \frac{1}{||\boldsymbol{\Sigma}^{*^{-1}}||_F} -\epsilon_3,
    \end{small}
    \end{equation}
    which can be deduced to
    \begin{equation}\label{eq-a-29}
    \begin{small}
        ||\hat{\boldsymbol{\Sigma}}^{-1} || \leq \frac{||\boldsymbol{\Sigma}^{*^{-1}}||_F}{1-\epsilon_3||\boldsymbol{\Sigma}^{*^{-1}}||_F}.
    \end{small}
    \end{equation}
    Furthermore,
    \begin{equation}
    \begin{small}
            \begin{aligned}
                & \quad \frac{||\boldsymbol{\Sigma}^{*^{-1}}||_F}{1-\epsilon_3||\boldsymbol{\Sigma}^{*^{-1}}||_F} = \frac{||\boldsymbol{\Sigma}^{*^{-1}}||_F - \epsilon_3||\boldsymbol{\Sigma}^{*^{-1}}||_F^2+\epsilon_3|| \boldsymbol{\Sigma}^{*^{-1}}||_F^2}{1-\epsilon_3||\boldsymbol{\Sigma}^{*^{-1}}||_F} \\
                & = ||\boldsymbol{\Sigma}^{*^{-1}}||_F + \frac{\epsilon_3||\boldsymbol{\Sigma}^{*^{-1}}||_F^2}{1-\epsilon_3||\boldsymbol{\Sigma}^{*^{-1}}||_F}.
            \end{aligned}
    \end{small}
    \end{equation}
    By defining $\epsilon_3^{\prime} = \frac{\epsilon_3||\boldsymbol{\Sigma}^{*^{-1}}||_F^2}{1-\epsilon_3||\boldsymbol{\Sigma}^{*^{-1}}||_F}$, we can get the upper bound $||\hat{\boldsymbol{\Sigma}}^{-1} || \leq |\hat{\boldsymbol{\Sigma}}^{-1} ||_F + \epsilon_3^{\prime}$
    Since $\epsilon_3$ is a small positive number, $\epsilon_3^{\prime}$ is small as well.
}

\noindent
{\bf Proof.}{
(Proof of Theorem 2)
For fixed feature dimension $|D|$, a small positive $\epsilon_3$, and a large sampling number $K$, the final result of $\mathbb{E}\left[|\hat{v}(S)-v(S)|\right]$ with the probability greater than $1- 2((|D|+1)^2+(|D|+1)) \exp(-2K \epsilon_3^2)$ is as follows.
\begin{equation}\label{eq-a-30}
\begin{small}
    \begin{aligned}
        & \quad \mathbb{E}\left[|\hat{v}(S)-v(S)|\right]\\
        & \leq \mathbb{E}\left[||\hat{\boldsymbol{\Sigma}}^{-1} || \cdot || \hat{\boldsymbol{\Gamma}} - \boldsymbol{\Gamma}^* ||\right] + \mathbb{E}\left[||\boldsymbol{\Sigma}^{*^{-1}}||\cdot ||\hat{\boldsymbol{\Sigma}}-\boldsymbol{\Sigma}^*||\cdot ||\hat{\boldsymbol{\Sigma}}^{-1}||\cdot
        ||\boldsymbol{\Gamma}^*||\right] +\epsilon_3^2\\
        & \leq ||\boldsymbol{\Sigma}^{*^{-1}}||_F \cdot \int_0^{+\infty}2(|D|+1)\exp(-2K\epsilon_3^2)d\epsilon_3 +\epsilon_3^2\\ 
        &+ ||\boldsymbol{\Sigma}^{*^{-1}}||_F \cdot \int_0^{+\infty}2(|D|+1)^2\exp(-2K\epsilon_3^2)d\epsilon_3 \\
        & \cdot (||\boldsymbol{\Sigma}^{*^{-1}}||_F+\epsilon_3^{\prime}) \cdot \sqrt{|D|+1} \\
        & = 2(|D|+1)\frac{\sqrt{\pi}}{\sqrt{2K}} ||\boldsymbol{\Sigma}^{*^{-1}}||_F \\
\notag
    \end{aligned}
\end{small}
\end{equation}
\begin{equation}
\begin{small}
    \begin{aligned}
        & + 2(|D|+1)^{\frac{5}{2}} \frac{\sqrt{\pi}}{\sqrt{2K}} ||\boldsymbol{\Sigma}^{*^{-1}}||_F(||\boldsymbol{\Sigma}^{*^{-1}}||_F+\epsilon_3^{\prime}) +\epsilon_3^2.
    \end{aligned}
\end{small}
\end{equation}
}

It can be seen from Eq.(\ref{eq-a-30}) that as the data dimension $|D|$ increases, the error estimated by KernelSHAP will increase. On the other hand, the estimation error will decrease as the number of samples $K$ increases.

\setcounter{equation}{0}
\renewcommand\theequation{B\arabic{equation}}
\setcounter{proposition}{0}
\renewcommand\theproposition{B\arabic{proposition}}
\setcounter{corollary}{0}
\renewcommand\thecorollary{B\arabic{corollary}}
\setcounter{lemma}{0}
\renewcommand\thelemma{B\arabic{lemma}}
\subsection{Error Bound of Variational autoencoder}
Variational autoencoder (VAE) is a typical generative model, which uses a latent variable model to give a probabilistic representation of the unknown data distribution $p(\mathbf{x})$. The basic idea of the variational autoencoder is to generate the observed variable $\mathbf{x}$ through a simple distributed latent variable $\mathbf{z}$. Specifically, VAE first models the joint distribution,
\begin{equation}\label{eq-b-1}
\begin{small}
    p(\mathbf{x},\mathbf{z}) = p(\mathbf{z})p_{\theta}(\mathbf{x}|\mathbf{z}),
\end{small}
\end{equation}
where the conditional distribution $p_{\theta}(\mathbf{x}|\mathbf{z})$ is called an encoder, which is modeled by a deep neural network with parameter $\theta$,$\mathbf{z}$ is the latent variable with a simple distribution, such as the standard normal distribution. By marginalizing the latent variable, the data distribution can be approximated.
\begin{equation}\label{eq-b-2}
\begin{small}
    p(\mathbf{x}) = \int p(\mathbf{x},\mathbf{z}) d\mathbf{z}.
\end{small}
\end{equation}
This latent variable model helps to get $p(\mathbf{x})$ easily. However, the calculation of the posterior distribution $p_{\theta}(\mathbf{z}|\mathbf{x})$ during training becomes intractable. To address this challenge, VAE constructs a distribution $q_{\tau}(\mathbf{z}|\mathbf{x})$ to approximate $p_{\theta}(\mathbf{z}|\mathbf{x})$, which is called the decoder parametered with $\tau$. The data log-likelihood of VAE can be expressed as follows.
\begin{equation}\label{eq-b-3}
\begin{small}
    \begin{aligned}
        \log p_{\theta}(\mathbf{x}) &= \log p_{\theta}(\mathbf{x}) \int q_{\tau}(\mathbf{z}|\mathbf{x})d\mathbf{z} \\
        & = \int q_{\tau}(\mathbf{z}|\mathbf{x})\log p_{\theta}(\mathbf{x})d\mathbf{z} \\
        & = \mathbb{E}_{q_{\tau}(\mathbf{z}|\mathbf{x})} \log p_{\theta}(\mathbf{x}) \\
        & = \mathbb{E}_{q_{\tau}(\mathbf{z}|\mathbf{x})} \log \frac{p_{\theta}(\mathbf{x},\mathbf{z})}{p_{\theta}(\mathbf{z}|\mathbf{x})} \\
        & = \mathbb{E}_{q_{\tau}(\mathbf{z}|\mathbf{x})} \log \frac{p_{\theta}(\mathbf{x},\mathbf{z})q_{\tau}(\mathbf{z}|\mathbf{x})}{p_{\theta}(\mathbf{z}|\mathbf{x})q_{\tau}(\mathbf{z}|\mathbf{x})} \\
        & = \mathbb{E}_{q_{\tau}(\mathbf{z}|\mathbf{x})} \log \frac{p_{\theta}(\mathbf{x},\mathbf{z})}{q_{\tau}(\mathbf{z}|\mathbf{x})} + \mathbb{E}_{q_{\tau}(\mathbf{z}|\mathbf{x})}\log \frac{q_{\tau}(\mathbf{z}|\mathbf{x})}{p_{\theta}(\mathbf{z}|\mathbf{x})} \\
        & = \mathbb{E}_{q_{\tau}(\mathbf{z}|\mathbf{x})} \log \frac{p_{\theta}(\mathbf{x},\mathbf{z})}{q_{\tau}(\mathbf{z}|\mathbf{x})} + D_{KL}\left(q_{\tau}(\mathbf{z}|\mathbf{x}) || p_{\theta}(\mathbf{z}|\mathbf{x})\right),
    \end{aligned}
\end{small}
\end{equation}
where $\mathbb{E}_{q_{\tau}(\mathbf{z}|\mathbf{x})} \log \frac{p_{\theta}(\mathbf{x},\mathbf{z})}{q_{\tau}(\mathbf{z}|\mathbf{x})}$ is called evidence lower bound (ELBO) and $D_{KL}(\cdot||\cdot)$ is Kullback-Leibler (KL) divergence. ELBO is the objective function that needs to be optimized by maximizing ELBO via gradient descent, and a fine-trained VAE can be obtained. $D_{KL}\left(q_{\tau}(\mathbf{z}|\mathbf{x}) || p_{\theta}(\mathbf{z}|\mathbf{x})\right)$ shows the gap between log-likelihood and ELBO, which is exactly the mismatch between the VAE encoder and the posterior\cite{shekhovtsov2022vae}.

Ref.~\cite{olsen2022using} introduces a Shapley value approximation approach based on VAE. The contribution function of Shapley value is sampled from the conditional probability $p_{\theta}(\mathbf{x}_{\bar{S}}|\mathbf{x}_S)$ generated by the VAE, and the sampling method is just like the method proposed in this article. The log-likelihood introduced in Ref.~\cite{olsen2022using} is as follows.
\begin{equation}\label{eq-b-4}
\begin{small}
\begin{aligned}
    & \quad \log p_{\theta}(\mathbf{x}_{\bar{S}}|\mathbf{x}_S) \\
    & = \mathbb{E}_{q_{\tau}(\mathbf{z}|\mathbf{x}_{\bar{S}},\mathbf{x}_S)}\left[\log\frac{p_{\theta}(\mathbf{x}_{\bar{S}},\mathbf{z}|\mathbf{x}_S)}{q_{\tau}(\mathbf{z}|\mathbf{x}_{\bar{S}},\mathbf{x}_S)}\right] +\mathbb{E}_{q_{\tau}(\mathbf{z}|\mathbf{x}_{\bar{S}},\mathbf{x}_S)}\left[\log \frac{q_{\tau}(\mathbf{z}|\mathbf{x}_{\bar{S}},\mathbf{x}_S)}{p_{\theta}(\mathbf{z}|\mathbf{x}_{\bar{S}},\mathbf{x}_S)}\right]\\
    & = \mathbb{E}_{q_{\tau}(\mathbf{z}|\mathbf{x}_{\bar{S}},\mathbf{x}_S)}\left[\log\frac{p_{\theta}(\mathbf{x}_{\bar{S}},\mathbf{z}|\mathbf{x}_S)}{q_{\tau}(\mathbf{z}|\mathbf{x}_{\bar{S}},\mathbf{x}_S)}\right]+ D_{KL}(q_{\tau}(\mathbf{z}|\mathbf{x}_{\bar{S}},\mathbf{x}_S) || p_{\theta}(\mathbf{z}|\mathbf{x}_{\bar{S}},\mathbf{x}_S)) \\ 
    & \geq \mathbb{E}_{q_{\tau}(\mathbf{z}|\mathbf{x}_{\bar{S}},\mathbf{x}_S)}\left[\log\frac{p_{\theta}(\mathbf{x}_{\bar{S}},\mathbf{z}|\mathbf{x}_S)}{q_{\tau}(\mathbf{z}|\mathbf{x}_{\bar{S}},\mathbf{x}_S)}\right].
\end{aligned}
\end{small}
\end{equation}

This VAE-based Shapley value estimation method trains the parameters in the VAE by maximizing the first term in Eq.(\ref{eq-b-4}), and then samples from the generated conditional distribution $p_{\theta}(\mathbf{x}_{\bar{S}}|\mathbf{x}_S)$ to calculate the contribution function $v$ in the Shapley value and the Shapley values of each input.

However, a problem arises here, that the generated conditional distribution $p_{\theta}(\mathbf{x}_{\bar{S}}|\mathbf{x}_S)$ convergences to the true conditional distribution $p_{data}(\mathbf{x}_{\bar{S}}|\mathbf{x}_S)$ if and only if $ D_{KL}(q_{\tau}(\mathbf{z}|\mathbf{x}_{\bar{S}},\mathbf{x}_S) || p_{\theta}(\mathbf{z}|\mathbf{x}_{\bar{S}},\mathbf{x}_S)) $ goes to zeros. Unfortunately, it’s not explicit to ensure this happens, since $q_{\tau}(\mathbf{z}|\mathbf{x}_{\bar{S}},\mathbf{x}_S)$ is modeled as exponential families, i.e. multivariate Gaussian distribution, the posterior distribution $p_{\theta}(\mathbf{z}|\mathbf{x})$ is not always possible to use the exponential families to approximate. Thus, for an unknown data distribution $p(\mathbf{x})$, $ D_{KL}(q_{\tau}(\mathbf{z}|\mathbf{x}_{\bar{S}},\mathbf{x}_S) || p_{\theta}(\mathbf{z}|\mathbf{x}_{\bar{S}},\mathbf{x}_S)) $ is hardly going to zeros, which indicates that there is a inference gap between the true distribution and the generated distribution.

According to Ref.~\cite{cremer2018inference}, the inference gap of VAE is decomposed into two components, the approximation gap $\log p_{data}(\mathbf{x})-\log p_{\theta^*}(\mathbf{x})$ and the amortization gap $\log p_{\theta^*}(\mathbf{x})-\log p_{\theta}(\mathbf{x})$, where $\theta^* \in \mathop{\text{argmin}}_{\theta}\log p_{\theta}(\mathbf{x})$. Since the analysis of the model error bound in this paper assumes that the model can achieve theoretical optimal performance after a large number of iterative training, the amortized error of VAE here is negligible. We focus on the approximation gap, followed by the above discussion, we have,
\begin{equation}\label{eq-b-5}
\begin{small}
    \begin{aligned}
    & \quad \log p_{data}(\mathbf{x}_{\bar{S}}|\mathbf{x}_S)-\log p_{\theta^*}(\mathbf{x}_{\bar{S}}|\mathbf{x}_S) \\
    &= D_{KL}\left(q_{\tau}(\mathbf{z}|\mathbf{x}_{\bar{S}},\mathbf{x}_S) || p_{\theta}(\mathbf{z}|\mathbf{x}_{\bar{S}},\mathbf{x}_S)\right).
    \end{aligned}
\end{small}
\end{equation}
Eq.(\ref{eq-b-5}) shows the difference between the log-likelihood of $\log p_{data}(\mathbf{x})$ and $\log p_{\theta^*}(\mathbf{x})$, this difference also leads to the probability density difference,
\begin{equation}\label{eq-b-6}
\begin{small}
    \left| \frac{p_{\theta^*}(\mathbf{x}_{\bar{S}|\mathbf{x}_S=\mathbf{x}_S^*})} {p_{data}(\mathbf{x}_{\bar{S}|\mathbf{x}_S=\mathbf{x}_S^*})} -1 \right| \leq \delta,
\end{small}
\end{equation}
where $\delta$ is the difference of the probability density between $p_{data}$ and $p_{\theta^*}$. It should be noted that $\delta$ is positive related to $D_{KL}\left(q_{\tau}(\mathbf{z}|\mathbf{x}) || p_{\theta}(\mathbf{z}|\mathbf{x})\right)$.

Similar to Assumption 1, assuming that the parameters of the estimated conditional density $p_{\theta}(\mathbf{x}_{\bar{S}}|\mathbf{x}_S=\mathbf{x}_S^*)$ are unbiased estimates of the parameters of the optimal estimated density $p_{\theta^*}(\mathbf{x}_{\bar{S}}|\mathbf{x}_S=\mathbf{x}_S^*)$, and for a small positive $\epsilon_1$, the probability density ratio is as follows.
\begin{equation}
\begin{small}
        \left| \frac{p_{\theta}(\mathbf{x}_{\bar{S}}|\mathbf{x}_S=\mathbf{x}_S^*)} {p_{\theta^*}(\mathbf{x}_{\bar{S}}|\mathbf{x}_S=\mathbf{x}_S^*)} -1 \right| \leq \epsilon_1.
\end{small}
\end{equation}
Then we begin the analysis of VAEAC method. Firstly, $\mathbb{E}[|\hat{v}(S)-v(S)|]$ of VAE-based approach can be derived as follows.
\begin{equation}\label{eq-b-7}
\begin{small}
\begin{aligned}
    & \quad \mathbb{E}[|\hat{v}(S)-v(S)|] \\
    & = \mathbb{E} \left[\left| \frac{1}{K}\sum_{k=1}^K f(\mathbf{x}_{\bar{S}}^{(k)},\mathbf{x}_S^*) -  \mathbb{E}_{p_{data}(\mathbf{x}^*_{\bar{S}}|\mathbf{x}_S=\mathbf{x}_S^*)}\left[ f(\mathbf{x}^*_{\bar{S}},\mathbf{x}_S^*)\right]\right| \right] \\ 
    & = \mathbb{E} \left[ \left| \frac{1}{K}\sum_{k=1}^K f(\mathbf{x}_{\bar{S}}^{(k)},\mathbf{x}_S^*) - \mathbb{E}_{p_{\theta}(\mathbf{x}_{\bar{S}}|\mathbf{x}_S=\mathbf{x}_S^*)}\left[ f(\mathbf{x}_{\bar{S}},\mathbf{x}_S^*)\right]  \right.\right.\\
    &\quad \left. \left.+ \mathbb{E}_{p_{\theta}(\mathbf{x}_{\bar{S}}|\mathbf{x}_S=\mathbf{x}_S^*)}\left[ f(\mathbf{x}_{\bar{S}},\mathbf{x}_S^*)\right] - \mathbb{E}_{p_{\theta^*}(\mathbf{x}_{\bar{S}}|\mathbf{x}_S=\mathbf{x}_S^*)}\left[ f(\mathbf{x}_{\bar{S}},\mathbf{x}_S^*)\right] \right.\right.\\
    &\quad \left.\left. + \mathbb{E}_{p_{\theta^*}(\mathbf{x}_{\bar{S}}|\mathbf{x}_S=\mathbf{x}_S^*)}\left[ f(\mathbf{x}_{\bar{S}},\mathbf{x}_S^*)\right] - \mathbb{E}_{p_{data}(\mathbf{x}^*_{\bar{S}}|\mathbf{x}_S=\mathbf{x}_S^*)}\left[ f(\mathbf{x}^*_{\bar{S}},\mathbf{x}_S^*)\right]\right|\right] \\
    & \leq \mathbb{E} \left[\left| \frac{1}{K}\sum_{k=1}^K f(\mathbf{x}_{\bar{S}}^{(k)},\mathbf{x}_S^*) - \mathbb{E}_{p_{\theta}(\mathbf{x}_{\bar{S}}|\mathbf{x}_S=\mathbf{x}_S^*)}\left[ f(\mathbf{x}_{\bar{S}},\mathbf{x}_S^*)\right] \right| \right] \\
    & \quad + \mathbb{E} \left[\left| \mathbb{E}_{p_{\theta}(\mathbf{x}_{\bar{S}}|\mathbf{x}_S=\mathbf{x}_S^*)}\left[ f(\mathbf{x}_{\bar{S}},\mathbf{x}_S^*)\right] -\mathbb{E}_{p_{\theta^*}(\mathbf{x}_{\bar{S}}|\mathbf{x}_S=\mathbf{x}_S^*)}\left[ f(\mathbf{x}_{\bar{S}},\mathbf{x}_S^*)\right] \right.\right. \\
    &\quad \left.\left. + \mathbb{E}_{p_{\theta^*}(\mathbf{x}_{\bar{S}}|\mathbf{x}_S=\mathbf{x}_S^*)}\left[ f(\mathbf{x}_{\bar{S}},\mathbf{x}_S^*)\right]-  \mathbb{E}_{p_{data}(\mathbf{x}^*_{\bar{S}}|\mathbf{x}_S=\mathbf{x}_S^*)}\left[ f(\mathbf{x}^*_{\bar{S}},\mathbf{x}_S^*)\right] \right|\right].
\end{aligned}
\end{small}
\end{equation}

The first term in the last inequality of Eq.(\ref{eq-b-7}) is called statistical error, and the second is called approximation error. The upper bound of the statistical error is as follows.
\begin{equation}\label{eq-b-8}
\begin{small}
    \mathbb{E} \left[\left| \frac{1}{K}\sum_{k=1}^K f(\mathbf{x}_{\bar{S}}^{(k)},\mathbf{x}_S^*) - \mathbb{E}_{p_{\theta}(\mathbf{x}_{\bar{S}}|\mathbf{x}_S=\mathbf{x}_S^*)}\left[ f(\mathbf{x}_{\bar{S}},\mathbf{x}_S^*)\right] \right| \right] \leq \frac{\sqrt{\pi}}{\sqrt{2K}}.
\end{small}
\end{equation}

The approximation error is related to the amortization gap and the approximation gap, the error introduced from amortization gap can be described as follows.
\begin{equation}
\begin{small}
        \begin{aligned}
            & \quad \mathbb{E} \left[\left| \mathbb{E}_{p_{\theta}(\mathbf{x}_{\bar{S}}|\mathbf{x}_S=\mathbf{x}_S^*)}\left[ f(\mathbf{x}_{\bar{S}},\mathbf{x}_S^*)\right] -\mathbb{E}_{p_{\theta^*}(\mathbf{x}_{\bar{S}}|\mathbf{x}_S=\mathbf{x}_S^*)}\left[ f(\mathbf{x}_{\bar{S}},\mathbf{x}_S^*)\right] \right| \right] \\
            & = \mathbb{E} \left[ \left| \int f(\mathbf{x}_{\bar{S}},\mathbf{x}_S^*)p_\theta(\mathbf{x}_{\bar{S}}|\mathbf{x}_S=\mathbf{x}_S^*)d\mathbf{x}_{\bar{S}} - \int f(\mathbf{x}_{\bar{S}}^*,\mathbf{x}_S^*)p_{\theta^*}(\mathbf{x}_{\bar{S}}|\mathbf{x}_S=\mathbf{x}_S^*)d\mathbf{x}_{\bar{S}}   \right| \right] \\
            & = \mathbb{E} \left[ \left| \int f(\mathbf{x}_{\bar{S}},\mathbf{x}_S^*)\frac{p_\theta(\mathbf{x}_{\bar{S}}|\mathbf{x}_S=\mathbf{x}_S^*)}{p_{\theta^*}(\mathbf{x}_{\bar{S}}|\mathbf{x}_S=\mathbf{x}_S^*)}p_{\theta^*}(\mathbf{x}_{\bar{S}}|\mathbf{x}_S=\mathbf{x}_S^*)d\mathbf{x}_{\bar{S}}  \right.\right.\\
            & \left.\left.- \int f(\mathbf{x}_{\bar{S}}^*,\mathbf{x}_S^*) p_{\theta^*}(\mathbf{x}_{\bar{S}}|\mathbf{x}_S=\mathbf{x}_S^*)d\mathbf{x}_{\bar{S}}\right| \right] \\
            & = \mathbb{E} \left[ \left| \mathbb{E}_{p_{\theta^*}(\mathbf{x}_{\bar{S}}|\mathbf{x}_S=\mathbf{x}_S^*)}\left[ \left(f(\mathbf{x}_{\bar{S}},\mathbf{x}_S^*)\frac{p_\theta(\mathbf{x}_{\bar{S}}|\mathbf{x}_S=\mathbf{x}_S^*)}{p_{\theta^*}(\mathbf{x}_{\bar{S}}|\mathbf{x}_S=\mathbf{x}_S^*)} -f(\mathbf{x}_{\bar{S}}^*,\mathbf{x}_S^*) \right)\right] \right| \right] \\
            & \leq \mathbb{E} \left[ \mathbb{E}_{p_{\theta^*}(\mathbf{x}_{\bar{S}}|\mathbf{x}_S=\mathbf{x}_S^*)} \left[ \left|f(\mathbf{x}_{\bar{S}},\mathbf{x}_S^*)\frac{p_\theta(\mathbf{x}_{\bar{S}}|\mathbf{x}_S=\mathbf{x}_S^*)}{p_{\theta^*}(\mathbf{x}_{\bar{S}}|\mathbf{x}_S=\mathbf{x}_S^*)}-f(\mathbf{x}_{\bar{S}}^*,\mathbf{x}_S^*) \right| \right] \right] \\
            & = \epsilon_1.
        \end{aligned}
\end{small}
\end{equation}
The the error introduced from approximation gap can be described as follows,
\begin{equation}\label{eq-b-9}
\begin{small}
\begin{aligned}
    & \quad \mathbb{E} \left[\left| \mathbb{E}_{p_{\theta^*}(\mathbf{x}_{\bar{S}}|\mathbf{x}_S=\mathbf{x}_S^*)}\left[ f(\mathbf{x}_{\bar{S}},\mathbf{x}_S^*)\right] -\mathbb{E}_{p_{data}(\mathbf{x}_{\bar{S}}|\mathbf{x}_S=\mathbf{x}_S^*)}\left[ f(\mathbf{x}_{\bar{S}},\mathbf{x}_S^*)\right] \right| \right] \\
    & = \mathbb{E} \left[ \left| \int f(\mathbf{x}_{\bar{S}},\mathbf{x}_S^*)p_{\theta^*}(\mathbf{x}_{\bar{S}}|\mathbf{x}_S=\mathbf{x}_S^*)d\mathbf{x}_{\bar{S}} \right.\right.\\
    & \quad \left.\left.- \int f(\mathbf{x}_{\bar{S}}^*,\mathbf{x}_S^*)p_{data}(\mathbf{x}_{\bar{S}}|\mathbf{x}_S=\mathbf{x}_S^*)d\mathbf{x}_{\bar{S}}   \right| \right] \\
    & = \mathbb{E} \left[ \left| \int f(\mathbf{x}_{\bar{S}},\mathbf{x}_S^*)\frac{p_{\theta^*}(\mathbf{x}_{\bar{S}}|\mathbf{x}_S=\mathbf{x}_S^*)}{p_{data}(\mathbf{x}_{\bar{S}}|\mathbf{x}_S=\mathbf{x}_S^*)}p_{data}(\mathbf{x}_{\bar{S}}|\mathbf{x}_S=\mathbf{x}_S^*)d\mathbf{x}_{\bar{S}}  \right.\right.\\
    & \quad - \left.\left.\int f(\mathbf{x}_{\bar{S}}^*,\mathbf{x}_S^*) p_{data}(\mathbf{x}_{\bar{S}}|\mathbf{x}_S=\mathbf{x}_S^*)d\mathbf{x}_{\bar{S}}\right| \right] \\
    & = \mathbb{E} \left[ \left| \mathbb{E}_{p_{data}(\mathbf{x}_{\bar{S}}|\mathbf{x}_S=\mathbf{x}_S^*)}\left[ \left(f(\mathbf{x}_{\bar{S}},\mathbf{x}_S^*)\frac{p_{\theta^*}(\mathbf{x}_{\bar{S}}|\mathbf{x}_S=\mathbf{x}_S^*)}{p_{data}(\mathbf{x}_{\bar{S}}|\mathbf{x}_S=\mathbf{x}_S^*)} -f(\mathbf{x}_{\bar{S}}^*,\mathbf{x}_S^*) \right)\right] \right| \right] \\ 
    \notag
\end{aligned}
\end{small}
\end{equation}
\begin{equation}
\begin{small}
\begin{aligned}
    & \leq \mathbb{E} \left[ \mathbb{E}_{p_{data}(\mathbf{x}_{\bar{S}}|\mathbf{x}_S=\mathbf{x}_S^*)} \left[ \left|f(\mathbf{x}_{\bar{S}},\mathbf{x}_S^*)\frac{p_{\theta^*}(\mathbf{x}_{\bar{S}}|\mathbf{x}_S=\mathbf{x}_S^*)}{p_{data}(\mathbf{x}_{\bar{S}}|\mathbf{x}_S=\mathbf{x}_S^*)}-f(\mathbf{x}_{\bar{S}}^*,\mathbf{x}_S^*) \right| \right] \right] \\
    & = \delta.
\end{aligned}
\end{small}
\end{equation}

Thus, the approximation error is bounded with $\epsilon_1+\delta$. Finally, the estimation error bound of VAE-based approach is as follows.
\begin{equation}\label{eq-b-16}
\begin{small}
        \mathbb{E}[|\hat{v}(S)-v(S)|] = \frac{\sqrt{\pi}}{\sqrt{2K}} + \epsilon_1+\delta.
\end{small}
\end{equation}
It can be seen from Eq.(\ref{eq-b-16}), the error bound of VAE-based method is limited by the variational inference. Once $D_{KL}\left(q_{\tau}(\mathbf{z}|\mathbf{x}) || p_{\theta}(\mathbf{z}|\mathbf{x})\right)$ becomes large, the error of this method can not be ignored.

\setcounter{equation}{0}
\renewcommand\theequation{C\arabic{equation}}
\setcounter{proposition}{0}
\renewcommand\theproposition{C\arabic{proposition}}
\setcounter{corollary}{0}
\renewcommand\thecorollary{C\arabic{corollary}}
\setcounter{lemma}{0}
\renewcommand\thelemma{C\arabic{lemma}}
\setcounter{table}{0}
\renewcommand\thetable{C\arabic{table}}
\setcounter{figure}{0}
\renewcommand\thefigure{C\arabic{figure}}
\subsection{Variance analysis of importance sampling in EmSHAP}
In this subsection, we analyze the variance of importance sampling used to estimate the partition function in EmSHAP. The difference between the variance of samples drawn from the proposal conditional distribution and those from the true conditional distribution admits a closed-form expression. Specifically, for samples from the true conditional distribution, the variance is as follows.
\begin{equation}\label{eq-3-4}
\begin{small}
       \text{Var}_{\mathbf{x}_{\bar{S}} \sim p(\mathbf{x}_{\bar{S}}|\mathbf{x}_S)}[Z_{{\bar{S}};\mathbf{x}_S}]
       = \mathbb{E}_{\mathbf{x}_{\bar{S}} \sim p(\mathbf{x}_{\bar{S}}|\mathbf{x}_S)}[Z_{{\bar{S}};\mathbf{x}_S}^2]-\left(\mathbb{E}_{\mathbf{x}_{\bar{S}} \sim p(\mathbf{x}_{\bar{S}}|\mathbf{x}_S)}[Z_{{\bar{S}};\mathbf{x}_S}]\right)^2.
\end{small}
\end{equation}
Meanwhile, for samples drawn from the proposal conditional distribution, the variance of the importance-weighted estimator is,
\begin{equation}\label{eq-3-5}
\begin{small}
\begin{aligned}
       &\quad \text{Var}_{\mathbf{x}_{\bar{S}} \sim q(\mathbf{x}_{\bar{S}}|\mathbf{x}_S)}\left[Z_{{\bar{S}};\mathbf{x}_S}\frac{p(\mathbf{x}_{\bar{S}}|\mathbf{x}_S)}{q(\mathbf{x}_{\bar{S}}|\mathbf{x}_S)}\right]  \\
       &= \mathbb{E}_{\mathbf{x}_{\bar{S}} \sim q(\mathbf{x}_{\bar{S}}|\mathbf{x}_S)}\left[\left(Z_{{\bar{S}};\mathbf{x}_S}\frac{p(\mathbf{x}_{\bar{S}}|\mathbf{x}_S)}{q(\mathbf{x}_{\bar{S}}|\mathbf{x}_S)}\right)^2\right]  \\
       & \quad- \left(\mathbb{E}_{\mathbf{x}_{\bar{S}} \sim q(\mathbf{x}_{\bar{S}}|\mathbf{x}_S)}\left[Z_{{\bar{S}};\mathbf{x}_S}\frac{p(\mathbf{x}_{\bar{S}}|\mathbf{x}_S)}{q(\mathbf{x}_{\bar{S}}|\mathbf{x}_S)}\right]\right)^2 \\
       &= \mathbb{E}_{\mathbf{x}_{\bar{S}} \sim p(\mathbf{x}_{\bar{S}}|\mathbf{x}_S)}\left[Z_{{\bar{S}};\mathbf{x}_S}^2\frac{p(\mathbf{x}_{\bar{S}}|\mathbf{x}_S)}{q(\mathbf{x}_{\bar{S}}|\mathbf{x}_S)}\right] - \left(\mathbb{E}_{\mathbf{x}_{\bar{S}} \sim p(\mathbf{x}_{\bar{S}}|\mathbf{x}_S)}[Z_{{\bar{S}};\mathbf{x}_S}]\right)^2.
\end{aligned}
\end{small}
\end{equation}

Subtracting Eq.(\ref{eq-3-5}) from Eq.(\ref{eq-3-4}) yields
\begin{equation}\label{eq-3-3}
\begin{small}
\begin{aligned}
    & \quad \text{Var}_{\mathbf{x}_{\bar{S}} \sim p(\mathbf{x}_{\bar{S}}|\mathbf{x}_S)}[Z_{{\bar{S}};\mathbf{x}_S}] - \text{Var}_{\mathbf{x}_{\bar{S}} \sim q(\mathbf{x}_{\bar{S}}|\mathbf{x}_S)}[Z_{{\bar{S}};\mathbf{x}_S}] \\
    & =
    \mathbb{E}_{\mathbf{x}_{\bar{S}} \sim p(\mathbf{x}_{\bar{S}}|\mathbf{x}_S)}\left[Z_{{\bar{S}};\mathbf{x}_S}^2\right] -
    \mathbb{E}_{\mathbf{x}_{\bar{S}} \sim p(\mathbf{x}_{\bar{S}}|\mathbf{x}_S)}\left[Z_{{\bar{S}};\mathbf{x}_S}^2 \frac{p(\mathbf{x}_{\bar{S}}|\mathbf{x}_S)}{q(\mathbf{x}_{\bar{S}}|\mathbf{x}_S)}\right].
\end{aligned}
\end{small}
\end{equation}

As can be seen from Eq.(\ref{eq-3-3}), when the proposal distribution $q(\mathbf{x}_{\bar{S}}|\mathbf{x}_S)$ approaches the true distribution $p(\mathbf{x}_{\bar{S}}|\mathbf{x}_S)$, the variance discrepancy between the approximate partition function $Z_{{\bar{S}};\mathbf{x}_S}$ and the true partition function diminishes and eventually vanishes.

\setcounter{equation}{0}
\renewcommand\theequation{C\arabic{equation}}
\setcounter{proposition}{0}
\renewcommand\theproposition{C\arabic{proposition}}
\setcounter{corollary}{0}
\renewcommand\thecorollary{C\arabic{corollary}}
\setcounter{lemma}{0}
\renewcommand\thelemma{C\arabic{lemma}}
\setcounter{table}{0}
\renewcommand\thetable{C\arabic{table}}
\subsection{Structure of predictive/classification models used in Section VI}
In this subsection, the structures of predictive/classification models used in Section VI are presented, which are shown in Tables~\ref{table-c1} to \ref{table-c4}. All four models are trained using the Adam optimizer with a learning rate of 0.001. The training epoch limit is set to 100 and the training process is stopped when the validation accuracy does not improve for 10 consecutive epochs.
\begin{table}[htbp]
\caption{The structure of CNN for classification of MNIST dataset}
\label{table-c1}
\centering
\renewcommand\arraystretch{1}
\setlength{\tabcolsep}{2mm}{
\begin{tabular}{cc}
\toprule
Layer & Configuration\\
\midrule
input & $28\times 28\times 1$ grayscale image \\
conv1 & conv$_{1, 10, 3\times3}$+ReLU+MaxPool\\
conv2 & conv$_{10, 20, 5\times5}$+ReLU+MaxPool\\
fc1 & fc$_{320, 10}$+Softmax\\
\bottomrule
\end{tabular}}
\end{table}

\begin{table}[htbp]
\caption{The structure of MLP for classification of ADTI dataset}
\label{table-c2}
\centering
\renewcommand\arraystretch{1}
\setlength{\tabcolsep}{2mm}{
\begin{tabular}{cc}
\toprule
Layer & Configuration\\
\midrule
input & $12$ \\
fc1 & fc$_{64,64}$+ReLU\\
fc2 & fc$_{64, 64}$+ReLU\\
fc3 & fc$_{64,2}$+Softmax\\
\bottomrule
\end{tabular}}
\end{table}

\begin{table}[htbp]
\caption{The structure of Informer for prediction of ETT dataset}
\label{table-c3}
\centering
\renewcommand\arraystretch{1}
\setlength{\tabcolsep}{2mm}{
\begin{tabular}{cc}
\toprule
Layer & Configuration\\
\midrule
input & $120\times6$ \\
encoder & ProbSparse self-attention$_{120\times512}$\\
decoder & generative decoder$_{48\times512}$\\
fc3 & fc$_{512, 1}$\\
\bottomrule
\end{tabular}}
\end{table}

\begin{table}[htbp]
\caption{The structure of CAML for classification of MIMIC dataset}
\label{table-c4}
\centering
\renewcommand\arraystretch{1}
\setlength{\tabcolsep}{2mm}{
\begin{tabular}{cc}
\toprule
Layer & Configuration\\
\midrule
input & 2500 \\
embedding & $10947\times100$\\
dropout & dropout(0.5)\\
conv1 & conv$_{100, 100, 5\times5, padding=2}$+tanh\\
attention & fc$_{100, 50}$+softmax+dot-product\\
fc1 & fc$_{100, 50}$\\
\bottomrule
\end{tabular}}
\end{table}

\setcounter{equation}{0}
\renewcommand\theequation{D\arabic{equation}}
\setcounter{proposition}{0}
\renewcommand\theproposition{D\arabic{proposition}}
\setcounter{corollary}{0}
\renewcommand\thecorollary{D\arabic{corollary}}
\setcounter{lemma}{0}
\renewcommand\thelemma{D\arabic{lemma}}
\setcounter{table}{0}
\renewcommand\thetable{D\arabic{table}}
\setcounter{figure}{0}
\renewcommand\thefigure{D\arabic{figure}}
\subsection{More details for feature attribution in MNIST and MIMIC-III datasets}
In this subsection, additional experimental results using different feature attribution methods on the MNIST and MIMIC-III datasets are presented.
\subsubsection{Feature attribution on MNIST dataset under noisy environment}

To test the performance of the proposed method under noisy data, Gaussian distributed noises with zero mean and varying degrees of variance are added to the original images. The variances vary from 0.1 to 0.5, with an interval of 0.1. Feature attribution is then performed on these noisy images. Figs.\ref{fig5} and \ref{fig6} present the attribution results produced by the six methods on the noisy images.
\begin{figure}[htbp]
\centerline{\includegraphics[width=\linewidth]{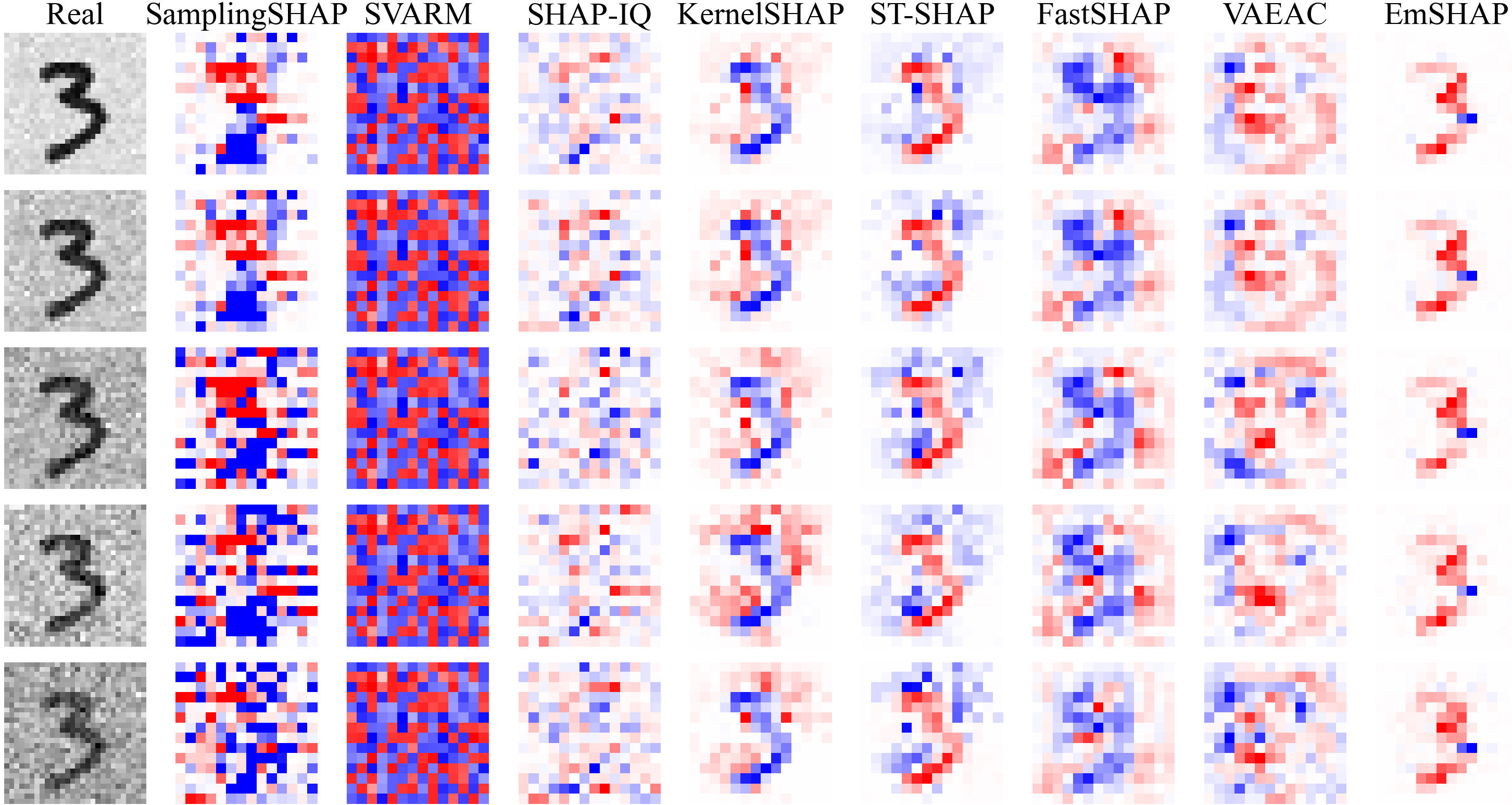}}
\caption{Feature attributions for noise-polluted image of digit 3. From top to bottom: noise variances of 0.1, 0.2, 0.3, 0.4, and 0.5.}
\label{fig5}
\end{figure}
\begin{figure}[htbp]
\centerline{\includegraphics[width=\linewidth]{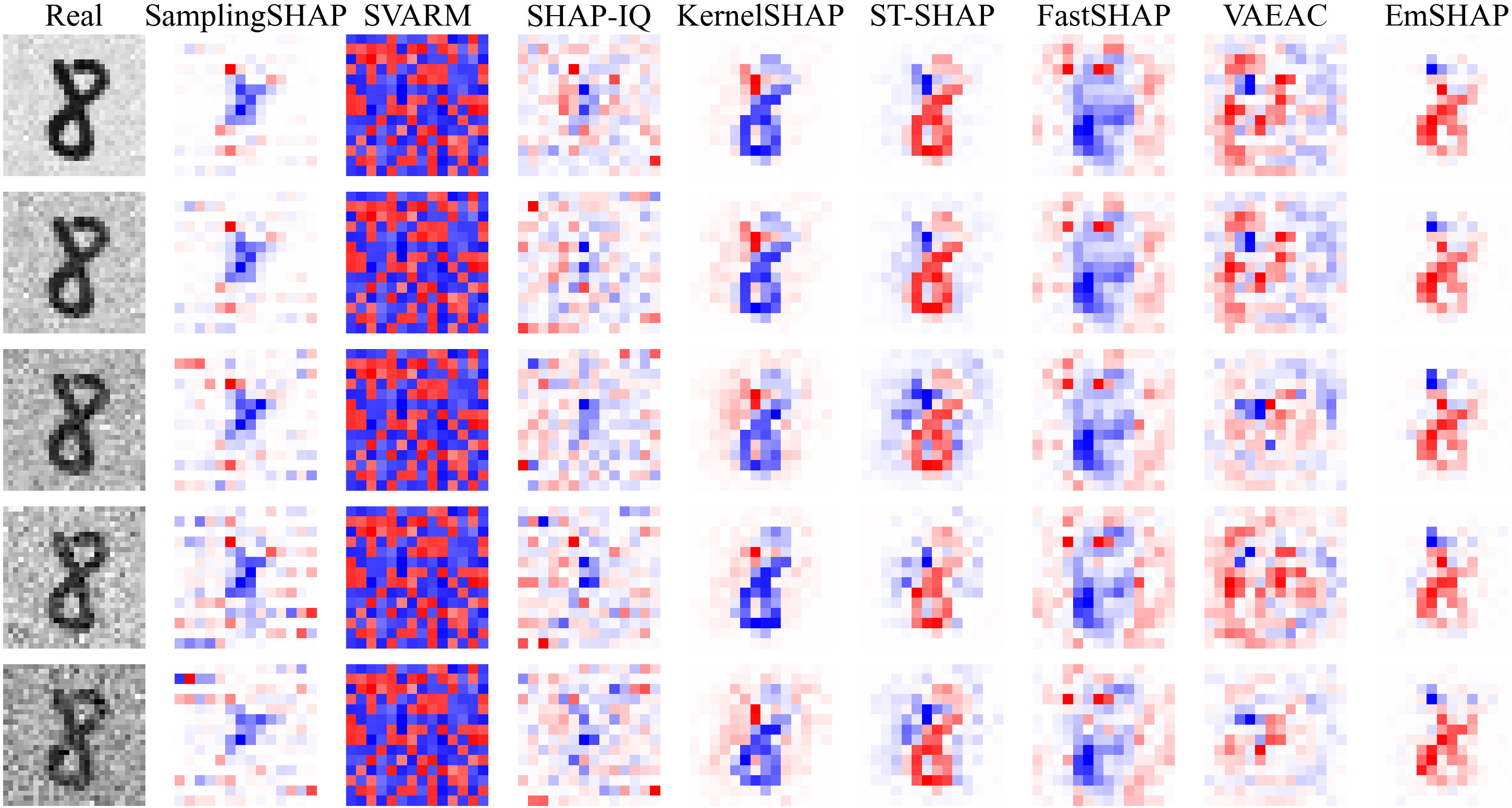}}
\caption{Feature attributions for noise-polluted image of digit 8. From top to bottom: noise variances of 0.1, 0.2, 0.3, 0.4, and 0.5.}
\label{fig6}
\end{figure}
It can be seen from Figs.~\ref{fig5} and \ref{fig6} that when the variances of noise increase, the feature attribution performance of all methods deteriorates. Whilst the sharpest performance deterioration can be observed for SamplingSHAP, the performance for KernelSHAP, ST-SHAP, and EmSHAP are relatively stable. This verifies the robustness of EmSHAP in the case of noise.


\subsubsection{Feature attribution on MIMIC-III dataset}
This subsection presents additional feature attribution results on the MIMIC-III dataset to further validate the effectiveness of EmSHAP.
Figs.~\ref{fig-E2} and~\ref{fig-E3} present the attribution results on two additional cases from the MIMIC-III dataset, corresponding to subject IDs 102024 and 136168, respectively. For subject ID 102024, the predicted codes include "\textit{end stage renal disease}" and "\textit{hyperpotassemia}". EmSHAP correctly highlights clinically meaningful terms such as "\textit{lupus}", "\textit{hyperkalemia}", and "\textit{Kayexalate}", which are directly associated with the patient’s renal dysfunction and electrolyte imbalance. Similarly, for subject ID 136168, the model predicts "\textit{coronary atherosclerosis of native coronary artery}" and "\textit{paroxysmal ventricular tachycardia}". EmSHAP assigns high attribution scores to terms such as "\textit{myocardial}, "\textit{ST elevation}", "\textit{electrocardiogram}", and "\textit{right atrial}", which are well recognized as indicative of acute myocardial infarction and arrhythmia, thereby demonstrating consistency with medical knowledge.

\begin{figure}[htbp]
\centering
\includegraphics[width=0.9\linewidth]{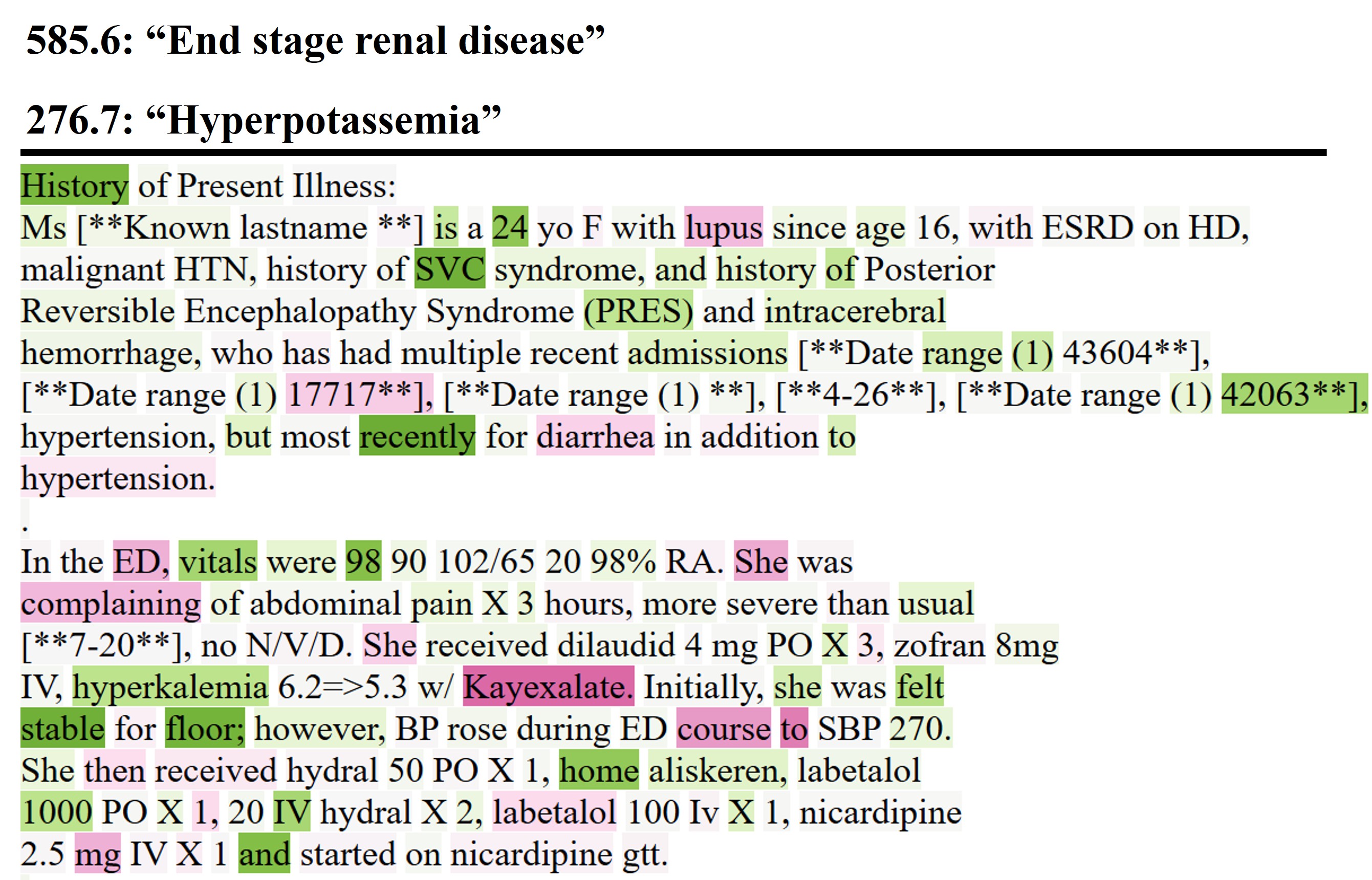}
\caption{Feature attribution results on MIMIC-III dataset (subject ID: 102024).}
\label{fig-E2}
\end{figure}

\begin{figure}[htbp]
\centering
\includegraphics[width=0.9\linewidth]{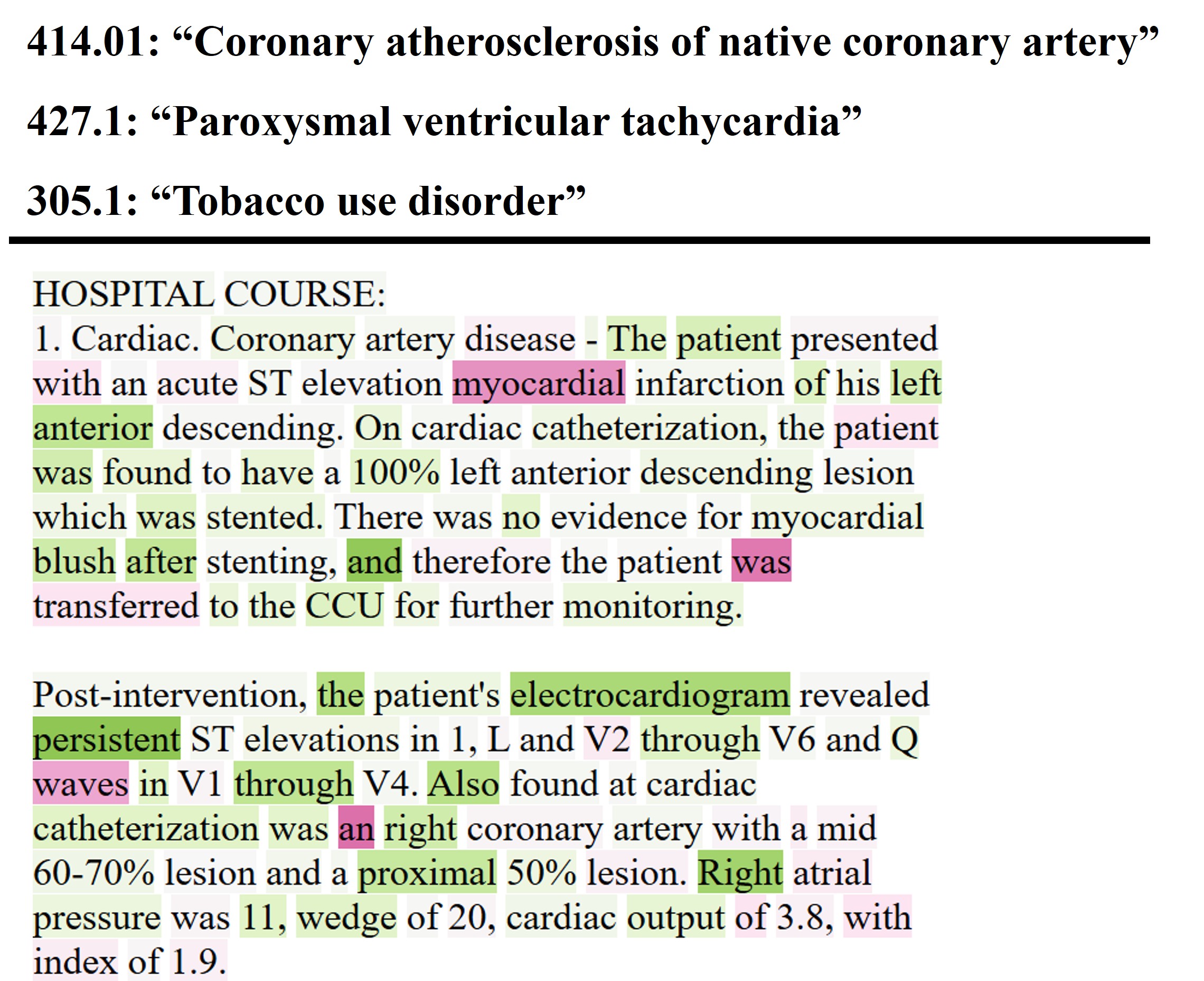}
\caption{Feature attribution results on MIMIC-III dataset (subject ID: 136168).}
\label{fig-E3}
\end{figure}

Despite the better feature attribution results, EmSHAP also spuriously produces attributions to general high-frequency words with limited clinical meaning (e.g., general descriptors like "\textit{she}" or "\textit{an}"). This likely arises due to the model's sensitivity to frequent background terms in clinical narratives. While these cases indicate that the method may not always yield perfectly precise explanations, EmSHAP nevertheless is proved to be capable of capturing clinically important signals in large-scale, high-dimensional datasets such as MIMIC-III. These results suggest that EmSHAP offers practically useful interpretability for automated medical coding by identifying the key textual features driving predictions under complex real-world conditions.

\setcounter{equation}{0}
\renewcommand\theequation{E\arabic{equation}}
\setcounter{proposition}{0}
\renewcommand\theproposition{E\arabic{proposition}}
\setcounter{corollary}{0}
\renewcommand\thecorollary{E\arabic{corollary}}
\setcounter{lemma}{0}
\renewcommand\thelemma{E\arabic{lemma}}
\setcounter{table}{0}
\renewcommand\thetable{E\arabic{table}}
\setcounter{figure}{0}
\renewcommand\thefigure{E\arabic{figure}}

\subsection{Performance analysis of EmSHAP under different parameters and model structures}

\begin{table}[!htb]
\caption{SIC AUC results for different masking rates}
\label{table7}
\centering
\renewcommand\arraystretch{0.5}
\setlength{\tabcolsep}{1mm}{
\begin{tabular}{cccccccc}
\toprule
\multirowcell{3}{Dataset} & \multirowcell{3}{SIC \\ AUC} &\multicolumn{6}{c}{Masking rate}\cr\cmidrule{3-8}
& & \multirowcell{2}{0.5} & \multirowcell{2}{0.6} & \multirowcell{2}{0.7} & \multirowcell{2}{0.8} & \multirowcell{2}{0.9} & Dynamic \\ & & & & & & & masking\\
\midrule
\multirowcell{2}{MNIST} & ADD$\uparrow$ & 0.954 & 0.937 & 0.945 & 0.943 & 0.946 & $\mathbf{0.965}$\\
& DEL$\downarrow$ & 0.103 & 0.140 & 0.129 & 0.132 & 0.126 & $\mathbf{0.037}$\\
\midrule
\multirowcell{2}{ADTI} & ADD$\uparrow$ & 0.927 & 0.928 & 0.931 & 0.923 & 0.921 & $\mathbf{0.933}$\\
& DEL$\downarrow$ & 0.420 & 0.420 & 0.413 & 0.433 & 0.436 & $\mathbf{0.411}$\\
\midrule
\multirowcell{2}{ETT} & ADD$\uparrow$ & 0.903 & 0.896 & 0.897 & 0.911 & 0.905 & $\mathbf{0.915}$\\
& DEL$\downarrow$ & 0.292 & 0.311 & 0.308 & 0.251 & 0.267 & $\mathbf{0.243}$\\
\midrule
\multirowcell{2}{MIMIC} & ADD$\uparrow$ & 0.667 & 0.591 & 0.620 & 0.605 & 0.603 & \textbf{0.699}\\
& DEL$\downarrow$ & 0.386 & 0.376 & 0.389 & 0.384 & 0.404 & \textbf{0.384}\\
\bottomrule
\end{tabular}}
\end{table}

\begin{table}[!htb]
\caption{SIC AUC results for different context vector dimensions}
\label{table9}
\centering
\renewcommand\arraystretch{0.5}
\setlength{\tabcolsep}{1mm}{
\begin{tabular}{cccccccc}
\toprule
\multirowcell{2}{Dataset} & \multirowcell{2}{SIC AUC} &\multicolumn{6}{c}{Dimension of context vector}\cr\cmidrule{3-8}
& & 0 & 8 & 16 & 32 & 64 & 128 \\
\midrule
\multirowcell{2}{MNIST} & ADD$\uparrow$ & 0.895 & 0.907& 0.944 & 0.965 & 0.965 & 0.971 \\
& DEL$\downarrow$  & 0.171 & 0.166 & 0.137 & 0.037 & 0.041 & 0.032 \\
\midrule
\multirowcell{2}{ADTI} & ADD$\uparrow$  & 0.893 & 0.910 & 0.933 & 0.935 & 0.937 & 0.942\\
& DEL$\downarrow$ &0.466 & 0.516 & 0.411 & 0.395 & 0.386 & 0.335\\
\midrule
\multirowcell{2}{ETT} & ADD$\uparrow$  & 0.847 & 0.868 & 0.908 & 0.915 & 0.926 & 0.936 \\
& DEL$\downarrow$  & 0.477 & 0.341 & 0.266 & 0.243 & 0.407 & 0.319\\
\midrule
\multirowcell{2}{MIMIC} & ADD$\uparrow$ & 0.571 & 0.602 & 0.630 & 0.699 & 0.701 & 0.701\\
& DEL$\downarrow$ & 0.492 & 0.441 & 0.405 & 0.384 & 0.373 & 0.371\\
\bottomrule
\end{tabular}}
\end{table}

\begin{table*}[!htb]
\caption{Performance of EmSHAP under different latent dimensions}
\label{table8}
\centering
\renewcommand\arraystretch{0.5}
\setlength{\tabcolsep}{2mm}{
\begin{tabular}{cccc|cc|cc|cc}
\toprule
\multirowcell{2}{Dataset} & \multirowcell{2}{Latent \\ dimension} & \multicolumn{2}{c|}{SIC AUC}& \multicolumn{2}{c|}{Run time}& \multicolumn{2}{c|}{Memory usage}& \multicolumn{2}{c}{Energy consumption}\cr\cmidrule{3-4}\cmidrule{5-6}\cmidrule{7-8}\cmidrule{9-10}
& & ADD$\uparrow$ & DEL$\downarrow$ &Training$\downarrow$ & Inference$\downarrow$ &Training$\downarrow$ & Inference$\downarrow$ &Training$\downarrow$ & Inference$\downarrow$\\
\midrule
\multirowcell{4}{MNIST} & 16 & 0.931 & 0.152 & 2813.42s & 1.46s & 1706.36MB&  1788.95MB & 7.23KJ & 1.49KJ \\
                        & 32 & 0.965 & 0.037 & 2977.10s & 1.63s & 1823.91MB & 1879.68MB & 8.62KJ & 2.13KJ  \\
                        & 64 & 0.965 & 0.033 & 3460.16s & 4.63s & 1946.19MB & 2024.79MB &9.48KJ & 2.63KJ\\
                        & 128 & 0.966 & 0.031 & 4103.57s & 13.89s & 2237.85MB& 2242.40MB  & 11.21KJ & 3.15KJ\\
\midrule
\multirowcell{4}{ADTI} & 16 & 0.921 & 0.466 & 230.49s & 0.33s & 918.45MB & 943.34MB & 11.38KJ& 0.05KJ\\
                        & 32 & 0.933 & 0.411 & 270.90s & 0.87s & 966.58MB & 946.81MB & 15.53KJ & 0.05KJ \\
                        & 64 & 0.950 & 0.368 &303.83s& 1.04s &1019.57MB & 1046.63MB & 19.12KJ & 0.07KJ \\
                        & 128 & 0.951 & 0.355 &326.36s &1.84s & 1120.41MB&  1148.21MB & 22.45KJ & 0.09KJ \\
\midrule
\multirowcell{4}{ETT} & 16 & 0.884 & 0.256 & 10.76s & 2E-3s & 519.68MB & 538.77MB & 0.82KJ & 0.03KJ\\
                        & 32  & 0.915 & 0.243 & 16.55s & 4E-3s & 554.11MB & 540.46MB & 0.82KJ & 0.03KJ  \\
                        & 64 & 0.920 & 0.227 & 25.45s & 0.10s & 611.91MB&  638.35MB & 0.82KJ & 0.03KJ \\
                        & 128 & 0.923 & 0.220 & 32.17s & 0.39s & 649.00MB & 653.30MB & 0.83KJ & 0.03KJ \\
\midrule
\multirowcell{4}{MIMIC} & 16 & 0.600 & 0.393 & 851.10s & 58.15s & 1159.61MB & 1169.85MB & 27.71KJ  & 8.81KJ \\
                        & 32 & 0.699 & 0.384 & 1168.72s & 59.93s & 1161.51MB & 1168.10MB & 28.03KJ & 9.16KJ \\
                        & 64 & 0.700 & 0.375 & 1394.08s& 72.99s & 1261.87MB & 1217.31MB & 29.35KJ & 9.37KJ\\
                        & 128 & 0.710 & 0.372 & 2266.46s & 86.59s & 1367.76MB & 1320.65MB & 53.79KJ & 10.33KJ \\
\bottomrule
\end{tabular}}
\end{table*}

\begin{table*}[!htb]
\caption{Comparison of different RNN-based proposal networks}
\label{table6}
\centering
\renewcommand\arraystretch{0.5}
\setlength{\tabcolsep}{2mm}{
\begin{tabular}{cccc|cc|cc}
\toprule
\multirowcell{2}{Dataset} & \multirowcell{2}{Proposal \\ network} & \multicolumn{2}{c|}{SIC AUC}& \multicolumn{2}{c|}{Run time}& \multicolumn{2}{c}{Memory usage}\cr\cmidrule{3-4}\cmidrule{5-6}\cmidrule{7-8}
& & ADD$\uparrow$ & DEL$\downarrow$ &Training$\downarrow$ & Inference$\downarrow$ &Training$\downarrow$ & Inference$\downarrow$ \\
\midrule
\multirowcell{4}{MNIST} & LSTM       & 0.952   & 0.108  & 3406.11s  & 2.01s  & 2067.50MB & 1999.63MB  \\
                        & GRU   & 0.965 & 0.037 & 2977.10s & 1.63s & 1879.91MB & 1823.68MB\\
                        & BiGRU	 & 0.942 &	0.084	&4108.33s &	2.75s &	3343.39MB &	4109.18MB\\
                        & Att-GRU&	0.966&	0.042&	4247.81s&	2.94s&	3341.72MB&	4111.50MB\\
\midrule
\multirowcell{4}{ADTI} & LSTM  & 0.930  & 0.416  & 303.00s   & 0.91s & 995.50MB  & 932.71MB \\
                        & GRU            & 0.933   & 0.411 & 270.90s & 0.87s & 966.58MB & 946.81MB\\
                        & BiGRU&	0.935&	0.403&	288.48s&	1.63s&	3338.62MB&	3531.10MB \\
                        & Att-GRU&	0.954&	0.494&	312.40s&	1.34s&	3345.10MB&	3535.34MB\\
\midrule
\multirowcell{4}{ETT}  & LSTM  & 0.906  & 0.292  & 17.39s   & 6E-3s & 565.50MB  & 552.71MB \\
                        & GRU            & 0.915   & 0.243 & 16.55s & 4E-3s & 554.11MB & 540.46MB\\
                        & BiGRU & 0.844	& 0.373	& 19.39s & 	0.03s	& 934.01MB & 	954.30MB \\
                        & Att-GRU & 0.843&	0.417&	20.24s&	0.04s&	1047.29MB&	1054.71MB \\

\midrule
\multirowcell{4}{MIMIC}  & LSTM  & 0.703  &  0.377 & 1884.14s   & 98.92s & 2371.65MB  & 2377.12MB \\
                        & GRU  &  0.699  & 0.384  & 1168.72s   & 59.33s & 1165.51MB   & 1168.10MB\\
                        & BiGRU &0.681	&0.366	&1661.75s	&84.92s	&2265.17MB	&2130.61MB \\
                        & Att-GRU & 0.598	&0.376	&1667.51s	&71.34s	&2390.36MB	&3889.35MB \\
\bottomrule
\end{tabular}}
\end{table*}

\begin{table*}[!htb]
\caption{Results of ablation experiments}
\label{table5}
\centering
\renewcommand\arraystretch{0.5}
\setlength{\tabcolsep}{2mm}{
\begin{tabular}{cccc|cc|cc}
\toprule
\multirowcell{2}{Dataset} & \multirowcell{2}{Explainer} & \multicolumn{2}{c|}{SIC AUC}& \multicolumn{2}{c|}{Run time}& \multicolumn{2}{c}{Memory usage}\cr\cmidrule{3-4}\cmidrule{5-6}\cmidrule{7-8}
& & ADD$\uparrow$ & DEL$\downarrow$ &Training$\downarrow$ & Inference$\downarrow$ &Training$\downarrow$ & Inference$\downarrow$\\
\midrule
\multirowcell{4}{MNIST} & Vanilla EBM        & 0.883   & 0.217       & $\mathbf{1568.18}$s & $\mathbf{1.02}$s & $\mathbf{1354.19}$MB & $\mathbf{1200.41}$MB\\
                        & Vanilla EBM+Dynamic masking   & 0.900   & 0.173       & 1839.42s  & 1.12s  & 1439.98MB & 1207.33MB  \\
                        & Vanilla EBM+GRU coupling       & 0.954   & 0.103       & 2840.98s & 1.62s   & 1777.21MB & 1730.01MB     \\
                        & EmSHAP            & $\mathbf{0.965}$ & $\mathbf{0.037}$ & 2977.10s & 1.63s & 1823.91MB & 1879.68MB\\
\midrule
\multirowcell{4}{ADTI} & Vanilla EBM       & 0.863  & 0.544  & $\mathbf{173.22}$s   & $\mathbf{0.81}$s & $\mathbf{889.21}$MB   & $\mathbf{850.41}$MB  \\
                        & Vanilla EBM+Dynamic masking   & 0.872  & 0.469  & 184.36s   & 0.82s & 950.21MB & 892.53MB  \\
                        & Vanilla EBM+GRU coupling       & 0.927  & 0.420  & 253.74s   & 0.90s  & 963.37MB & 930.10MB   \\
                        & EmSHAP            & $\mathbf{0.933}$   & $\mathbf{0.411}$ & 270.90s & 0.87s & 966.58MB & 946.81MB\\
\midrule
\multirowcell{4}{ETT} & Vanilla EBM       & 0.864  & 0.440  & $\mathbf{14.19}$s   & $\mathbf{1E-3}$s & $\mathbf{515.72}$MB   & $\mathbf{510.33}$MB  \\
                        & Vanilla EBM+Dynamic masking   & 0.873  & 0.339  & 14.77s   & 3E-3s & 520.69MB & 512.53MB  \\
                        & Vanilla EBM+GRU couping       & 0.903  & 0.292  & 15.99s   & 4E-3s  & 553.37MB & 540.10MB   \\
                        & EmSHAP            & $\mathbf{0.915}$   & $\mathbf{0.243}$ & 16.55s & 4E-3s & 554.11MB & 540.46MB\\
\midrule
\multirowcell{4}{MIMIC} & Vanilla EBM       & 0.530 & 0.477 & \textbf{64.50}s   & \textbf{31.21}s & \textbf{1051.62}MB   & \textbf{1075.91}MB  \\
                        & Vanilla EBM+Dynamic masking   & 0.560 & 0.464 & 69.25s & 32.70s & 1065.44MB & 1079.08MB  \\
                        & Vanilla EBM+GRU coupling       & 0.667  & 0.386  & 1167.89s   & 59.33s  & 1074.56MB & 1113.55MB   \\
                        & EmSHAP            &  \textbf{0.699}  & \textbf{0.384}  & 1168.72s   & 59.93s & 1165.51MB   & 1168.10MB\\
\bottomrule
\end{tabular}}
\end{table*}

\bibliographystyle{ieeetr}
\bibliography{reference}

\begin{thebibliography}{10}

\bibitem{Aggarwal2021medical}
R.~Aggarwal, V.~Sounderajah, and e.~a. Martin, Guy, ``Diagnostic accuracy of deep learning in medical imaging: a systematic review and meta-analysis,'' {\em NPJ Digital Medicine}, vol.~4, p.~65, 2021.

\bibitem{qi2020wrinkled}
Q.~Qi and V.~A. Ho, ``Wrinkled soft sensor with variable afferent morphology: Case of bending actuation,'' {\em IEEE Robotics and Automation Letters}, vol.~5, no.~3, pp.~4102--4109, 2020.

\bibitem{Ozbayoglul2020financial}
A.~M. Ozbayoglu, M.~U. Gudelek, and O.~B. Sezer, ``Deep learning for financial applications : a survey,'' {\em Applied Soft Computing}, vol.~93, p.~106384, 2020.

\bibitem{Arrieta2020XAI}
A.~B. Arrieta, N.~Díaz-Rodríguez, and J.~Del~Ser, ``Explainable artificial intelligence (xai): concepts, taxonomies, opportunities and challenges toward responsible ai,'' {\em Information Fusion}, vol.~58, pp.~82--115, 2020.

\bibitem{doshi2017towards}
F.~Doshi-Velez and B.~Kim, ``Towards a rigorous science of interpretable machine learning,'' {\em arXiv preprint arXiv:1702.08608}, 2017.

\bibitem{shapley1953value}
L.~S. Shapley, ``A value for n-person games,'' {\em Contributions to the Theory of Games (AM-28)}, vol.~2, pp.~307--318, 1953.

\bibitem{Strumbelj2010explanation}
E.~Strumbelj and I.~Kononenko, ``An efficient explanation of individual classifications using game theory,'' {\em Journal of Machine Learning Research}, vol.~11, pp.~1--18, 2010.

\bibitem{pmlr-v119-kumar20e}
I.~E. Kumar, S.~Venkatasubramanian, C.~Scheidegger, and S.~Friedler, ``Problems with shapley-value-based explanations as feature importance measures,'' in {\em International Conference on Machine Learning}, pp.~5491--5500, 2020.

\bibitem{pmlr-v119-sundararajan20b}
M.~Sundararajan and A.~Najmi, ``The many shapley values for model explanation,'' in {\em International Conference on Machine Learning}, pp.~9269--9278, 2020.

\bibitem{9565902}
D.~Fryer, I.~Strümke, and H.~Nguyen, ``Shapley values for feature selection: The good, the bad, and the axioms,'' {\em IEEE Access}, vol.~9, pp.~144352--144360, 2021.

\bibitem{vstrumbelj2014explaining}
E.~Strumbelj and I.~Kononenko, ``Explaining prediction models and individual predictions with feature contributions,'' {\em Knowledge and Information Systems}, vol.~41, pp.~647--665, 2014.

\bibitem{chen2018shapley}
J.~Chen, L.~Song, M.~J. Wainwright, and M.~I. Jordan, ``L-shapley and c-shapley: Efficient model interpretation for structured data,'' in {\em International Conference on Learning Representations}, 2019.

\bibitem{lundberg2017unified}
S.~M. Lundberg and S.-I. Lee, ``A unified approach to interpreting model predictions,'' {\em Advances in Neural Information Processing Systems}, vol.~30, 2017.

\bibitem{covert2020improving}
I.~Covert and S.-I. Lee, ``Improving kernelshap: Practical shapley value estimation using linear regression,'' in {\em International Conference on Artificial Intelligence and Statistics}, pp.~3457--3465, PMLR, 2021.

\bibitem{chen2022explaining}
H.~Chen, S.~M. Lundberg, and S.-I. Lee, ``Explaining a series of models by propagating shapley values,'' {\em Nature communications}, vol.~13, no.~1, p.~4512, 2022.

\bibitem{kar2002axiomatization}
A.~Kar, ``Axiomatization of the shapley value on minimum cost spanning tree games,'' {\em Games and Economic Behavior}, vol.~38, no.~2, pp.~265--277, 2002.

\bibitem{skibski2014algorithms}
O.~Skibski, T.~P. Michalak, T.~Rahwan, and M.~Wooldridge, ``Algorithms for the shapley and myerson values in graph-restricted games,'' in {\em International Conference on Autonomous Agents and Multi-Agent Systems}, pp.~197--204, 2014.

\bibitem{olsen2022using}
L.~H.~B. Olsen, I.~K. Glad, M.~Jullum, and K.~Aas, ``Using shapley values and variational autoencoders to explain predictive models with dependent mixed features,'' {\em Journal of Machine Learning Research}, vol.~23, no.~1, pp.~9553--9603, 2022.

\bibitem{matsui2001np}
Y.~Matsui and T.~Matsui, ``Np-completeness for calculating power indices of weighted majority games,'' {\em Theoretical Computer Science}, vol.~263, no.~1-2, pp.~305--310, 2001.

\bibitem{Strumbelj2014explanation}
E.~Strumbelj and I.~Kononenko, ``Explaining prediction models and individual predictions with feature contributions,'' {\em Knowledge and Information Systems}, vol.~41, pp.~647--665, 2014.

\bibitem{pang2025shapley}
J.~Pang, J.~Pei, H.~Xia, X.~Li, and J.~Liu, ``Shapley value estimation based on differential matrix,'' {\em Proceedings of the ACM on Management of Data}, vol.~3, no.~1, pp.~1--28, 2025.

\bibitem{castro2009polynomial}
J.~Castro, D.~G{\'o}mez, and J.~Tejada, ``Polynomial calculation of the shapley value based on sampling,'' {\em Computers \& operations research}, vol.~36, no.~5, pp.~1726--1730, 2009.

\bibitem{Okhrati2020Multilinear}
R.~Okhrati and A.~Lipani, ``A multilinear sampling algorithm to estimate shapley values,'' in {\em 2020 25th International Conference on Pattern Recognition (ICPR)}, pp.~7992--7999, 2021.

\bibitem{Rory2022Sampling}
R.~Mitchell, J.~Cooper, E.~Frank, and G.~Holmes, ``Sampling permutations for shapley value estimation,'' {\em Journal of Machine Learning Research}, vol.~23, no.~43, pp.~1--46, 2022.

\bibitem{tsai2023faith}
C.-P. Tsai, C.-K. Yeh, and P.~Ravikumar, ``Faith-shap: The faithful shapley interaction index,'' {\em Journal of Machine Learning Research}, vol.~24, no.~94, pp.~1--42, 2023.

\bibitem{castro2017improving}
J.~Castro, D.~G{\'o}mez, E.~Molina, and J.~Tejada, ``Improving polynomial estimation of the shapley value by stratified random sampling with optimum allocation,'' {\em Computers \& Operations Research}, vol.~82, pp.~180--188, 2017.

\bibitem{zhang2023efficient}
J.~Zhang, Q.~Sun, J.~Liu, L.~Xiong, J.~Pei, and K.~Ren, ``Efficient sampling approaches to shapley value approximation,'' {\em Proceedings of the ACM on Management of Data}, vol.~1, no.~1, pp.~1--24, 2023.

\bibitem{simon2020projected}
G.~Simon and T.~Vincent, ``A projected stochastic gradient algorithm for estimating shapley value applied in attribute importance,'' in {\em International Cross-Domain Conference for Machine Learning and Knowledge Extraction}, pp.~97--115, Springer, 2020.

\bibitem{jethani2022fastshap}
N.~Jethani, M.~Sudarshan, I.~C. Covert, S.-I. Lee, and R.~Ranganath, ``Fast{SHAP}: Real-time shapley value estimation,'' in {\em International Conference on Learning Representations}, 2022.

\bibitem{kelodjou2024shaping}
G.~Kelodjou, L.~Roz{\'e}, V.~Masson, L.~Gal{\'a}rraga, R.~Gaudel, M.~Tchuente, and A.~Termier, ``Shaping up shap: Enhancing stability through layer-wise neighbor selection,'' in {\em Proceedings of the AAAI Conference on Artificial Intelligence}, vol.~38, pp.~13094--13103, 2024.

\bibitem{bordt2023shapley}
S.~Bordt and U.~von Luxburg, ``From shapley values to generalized additive models and back,'' in {\em International Conference on Artificial Intelligence and Statistics}, pp.~709--745, PMLR, 2023.

\bibitem{chau2022rkhs}
S.~L. Chau, R.~Hu, J.~Gonzalez, and D.~Sejdinovic, ``Rkhs-shap: Shapley values for kernel methods,'' {\em Advances in neural information processing systems}, vol.~35, pp.~13050--13063, 2022.

\bibitem{ancona2019explaining}
M.~Ancona, C.~Oztireli, and M.~Gross, ``Explaining deep neural networks with a polynomial time algorithm for shapley value approximation,'' in {\em International conference on machine learning}, pp.~272--281, PMLR, 2019.

\bibitem{wang2021shapley}
R.~Wang, X.~Wang, and D.~I. Inouye, ``Shapley explanation networks,'' in {\em International Conference on Learning Representations}, 2021.

\bibitem{Scott2020treeshap}
S.~M. Lundberg, G.~Erion, H.~Chen, A.~DeGrave, J.~M. Prutkin, B.~Nair, R.~Katz, J.~Himmelfarb, N.~Bansal, and S.-I. Lee, ``From local explanations to global understanding with explainable ai for trees,'' {\em Nature Machine Intelligence}, vol.~2, pp.~56--67, 2020.

\bibitem{frye2021shapley}
C.~Frye, D.~de~Mijolla, T.~Begley, L.~Cowton, M.~Stanley, and I.~Feige, ``Shapley explainability on the data manifold,'' in {\em International Conference on Learning Representations}, 2021.

\bibitem{cremer2018inference}
C.~Cremer, X.~Li, and D.~Duvenaud, ``Inference suboptimality in variational autoencoders,'' in {\em International Conference on Machine Learning}, pp.~1078--1086, 2018.

\bibitem{fumagalli2023shap}
F.~Fumagalli, M.~Muschalik, P.~Kolpaczki, E.~H{\"u}llermeier, and B.~Hammer, ``Shap-iq: Unified approximation of any-order shapley interactions,'' {\em Advances in Neural Information Processing Systems}, vol.~36, pp.~11515--11551, 2023.

\bibitem{muschalik2024beyond}
M.~Muschalik, F.~Fumagalli, B.~Hammer, and E.~H{\"u}llermeier, ``Beyond treeshap: Efficient computation of any-order shapley interactions for tree ensembles,'' in {\em Proceedings of the AAAI Conference on Artificial Intelligence}, vol.~38, pp.~14388--14396, 2024.

\bibitem{Fabian2024kernelSHAPIQ}
F.~Fumagalli, M.~Muschalik, P.~Kolpaczki, E.~Hüllermeier, and B.~Hammer, ``Kernelshap-iq: Weighted least square optimization for shapley interactions,'' in {\em Forty-first International Conference on Machine Learning}, 2024.

\bibitem{Patrick2024svarmiq}
P.~Kolpaczki, M.~Muschalik, F.~Fumagalli, B.~Hammer, and E.~Hüllermeier, ``Svarm-iq: Efficient approximation of any-order shapley interactions through stratification,'' in {\em Proceedings of the International Conference on Artificial Intelligence and Statistics}, pp.~3520--3528, 2024.

\bibitem{chenhugh2023shapoverview}
H.~Chen, I.~C. Covert, S.~M. Lundberg, and S.-I. Lee, ``Algorithms to estimate shapley value feature attributions,'' {\em Nature Machine Intelligence}, vol.~5, pp.~590--601, 2023.

\bibitem{He_2016_CVPR}
K.~He, X.~Zhang, S.~Ren, and J.~Sun, ``Deep residual learning for image recognition,'' in {\em Computer Vision and Pattern Recognition}, 2016.

\bibitem{strauss2021arbitrary}
R.~Strauss and J.~B. Oliva, ``Arbitrary conditional distributions with energy,'' {\em Advances in Neural Information Processing Systems}, vol.~34, pp.~752--763, 2021.

\bibitem{Arbel2021Gebm}
M.~Arbel, L.~Zhou, and A.~Gretton, ``Generalized energy based models,'' in {\em International Conference on Representation Learning}, pp.~4231--4239, 2021.

\bibitem{lecun2006tutorial}
Y.~LeCun, S.~Chopra, R.~Hadsell, M.~Ranzato, and F.~Huang, ``A tutorial on energy-based learning,'' {\em Predicting structured data}, vol.~1, no.~0, 2006.

\bibitem{zhang2022identifiable}
Y.~Zhang, J.~Berrevoets, and M.~Van Der~Schaar, ``Identifiable energy-based representations: an application to estimating heterogeneous causal effects,'' in {\em International Conference on Artificial Intelligence and Statistics}, pp.~4158--4177, 2022.

\bibitem{khemakhem2020ice}
I.~Khemakhem, R.~Monti, D.~Kingma, and A.~Hyvarinen, ``Ice-beem: Identifiable conditional energy-based deep models based on nonlinear ica,'' {\em Advances in Neural Information Processing Systems}, vol.~33, pp.~12768--12778, 2020.

\bibitem{Importance2010Tokdar}
S.~T. Tokdar and R.~E. Kass, ``Importance sampling: a review,'' {\em Wires Computational statistics}, vol.~2, pp.~54--60, 2010.

\bibitem{zhai16}
S.~Zhai, Y.~Cheng, W.~Lu, and Z.~Zhang, ``Deep structured energy based models for anomaly detection,'' in {\em Proceedings of The 33rd International Conference on Machine Learning}, vol.~48, pp.~1100--1109, PMLR, 2016.

\bibitem{Grathwohl2020Your}
W.~Grathwohl, K.-C. Wang, J.-H. Jacobsen, D.~Duvenaud, M.~Norouzi, and K.~Swersky, ``Your classifier is secretly an energy based model and you should treat it like one,'' in {\em International Conference on Learning Representations}, 2020.

\bibitem{ma2020rademacher}
C.~Ma, Q.~Wang, {\em et~al.}, ``Rademacher complexity and the generalization error of residual networks,'' {\em Communications in Mathematical Sciences}, vol.~18, no.~6, pp.~1755--1774, 2020.

\bibitem{kusupati2018fastgrnn}
A.~Kusupati, M.~Singh, K.~Bhatia, A.~Kumar, P.~Jain, and M.~Varma, ``Fastgrnn: A fast, accurate, stable and tiny kilobyte sized gated recurrent neural network,'' {\em Advances in neural information processing systems}, vol.~31, 2018.

\bibitem{Alsmeyer2011}
G.~Alsmeyer, {\em Chebyshev's Inequality}, pp.~239--240.
\newblock Berlin, Heidelberg: Springer Berlin Heidelberg, 2011.

\bibitem{Hoeffding}
W.~Hoeffding, ``Probability inequalities for sums of bounded random variables,'' {\em Journal of the American Statistical Association}, vol.~58, no.~301, pp.~13--30, 1963.

\bibitem{maleki2013bounding}
S.~Maleki, L.~Tran-Thanh, G.~Hines, T.~Rahwan, and A.~Rogers, ``Bounding the estimation error of sampling-based shapley value approximation,'' {\em arXiv preprint arXiv:1306.4265}, 2013.

\bibitem{kolpaczki2024approximating}
P.~Kolpaczki, V.~Bengs, M.~Muschalik, and E.~H{\"u}llermeier, ``Approximating the shapley value without marginal contributions,'' in {\em Proceedings of the AAAI Conference on Artificial Intelligence}, vol.~38, pp.~13246--13255, 2024.

\bibitem{verdinelli2024feature}
I.~Verdinelli and L.~Wasserman, ``Feature importance: A closer look at shapley values and loco,'' {\em Statistical Science}, vol.~39, no.~4, pp.~623--636, 2024.

\bibitem{sebastian2024feature}
C.~Sebasti{\'a}n and C.~E. Gonz{\'a}lez-Guill{\'e}n, ``A feature selection method based on shapley values robust for concept shift in regression,'' {\em Neural Computing and Applications}, vol.~36, no.~23, pp.~14575--14597, 2024.

\bibitem{williamson2020efficient}
B.~Williamson and J.~Feng, ``Efficient nonparametric statistical inference on population feature importance using shapley values,'' in {\em International conference on machine learning}, pp.~10282--10291, PMLR, 2020.

\bibitem{kapishnikov2019xrai}
A.~Kapishnikov, T.~Bolukbasi, F.~Vi{\'e}gas, and M.~Terry, ``Xrai: Better attributions through regions,'' in {\em Proceedings of the IEEE/CVF international conference on computer vision}, pp.~4948--4957, 2019.

\bibitem{zhuo2024Integrated}
Y.~Zhuo and Z.~Ge, ``Ig2: Integrated gradient on iterative gradient path for feature attribution,'' {\em IEEE Transactions on Pattern Analysis and Machine Intelligence}, vol.~46, no.~11, pp.~7173--7190, 2024.

\bibitem{erion2021improving}
G.~Erion, J.~D. Janizek, P.~Sturmfels, S.~M. Lundberg, and S.-I. Lee, ``Improving performance of deep learning models with axiomatic attribution priors and expected gradients,'' {\em Nature machine intelligence}, vol.~3, no.~7, pp.~620--631, 2021.

\bibitem{nash2019autoregressive}
C.~Nash and C.~Durkan, ``Autoregressive energy machines,'' in {\em International Conference on Machine Learning}, pp.~1735--1744, PMLR, 2019.

\bibitem{haoyietal2023informerEx}
H.~Zhou, J.~Li, S.~Zhang, S.~Zhang, M.~Yan, and H.~Xiong, ``Expanding the prediction capacity in long sequence time-series forecasting,'' {\em Artificial Intelligence}, vol.~318, p.~103886, 2023.

\bibitem{haoyietal2021informer}
H.~Zhou, S.~Zhang, J.~Peng, S.~Zhang, J.~Li, H.~Xiong, and W.~Zhang, ``Informer: Beyond efficient transformer for long sequence time-series forecasting,'' in {\em Proceedings of the AAAI Conference on Artificial Intelligence}, vol.~35, pp.~11106--11115, 2021.

\bibitem{Goodfellow2015Explaining}
I.~J. Goodfellow, J.~Shlens, and C.~Szegedy, ``Explaining and harnessing adversarial examples,'' in {\em International Conference on Learning Representations}, 2015.

\bibitem{johnson2016mimic}
A.~E. Johnson, T.~J. Pollard, L.~Shen, L.-w.~H. Lehman, M.~Feng, M.~Ghassemi, B.~Moody, P.~Szolovits, L.~Anthony~Celi, and R.~G. Mark, ``Mimic-iii, a freely accessible critical care database,'' {\em Scientific data}, vol.~3, no.~160035, 2016.

\bibitem{Edin2023Automated}
J.~Edin, A.~Junge, J.~D. Havtorn, L.~Borgholt, M.~Maistro, T.~Ruotsalo, and L.~Maal\o{}e, ``Automated medical coding on mimic-iii and mimic-iv: A critical review and replicability study,'' in {\em Proceedings of the 46th International ACM SIGIR Conference on Research and Development in Information Retrieval}, no.~11, p.~2572–2582, 2023.

\bibitem{mullenbach2018explainable}
J.~Mullenbach, S.~Wiegreffe, J.~Duke, J.~Sun, and J.~Eisenstein, ``Explainable prediction of medical codes from clinical text,'' in {\em Proceedings of the 2018 Conference of the North {A}merican Chapter of the Association for Computational Linguistics: Human Language Technologies, Volume 1 (Long Papers)}, pp.~1101--1111, 2018.

\bibitem{tian2023heterogeneous}
Y.~Tian, K.~Dong, C.~Zhang, C.~Zhang, and N.~V. Chawla, ``Heterogeneous graph masked autoencoders,'' in {\em Proceedings of the AAAI Conference on Artificial Intelligence}, vol.~37, pp.~9997--10005, 2023.

\bibitem{boole1847mathematical}
G.~Boole, {\em The mathematical analysis of logic}.
\newblock Philosophical Library, 1847.

\bibitem{garreau20a}
D.~Garreau and U.~von Luxburg, ``Explaining the explainer: A first theoretical analysis of lime,'' in {\em Proceedings of the Twenty Third International Conference on Artificial Intelligence and Statistics} (S.~Chiappa and R.~Calandra, eds.), vol.~108 of {\em Proceedings of Machine Learning Research}, pp.~1287--1296, PMLR, 26-28 Aug 2020.

\bibitem{weyl1912asymptotische}
H.~Weyl, ``Das asymptotische verteilungsgesetz der eigenwerte linearer partieller differentialgleichungen (mit einer anwendung auf die theorie der hohlraumstrahlung),'' {\em Mathematische Annalen}, vol.~71, no.~4, pp.~441--479, 1912.

\bibitem{shekhovtsov2022vae}
A.~Shekhovtsov, D.~Schlesinger, and B.~Flach, ``{VAE} approximation error: {ELBO} and exponential families,'' in {\em International Conference on Learning Representations}, 2022.

\end{thebibliography}

%
\vspace*{-1.4cm}
\begin{IEEEbiographynophoto}
{Cheng Lu} received the B.S. degree from China Jiliang University, China, in 2020. He is currently working towards the Ph.D. degree at the College of Metrology Measurement and Instrument, China Jiliang University, Hangzhou. His research interests include explainable artificial intelligence and industrial process modeling.
\end{IEEEbiographynophoto}
\vspace*{-1.4cm}
\begin{IEEEbiographynophoto}
{Jiusun Zeng} received a B.S. degree in 2004 and a Ph.D. degree in 2009, both from Zhejiang University, China. He served as a postdoctoral research associate in the Institute of Cyber Systems and Control of Zhejiang University from 2009 to 2011. After that he became a faculty member of China Jiliang University. He joined Hangzhou Normal University in 2022 and is currently a full Professor at the School of Mathematics, Hangzhou Normal University. His research interests focus on industrial big data analysis, artificial intelligence and applications.
\end{IEEEbiographynophoto}
\vspace*{-1.4cm}
\begin{IEEEbiographynophoto}
{Yu Xia} received the Ph.D. degree in mathematics from Zhejiang University, Hangzhou, China, in 2018. She is currently with the School of Mathematics, Hangzhou Normal University, Hangzhou. Her research interests include signal processing, phase retrieval, and blind deconvolution. \end{IEEEbiographynophoto}
\vspace*{-1.4cm}
\begin{IEEEbiographynophoto}
{Jinhui Cai} received the Ph.D. degree in College of Control Science and Engineering from Zhejiang University, Hangzhou, China, in 2005. He is currently a full Professor at the College of Metrology Measurement and Instrument, China Jiliang University. His research interests focus on digital metrology.
\end{IEEEbiographynophoto}
\vspace*{-1.4cm}
\begin{IEEEbiographynophoto}
{Shihua Luo} received the B.S. degree and M.S. degree in Mathematics from Jiangxi Normal University, China, in 1998 and 2001, and the Ph.D. degree in Operational Research and Control Theory from Zhejiang University, China, in 2007. He is currently a Professor at Jiangxi University of Finance and Economics. His research interests are in the field of modeling and optimization of complex systems, especially on the extraction of nonlinear characteristics of chemical process and financial process.
\end{IEEEbiographynophoto}

\end{document}